\documentclass{article}

\usepackage{arxiv}

\usepackage[utf8]{inputenc}
\usepackage[T1]{fontenc}

\usepackage[numbers]{natbib}
\usepackage{hyperref}
\usepackage{url}

\usepackage{amsmath}
\usepackage{mathtools}
\usepackage{amssymb}
\usepackage{amsfonts}
\usepackage{amsthm}

\usepackage{graphicx}
\usepackage{booktabs}
\usepackage{multirow}

\usepackage{algorithm}
\usepackage{algorithmic}

\usepackage{nicefrac}
\usepackage{microtype}
\usepackage{xcolor}

\newtheorem{theorem}{Theorem}
\newtheorem{lemma}{Lemma}

\newtheorem{assumption}{Assumption}
\newtheorem{definition}{Definition}
\newtheorem{remark}{Remark}

\title{Improving Bayesian Optimization via Training-Aware
Conditional Diffusion Models}

\author{
Yilin Zheng\thanks{Equal contribution.} \\
National University of Singapore \\
Singapore
\And
Haowei Wang\footnotemark[1] \\
National University of Singapore \\
Singapore
\And
Szu Hui Ng \\
National University of Singapore \\
Singapore
\And
Enlu Zhou \\
Georgia Institute of Technology \\
Atlanta, GA, USA
}

\begin{document}
\maketitle
\begin{abstract}
Bayesian optimization (BO) is a widely used approach for black-box optimization that uses a Gaussian process (GP) as a surrogate and guides sequential evaluations via an acquisition function, with the ultimate goal of locating the global optimum $\mathbf{x}^{\star}$. To align with this goal, information-based acquisition functions such as Predictive Entropy Search (PES) model $\mathbf{x}^{\star}$ as a random variable and reduce the entropy of its distribution, but approximating this distribution via traditional GP posterior sampling is computationally expensive. To address this limitation, we leverage Conditional Diffusion Models (CDMs) to efficiently approximate the distribution of $\mathbf{x}^{\star}$ and develop BO-inherent training strategies for CDMs. Motivated by the structural properties of the CDM-learned distribution, we further develop an acquisition strategy termed Diffusion-based Mode Seeking (DMS) to guide the sequential evaluation. We establish a sub-optimality guarantee for the CDM-learned distribution and demonstrate through extensive experiments that DMS outperforms standard BO baselines.
\end{abstract}

\section{Introduction}\label{sec:intro}
Black-box optimization problems are prevalent in various scientific and engineering domains, such as hyperparameter tuning for neural networks \citep{wistuba2018scalable, cowen2022hebo}, optimal molecular design in drug discovery \citep{schweidtmann2018machine, khan2023toward}, and experimental design in food science \citep{khongkomolsakul2025improving}. Without loss of generality, such optimization problems can be formulated as
\begin{equation}
\max_{\mathbf{x} \in \mathcal{X}} f(\mathbf{x}),
\end{equation}
where the design space $\mathcal{X}$ is a compact subset of $\mathbb{R}^d$. In many cases, the black-box function $f$ can only be evaluated through noisy observations, and its derivatives are unavailable \citep{frazier2018tutorial}. The ultimate goal of black-box optimization is to identify the location of the global optimum $\mathbf{x}^{\star}$, which is a nontrivial challenge in general. 

When sequential evaluations of the black-box function are permitted, Bayesian Optimization (BO) is a widely used approach for addressing such challenging problems because of its high data efficiency \citep{jones1998efficient, wang2017max, frazier2018tutorial, wang2023recent}. BO typically employs a Gaussian Process (GP) as a probabilistic surrogate model for the black-box function $f$, and relies on an acquisition function to determine the next evaluation point. Among various acquisition strategies, information-based acquisition functions, including Entropy Search (ES) \citep{hennig2012entropy} and Predictive Entropy Search (PES) \citep{hernandez2014predictive} are particularly relevant to the ultimate goal of black-box optimization. They treat $\mathbf{x}^{\star}$ as a random variable, model its distribution using the GP posterior, and select the next evaluation point by maximizing the expected reduction in the entropy of the distribution of $\mathbf{x}^{\star}$. However, these approaches are computationally expensive and poorly scalable with the input dimension, as modeling the distribution of $\mathbf{x}^{\star}$ requires repeatedly sampling function sample paths from the GP posterior and computing their global optima. 

To model the distribution of $\mathbf{x}^{\star}$ in a more efficient and effective manner, we notice that prior works in data-driven black-box optimization \citep{kumar2020model, krishnamoorthy2023diffusion, li2024diffusion, wu2024diffusion, wang2025nested} have explored modeling the one-to-many inverse mapping from a target observation ${y}$ to the corresponding inputs $\mathbf{x}$ such that $f(\mathbf{x}) +\epsilon= {y}$, where $\epsilon$ represents the noise. As multiple inputs may yield the same observation, this inverse mapping can naturally be represented by a conditional distribution $P(\mathbf{x} \mid {y})$. Consequently, given an estimate $\widehat{y}^{\star}$ of the optimal function value, the resulting conditional distribution $P(\mathbf{x}\mid y=\widehat{y}^{\star})$ can be regarded as an estimate of the distribution of $\mathbf{x}^{\star}$, sampling from which can consequently provide reliable $\mathbf{x}^{\star}$ candidates. 

In these prior works, Conditional Diffusion Models (CDMs) \citep{song2020score, dhariwal2021diffusion, ho2022classifier} are typically employed to approximate the target distribution $P(\mathbf{x} \mid {y})$, and we denote by $\widehat{P}(\mathbf{x}\mid y)$ the distribution learned by CDMs. The outstanding performance CDMs and their efficiency of generating candidates of $\mathbf{x}^{\star}$ make them promising to be used in BO to approximate the distribution of $\mathbf{x}^{\star}$. 

However, we observe that directly applying CDMs in BO presents nontrivial challenges, particularly in how to properly train a CDM under BO-specific scenarios. These challenges arise mainly from two sources: (1) BO typically starts with a very small dataset $\mathcal{D}_n=\{(\mathbf{x}_i,y_i)\}_{i=1}^n$, making it difficult for a CDM to learn a meaningful conditional distribution. While augmenting the dataset via pseudo-labeling according to a regression model \citep{li2024diffusion} is a natural remedy, it raises another non-trivial question of how to assign reasonable pseudo-labels to unevaluated designs; (2) even with pseudo-labeling augmentation ideas, it remains challenging to identify input designs associated with high pseudo-label values, which are crucial for training CDMs to generate high-quality $\mathbf{x}^{\star}$ candidates. This is because pseudo-labels are intended to indicate the potential optimality of an input and provide guidance to CDMs on which input regions are likely to contain high-quality $\mathbf{x}^{\star}$ candidates. Without finding inputs with high pseudo-labels, CDMs may fail to receive informative training signals and thus struggle to generate high-quality $\mathbf{x}^{\star}$ candidates.

In this paper, we propose BO-inherent training strategies to address the above challenges, enabling the training of CDMs to effectively approximate the distribution of $\mathbf{x}^{\star}$. While the distribution learned by CDMs could in principle be used to provide samples of $\mathbf{x}^{\star}$ for entropy approximations in ES/PES, we observe a difference in its structure compared to the distribution induced by GP posterior sampling. In particular, the CDM-learned distribution tends to be sharply concentrated around a small set of high-quality regions, whereas GP-induced distributions over $\mathbf{x}^{\star}$ often spread their mass over broader regions of the input space. Consequently, entropy reduction is well suited to the latter, but less natural for the former. Motivated by this observation, we introduce an intuitive and effective acquisition strategy termed Diffusion-based Mode Seeking (DMS), which selects the mode of the density of the learned distribution $\widehat{P}\left(\mathbf{x} \mid y=\widehat{y}^{\star}\right)$ as the next evaluation point.

Moreover, motivated by recent analyses of diffusion-based optimization \citep{yuan2023reward, li2024diffusion}, we derive sub-optimality bounds for CDM-learned distribution under non-linear objectives modeled by GP posteriors, which is, to our best knowledge, not covered by prior analyses that mainly based on assumptions of linear objectives.

Our contributions are summarized as follows: 

\begin{itemize}
\item[1).] We develop effective CDM training strategies that leverage the GP as the estimator to the objective function for pseudo-label assignment, and employ short-run L-BFGS to identify input regions with high pseudo-labels. Building on the CDM trained by the pseudo-dataset, we further design the DMS acquisition strategy, together forming a practical and scalable pipeline for integrating CDMs into BO.

\item[2).] We establish the sub-optimality guarantee for CDM-learned distribution under non-linear objective functions modeled by the GP posterior, which provides a quality certificate for the next evaluation point determined by DMS. 

\item[3).] We evaluate DMS on both synthetic benchmarks and real-world optimization tasks, and show that it achieves strong performance against other commonly used BO baselines. Ablation studies further validate the effectiveness of our proposed training strategies for CDMs, and align with our theoretical analysis.

\end{itemize}
The remainder of this paper is organized as follows. Section~\ref{sec:background} recaps main components of both BO and CDM. Section~\ref{sec:methodology} presents the proposed method and algorithmic details. Section~\ref{sec:theoretical_analysis} provides our theoretical analysis. Section~\ref{sec:experiments} reports our experimental results, followed by conclusions in Section~\ref{sec:conclusion}. 

\section{Background}\label{sec:background}


\subsection{Bayesian Optimization}\label{subsec:BO}
BO aims to find the global maximizer of a black-box function $f$ given an observed dataset $\mathcal{D}_n=\left\{\left(\mathbf{x}_i, y_i\right)\right\}_{i=1}^n$, where $y_i=f\left(\mathbf{x}_i\right)+\epsilon_i$ and the noise terms $\epsilon_i \sim \mathcal{N}\left(0, \sigma^2\right)$ are i.i.d. The objective function $f$ is typically assumed to be continuous over the input space $\mathcal{X}$, and its analytical form and derivatives are unknown. Standard BO algorithms alternate between fitting a GP posterior to the current observed dataset and maximizing an acquisition function $\alpha_n^{\text{acq}}(\mathbf{x})$ to find the next evaluation point, repeating this process until the given evaluation budget $B$ is exhausted. For completeness, we briefly recap GP and acquisition functions respectively as follows. 

\subsubsection{Gaussian Process}\label{subsec:GP}
A Gaussian process (GP) prior is commonly placed over the objective function $f$, fully specified by a mean function $\mu(\mathbf{x}):=\mathbb{E}[f(\mathbf{x})]$ and a kernel function $k\left(\mathbf{x}, \mathbf{x}^{\prime}\right):=\mathbb{E}\left[\left(f(\mathbf{x})-\mu_0(\mathbf{x})\right)\left(f\left(\mathbf{x}^{\prime}\right)-\mu_0\left(\mathbf{x}^{\prime}\right)\right)\right]$, such that $f(\mathbf{x}) \sim \mathcal{G} \mathcal{P}\left(\mu(\mathbf{x}), k\left(\mathbf{x}, \mathbf{x}^{\prime}\right)\right)$. For simplicity, we assume a zero prior mean, i.e., $\mu(\mathbf{x})=0$. Given the observed dataset $\mathcal{D}_n$, the GP posterior $f(\mathbf{x})\mid \mathcal{D}_n \sim \mathcal{GP}(\mu_n(\mathbf{x}), k_n(\mathbf{x}, \mathbf{x}^\prime))$ has mean $\mu_n(\mathbf{x})=\mathbf{k}_n(\mathbf{x})^{\top} (\mathbf{K}_n+\sigma^2\mathbf{I})^{-1} \mathbf{y}_n$ and covariance $k_n\left(\mathbf{x}, \mathbf{x}^{\prime}\right)=k\left(\mathbf{x}, \mathbf{x}^{\prime}\right)-\mathbf{k}_n(\mathbf{x})^{\top}( \mathbf{K}_n+\sigma^2\mathbf{I})^{-1} \mathbf{k}_n\left(\mathbf{x}^{\prime}\right)$, 
where $\mathbf{k}_n(\mathbf{x})=\left[k\left(\mathbf{x}, \mathbf{x}_i\right)\right]_{\mathbf{x}_i \in \mathcal{D}_n},  \mathbf{K}_n=\left[k\left(\mathbf{x}_i, \mathbf{x}_j\right)\right]_{\mathbf{x}_i, \mathbf{x}_j \in \mathcal{D}_n},  \mathbf{y}_n=\left[y_i\right]_{i=1}^n$. The posterior variance is $\sigma^2_n(\mathbf{x})=k_n(\mathbf{x}, \mathbf{x})$.

\subsubsection{Acquisition Functions}\label{subsec:acquisition-funtion}
Based on the current GP posterior, different heuristic strategies have been proposed to develop different acquisition functions. Commonly used acquisition functions include Probability of Improvement (PI) \citep{kushner1964new}, Expected Improvement (EI) \citep{jones1998efficient}, Upper Confidence Bound (UCB) \citep{srinivas2009gaussian}, and Thompson Sampling (TS) \citep{agrawal2012analysis}. More recently, information-based acquisition functions have emerged as an effective class of methods, such as Predictive Entropy Search (PES) \citep{hernandez2014predictive}, Max-value Entropy Search (MES) \citep{wang2017max} and Joint Entropy Search (JES) \citep{hvarfner2022joint}. We refer readers to the survey of \citet{wang2023recent} for a comprehensive overview of acquisition functions.

\subsection{Conditional Diffusion Models}\label{subsec:CDM}

Diffusion models are a class of generative models that approximates complex data distributions $P(\mathbf{x})$ by progressively corrupting samples from $P(\mathbf{x})$ with noise and subsequently denoising it through a learned reverse process \citep{song2019generative, ho2020denoising}. Conditional Diffusion Models (CDMs) extend this framework to approximate a conditional distribution $P(\mathbf{x} \mid y)$, enabling controlled sample generation under a specified condition $y$, such as a textual prompt in image generation tasks or a target observation in optimization problems. In this work, we focus on score-based CDMs formulated through Stochastic Differential Equations (SDEs), as proposed by \citet{song2020score}. For clarification, hereafter we let $P$ denote the distribution, and $p$ the corresponding density. 

\subsubsection{Forward \& Backward Process}
To model the progressive corruption of data with noise, CDMs define a forward process as a continuous-time stochastic process $\{\mathbf{x}_t^y\}_{t\in(0,T]}$, where $\mathbf{x}_0^y$ are samples drawn from $P(\mathbf{x}\mid y)$, and $\mathbf{x}_t^y$ denotes the random variable obtained by perturbing $\mathbf{x}_0^y$ with noise up to time $t$. We denote by $P_t(\mathbf{x} \mid y)$ the distribution of $\mathbf{x}_t^y$, with $P_0(\mathbf{x} \mid y)=P(\mathbf{x} \mid y)$. Formally, the forward process can be represented by the following SDE with $t\in[0, T]$:
\begin{equation}\label{eq:forward_process}
\mathrm d \mathbf{x}_t^y = \mathbf{f}(\mathbf{x}_t^y,t)\,\mathrm dt + \sigma(t)\,\mathrm d \mathbf{w}_t
\end{equation}
where $\mathbf{f}:\mathbb{R}^d\times [0, T]\to \mathbb{R}^d$ is the drift coefficient, $\sigma:[0,T]\to \mathbb{R}$ is the diffusion coefficient, and $\mathbf{w}_t$ denotes a $d$-dimensional standard Brownian motion. The drift $\mathbf{f}$ and diffusion coefficient $\sigma$ are usually designed so that the terminal distribution $P_T(\mathbf{x}\mid y)$ converges to a simple and tractable distribution $\Pi$, such as a standard Gaussian. 

Backward process aims to reverse the forward process, with $\mathbf{x}_t^{y,\leftarrow}$ following $P_{T-t}(\mathbf{x}\mid y)$. Specifically, starting from a sample $\mathbf{x}_0^{y,\leftarrow}$ drawn from the terminal distribution $P_T(\mathbf{x}\mid y)$, backward process is a reverse-time stochastic process that gradually removes noise such that $P_t(\mathbf{x}\mid y)$ evolves toward the target distribution, i.e., $P_t(\mathbf{x}\mid y)\to P_0(\mathbf{x}\mid y)$ as $t\to 0$. The backward process corresponding to Eq.~\ref{eq:forward_process} can be characterized by another SDE with $t\in[0, T]$ \citep{anderson1982reverse}:
\begin{equation}\label{eq:backward_process}
\mathrm{d} \mathbf{x}_t^{y, \leftarrow}=\left[\mathbf{f}\left(\mathbf{x}_t^{y, \leftarrow}, t\right)-\sigma(t)^2 \nabla_{\mathbf{x}} \log p_t\left(\mathbf{x}_t^{y, \leftarrow} \mid y\right)\right] \mathrm{d} t + \sigma(t) \mathrm{d} \overline{\mathbf{w}}_t,
\end{equation}

where $\nabla_{\mathbf{x}} \log p_t(\mathbf{x}_t^{y,\leftarrow}\mid y)$ is known as the conditional score function of the marginal density at time $t$, and $\bar{\mathbf{w}}_t$ is a $d$-dimensional reverse-time Brownian motion. 

\subsubsection{Training}

In practice, the unknown conditional score function is approximated by a neural network $\mathbf{s}_{\boldsymbol{\theta}}(\mathbf{x},t,y)$ with parameters $\boldsymbol{\theta}$, which can be trained via minimizing a loss function called denoising score matching \citep{vincent2011connection, song2020score}:
\begin{equation}\label{eq:denoising_score_matching}
    \underset{t}{\mathbb{E}}\left[ \underset{\mathbf{x}_0^y, y}{\mathbb{E}}\left[ \underset{\mathbf{x}_t^y \mid \mathbf{x}_0^y}{\mathbb{E}}\left[\left\|\mathbf{s}_{\boldsymbol{\theta}}\left(\mathbf{x}_t^y, t, y\right)-\nabla_{\mathbf{x}} \log p_t\left(\mathbf{x}_t^y \mid \mathbf{x}_0^y\right)\right\|_2^2\right]\right]\right]
\end{equation}
where $t \sim \mathcal{U}(0,T)$, and $(\mathbf{x}_0^y,y)$ are Monte Carlo samples drawn from the given training dataset. The expectation over $\mathbf{x}_t$ is taken with respect to the forward transition kernel
$p_t(\mathbf{x}_t^y \mid \mathbf{x}_0^y)$. This objective is tractable since, under commonly used forward SDEs, the transition kernel admits a closed form. More details are provided in Appendix~\ref{app:transition-kernel}. 

After training, the learned score predictor $\widehat{\mathbf{s}}_{\boldsymbol{\theta}}(\mathbf{x}_t^{y,\leftarrow},t,y)$ is used to replace the unknown conditional score $\nabla_{\mathbf{x}} \log p_t(\mathbf{x}_t^{y,\leftarrow}\mid y)$ in the backward SDE, enabling conditional sampling by simulating the SDE. We denote by $\widehat{P}_{T-t}(\mathbf{x} \mid y)$ the distribution of the random variable $\mathbf{x}_t^{y,\leftarrow }$. 

Notably, in practice, rather than the terminal distribution $P_T$, the backward process starts from the limiting tractable distribution $\Pi$. Additionally, we need an early stopping time $t_0>0$ close to zero and sample $t$ uniformly from $\left[t_0, T\right]$ during training. This choice is made for numerical stability, as the score $\nabla_{\mathbf{x}} \log p_t\left(\mathbf{x}_t \mid \mathbf{x}_0, y\right)$ becomes ill-conditioned and may diverge as $t \rightarrow 0$ \citep{song2020score, li2024diffusion}. For the same reason, the backward SDE is only simulated from $0$ to $T-t_0$. Finally, the CDM effectively learns the distribution $\widehat{P}_{t_0}(\mathbf{x}\mid y)$, which is used in practice as an approximation to the target distribution $P(\mathbf{x}\mid y)$. Following the notation introduced in Section~\ref{sec:intro}, we simply denote this learned distribution by $\widehat{P}(\mathbf{x}\mid y)$, and by $\widehat{p}(\mathbf{x}\mid y)$ the corresponding density. 

It is also worth noting that, to enhance the fidelity of generated samples with respect to the given condition $y$, Classifier-Free Guidance (CFG) \citep{ho2022classifier} has been proposed as an effective engineering technique, and further details are provided in Appendix~\ref{app:classifier-free-guidance}. 

\section{Methodology}\label{sec:methodology}
\begin{algorithm}[t]
\caption{Bayesian Optimization with Diffusion-based Mode Seeking}
\label{alg:DMS}
\begin{algorithmic}[1]
\STATE {\bfseries Input:} dataset $\mathcal{D}_n$, evaluation budget $B$, GP prior, score predictor $\mathbf{s}_{\boldsymbol{\theta}}$, surrogate $\widehat f$
\STATE {\bfseries Output:} final dataset $\mathcal{D}_{n+B}$

\FOR{$j=n$ to $n+B-1$}
    \STATE $\mu_j(\mathbf{x}), \sigma_j(\mathbf{x}) \leftarrow \mathrm{FitGP}(\mathcal{D}_j)$
    \STATE $\{\mathbf{x}_i\}_{i=1}^{m} \leftarrow \mathrm{ShortRunL\text{-}BFGS}(\mathrm{GP}(\mu_j(\mathbf{x}),\sigma_j(\mathbf{x})))$
    \STATE $\{\widehat{y}_i\}_{i=1}^{m} \leftarrow \widehat f(\{\mathbf{x}_i\}_{i=1}^{m})$, \quad
    $\widehat{\mathcal{D}}_m \leftarrow \{(\mathbf{x}_i,\widehat{y}_i)\}_{i=1}^{m}$
    \STATE $\widehat{\mathbf{s}}_{\boldsymbol{\theta}}(\mathbf{x},t,y) \leftarrow \mathrm{DenoisingScoreMatching}(\widehat{\mathcal{D}}_m)$
    \STATE $\mathbf{x}_{j+1} \leftarrow \mathrm{DMS}(\widehat{\mathbf{s}}_{\boldsymbol{\theta}},\widehat{\mathcal{D}}_m)$
    \STATE $y_{j+1} \leftarrow f(\mathbf{x}_{j+1})+\epsilon_{j+1}$,
    \quad $\mathcal{D}_{j+1} \leftarrow \mathcal{D}_j\cup\{(\mathbf{x}_{j+1},y_{j+1})\}$
\ENDFOR

\vspace{0.5em}
\STATE {\bfseries Function} $\mathrm{DMS}(\widehat{\mathbf{s}}_{\boldsymbol{\theta}},\widehat{\mathcal{D}}_m)$
\STATE $\widehat{y}^{\star} \leftarrow \max_{i\in[m]} \widehat{y}_i$
\FOR{$s=1$ to $S$}
    \STATE $\mathbf{x}^{\leftarrow,(s)}_0 \sim \Pi$, e.g., $\mathcal{N}(\mathbf{0},\mathbf{I})$
    \STATE $\mathbf{x}^{\star}_s \leftarrow \mathrm{BackwardSDE}(\mathbf{x}^{\leftarrow,(s)}_0,\widehat{\mathbf{s}}_{\boldsymbol{\theta}},\widehat{y}^{\star})$
\ENDFOR
\STATE $\mathbf{x}_{\mathrm{next}} \leftarrow \mathrm{MeanShift}(\{\mathbf{x}^{\star}_s\}_{s=1}^{S})$
\STATE {\bfseries return} $\mathbf{x}_{\mathrm{next}}$
\end{algorithmic}
\end{algorithm}

\vspace{-0.5em}

In this section, we first introduce the strategies that could effectively address the two challenges of training CDMs in BO as introduced in Section~\ref{sec:intro}. Secondly, we propose our acquisition strategy \textit{Diffusion-based Mode Seeking} (DMS) based on the properly trained CDM. 

\subsection{Training strategies}\label{subsec:solutions}
\subsubsection{Strategy for Limited Training Data}
The first challenge in applying CDMs to BO is the extremely limited size of the observed dataset $\mathcal{D}_n$. Since CDMs aim to approximate a conditional distribution $P(\mathbf{x}\mid y)$, training them directly on $\mathcal{D}_n$ is ineffective when $n$ is small. This data scarcity prevents the model from capturing meaningful conditional structure and often leads to unstable or degenerate generation.

An intuitive remedy is to augment the training dataset via pseudo-labeling, thereby constructing additional labeled pairs beyond the expensive evaluations of $f$. In the context of BO, the GP surrogate provides a natural and principled model $\widehat{f}$ of the unknown objective function $f$.

To obtain a predictive estimate of the objective function $f$ that accounts for both the posterior mean and uncertainty, we adopt the following $\widehat{f}(\mathbf{x})$:
\begin{equation}\label{eq:pseudo-labels}
\widehat f(\mathbf{x}) = \mu_n(\mathbf{x}) + \rho\sigma_n(\mathbf{x}),
\end{equation}
where $\mathbf{x}$ is an input design that has not been evaluated yet. Then we assign the pseudo-label by $\widehat{y}(\mathbf{x})=\widehat{f}(\mathbf{x})$. 


This strategy assigns larger values to input designs that either exhibit high posterior mean or high posterior uncertainty. As a result, regions that are potentially optimal or insufficiently explored are both emphasized during the construction of the training data. When training CDMs on such pseudo-labels, the learned conditional distribution encourages candidate generation toward these regions, enabling the model to balance exploitation and exploration simultaneously. Therefore, we refer to it as balance-aware pseudo-labeling strategy. 

Of note, another naive strategy is simply the GP regression estimator, i.e. $\widehat{f}(\mathbf{x})=\mu_n(\mathbf{x})$. While simple, this approach neglects posterior uncertainty and thus provides purely exploitative information, which may restrict candidate generation to already explored regions and hinder effective exploration. 

\subsubsection{Strategy for Identifying Promising Inputs}

The second challenge is how to identify input designs associated with high balance-aware pseudo-label values, which are essential for providing informative training signals to CDMs. High pseudo-labels are intended to indicate the potential optimality of an input and guide CDMs toward regions likely to contain high-quality $\mathbf{x}^{\star}$ candidates. If pseudo-labels are assigned to arbitrary input designs, the resulting pseudo-label values are unlikely to be high, especially as the input dimensionality $d$ increases. Training CDMs on such data therefore provides weak signals, which in turn limits their ability to generate high-quality $\mathbf{x}^{\star}$ candidates.

To address this issue, we explicitly guide the selection of input designs toward promising regions. At each iteration, we first generate a set of initial inputs $\{\mathbf{x}_i^{(0)}\}_{i=1}^m$ using a Sobol sequence over the input space $\mathcal{X}=[\mathbf{l}, \mathbf{u}]$. Starting from each $\mathbf{x}^{(0)}_i$, we apply $K$ iterations of short-run L-BFGS to locally refine the inputs toward regions with high pseudo-labels:
\begin{equation}\label{eq:L-BFGS}
    \mathbf{x}_i^{(k+1)}=\Pi_{\mathcal{X}}\left(\mathbf{x}_i^{(k)}+\eta_i^{(k)} \mathbf{d}_i^{(k)}\right), \quad k=0, \ldots, K-1,
\end{equation}
where $\mathbf{d}_i^{(k)}$ is the L-BFGS search direction computed from $\nabla \widehat{f}\left(\mathbf{x}_i^{(k)}\right)$, $\eta_i^{(k)}$ is determined by a strong Wolfe line search method \citep{wright1999numerical}, and $\Pi_\mathcal{X}$ denotes projection onto $\mathcal{X}$, i.e., $\Pi_{\mathcal{X}}(\mathbf{x})=\min \{\max \{\mathbf{x}, \mathbf{l}\}, \mathbf{u}\}$. 

Here, “short-run” indicates that the L-BFGS procedure is deliberately truncated after a small $K$, so that the refined inputs explore multiple basins of attraction rather than collapsing to a single local optimum. At the same time, the resulting inputs $\{\mathbf{x}_i\}_{i=1}^m$ remain concentrated around promising input regions with high balance-aware pseudo-labels, providing diverse and informative data pairs that capture local structures. The resulting $\{\mathbf{x}_i^{(K)}\}_{i=1}^m$ are subsequently assigned balance-aware pseudo-labels, forming the pseudo-dataset $\widehat{\mathcal{D}}_m=\{(\mathbf{x}_i, \widehat{y}_i)\}_{i=1}^m$ for the following training of CDMs. 

In contrast, using only the initial Sobol sequence $\{\mathbf{x}_i^{(0)}\}_{i=1}^m$ in the input space without refinement, provides broad coverage but insufficient focus on promising regions. This often leads to less-informative pseudo-labels and degrade the performance of CDM.

\subsection{Diffusion-based Mode Seeking}\label{sec:DMS}
With the pseudo-dataset $\widehat{\mathcal{D}}_m$ constructed in Section~\ref{subsec:solutions}, we can subsequently train the CDM by the denoising score matching. Notably, since the CDM is solely trained on the pseudo-dataset $\widehat{\mathcal{D}}_m$, the distribution learned by CDM should be understood as an approximation to $P(\mathbf{x}\mid \widehat{y})$ rather than the ground truth $P(\mathbf{x}\mid y)$. Hence, hereafter we denote by $\widehat{P}(\mathbf{x}\mid \widehat{y})$ the distribution learned by the CDM, and $\widehat{p}(\mathbf{x}\mid \widehat{y})$ the density. 

After training, the maximal pseudo-label in $\widehat{\mathcal{D}}_m$, denoted by $\widehat{y}^{\star}$, typically serves as a reasonable approximation of the highest attainable balance-aware pseudo-label and, consequently, of the optimal function value. Hence, $\widehat{P}\left(\mathbf{x} \mid \widehat{y}=\widehat{y}^{\star}\right)$ can be interpreted as the distribution of $\mathbf{x}^{\star}$ learned by the CDM.

The remaining question is how to select the next evaluation point. A natural idea is to follow ES/PES by using the CDM to generate candidates for $\mathbf{x}^{\star}$ and computing the corresponding entropy. However, the motivation behind ES/PES relies on the observation that the GP-induced distribution over $\mathbf{x}^{\star}$ is typically highly dispersed and uncertain, making entropy reduction a meaningful objective. In contrast, we empirically observe that the distribution learned by the CDM is more sharply concentrated than GP-induced distribution. This difference is illustrated in Appendix~\ref{app:distribution-difference}, where we visualize the corresponding distributions in a two-dimensional setting. 

These observations motivate us to design a new acquisition strategy tailored to the CDM-learned distribution. While dispersed GP-induced distribution necessitates entropy reduction, the sharper concentration of the CDM-learned distribution provides a more reliable and confident indicator of the location of global optimum. Consequently, we propose selecting the mode of the learned density $\widehat{p}(\mathbf{x}\mid \widehat{y}=\widehat{y}^{\star})$, which represents the location that the CDM deems most likely to correspond to $\mathbf{x}^{\star}$, as the next evaluation point. 

However, since $\widehat{P}$ is available only through samples rather than an explicit density expression, we first draw $\{\mathbf{x}^{\star}_s\}_{s=1}^S$ from $\widehat{P}(\mathbf{x}\mid \widehat{y}=\widehat{y}^{\star})$ via the backward SDE, and then apply mean-shift clustering \citep{comaniciu2002mean} to estimate the dominant mode of the empirical sample distribution. The resulting mode estimate is used as the next evaluation point, and we refer to this acquisition strategy as \textit{Diffusion-based Mode Seeking} (DMS), as summarized in Algorithm~\ref{alg:DMS}. Introduction to mean-shift is provided in Appendix~\ref{app:mean-shift}.

\begin{figure*}[t]
\centering

\begin{minipage}{\textwidth}
\centering
  \begin{minipage}{0.235\linewidth}
    \centering
    \includegraphics[width=\linewidth]{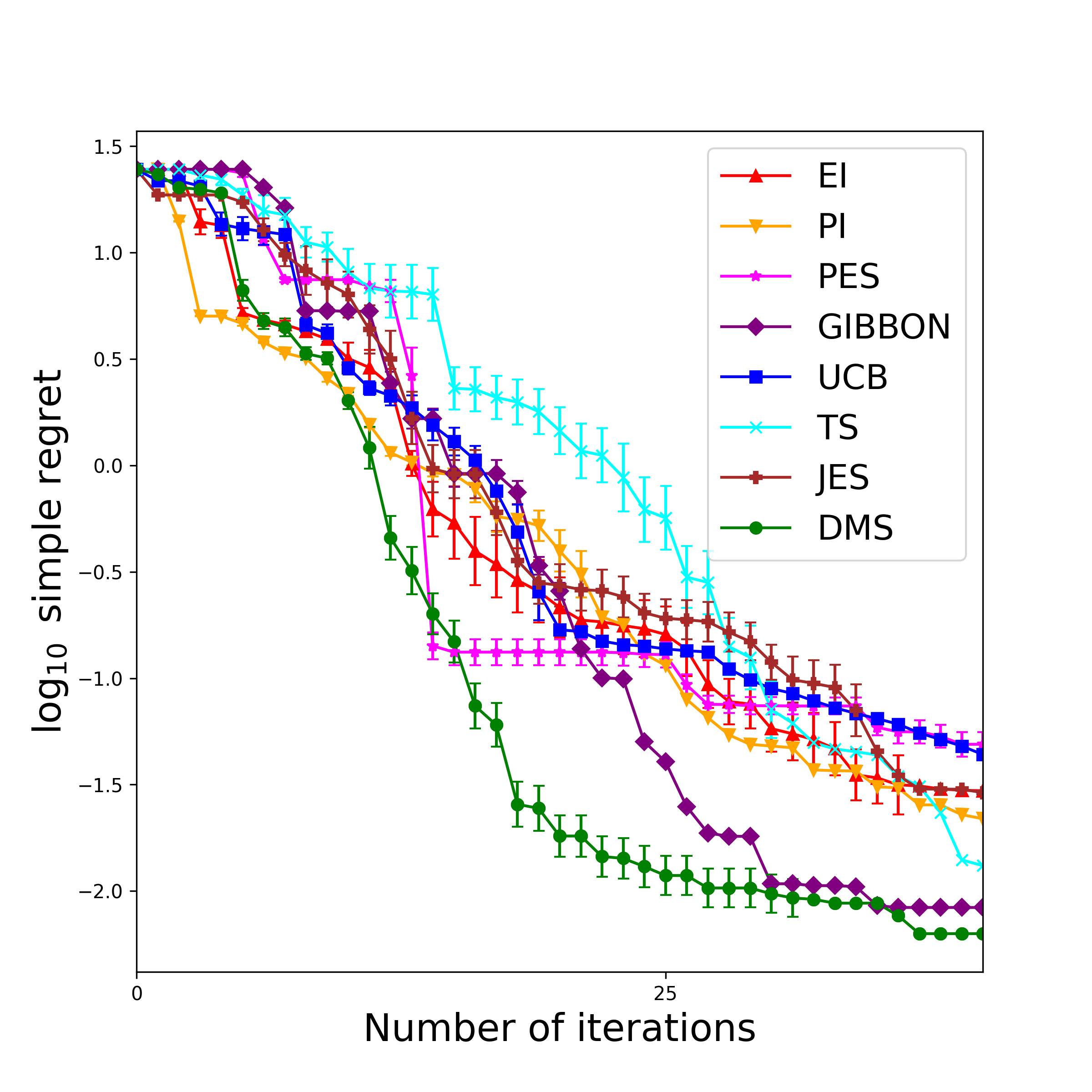}
    \small (a) Styblinski-Tang $(d=2)$
  \end{minipage}\hspace{0.01\linewidth}
  \begin{minipage}{0.235\linewidth}
    \centering
    \includegraphics[width=\linewidth]{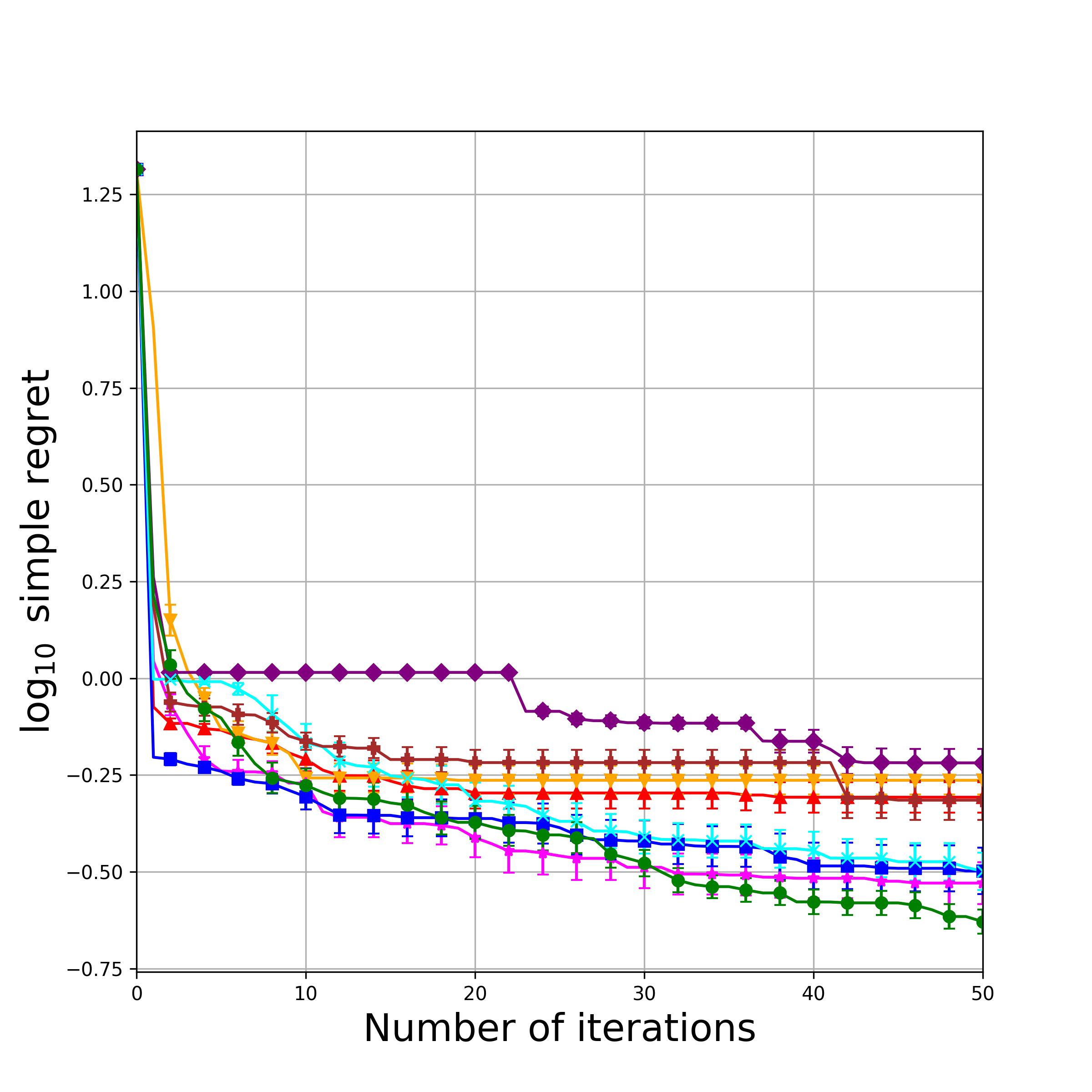}
    \small (b) Griewank $(d=3)$
  \end{minipage}\hspace{0.01\linewidth}
  \begin{minipage}{0.235\linewidth}
    \centering
    \includegraphics[width=\linewidth]{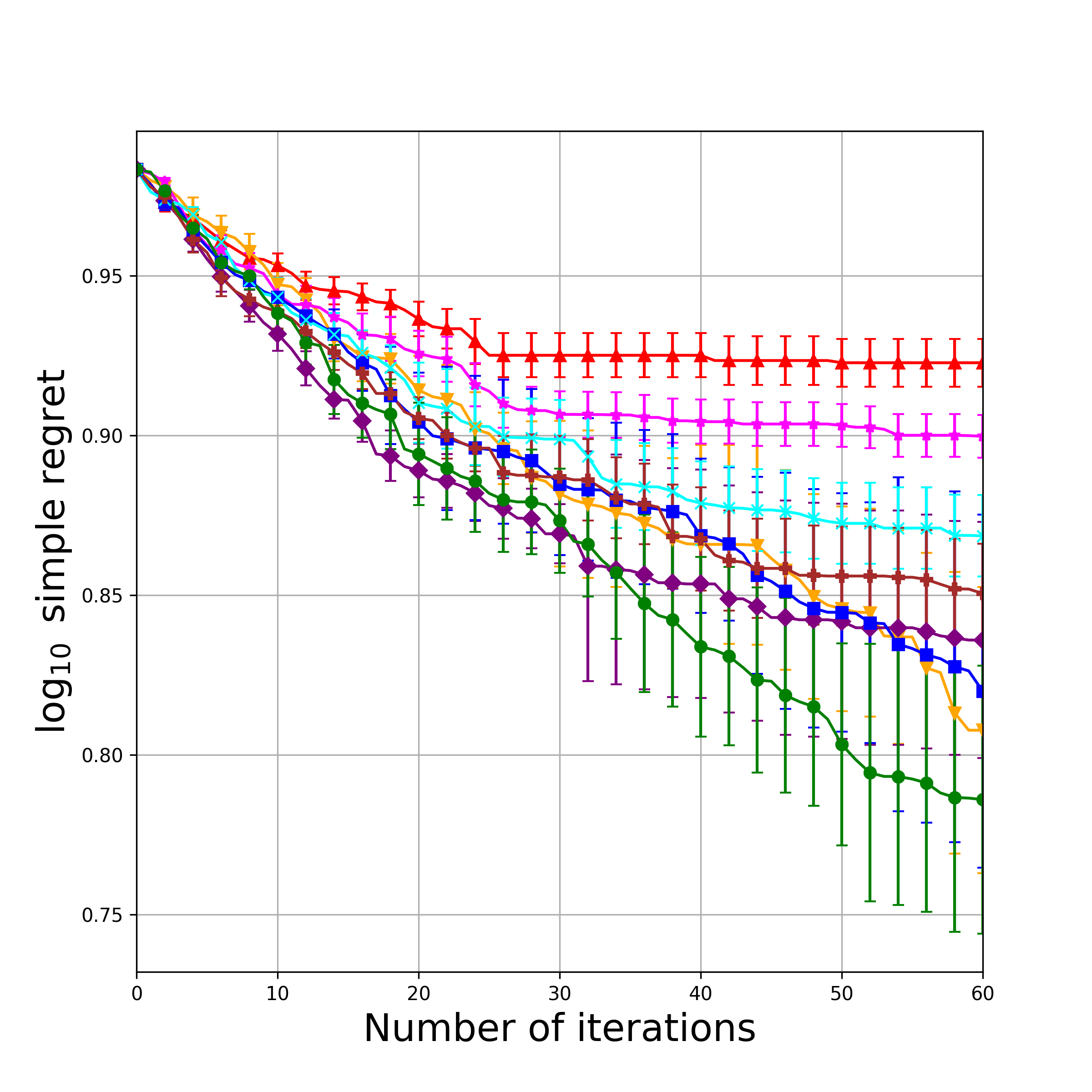}
    \small (c) Shekel $(d=4)$
  \end{minipage}\hspace{0.01\linewidth}
  \begin{minipage}{0.235\linewidth}
    \centering
    \includegraphics[width=\linewidth]{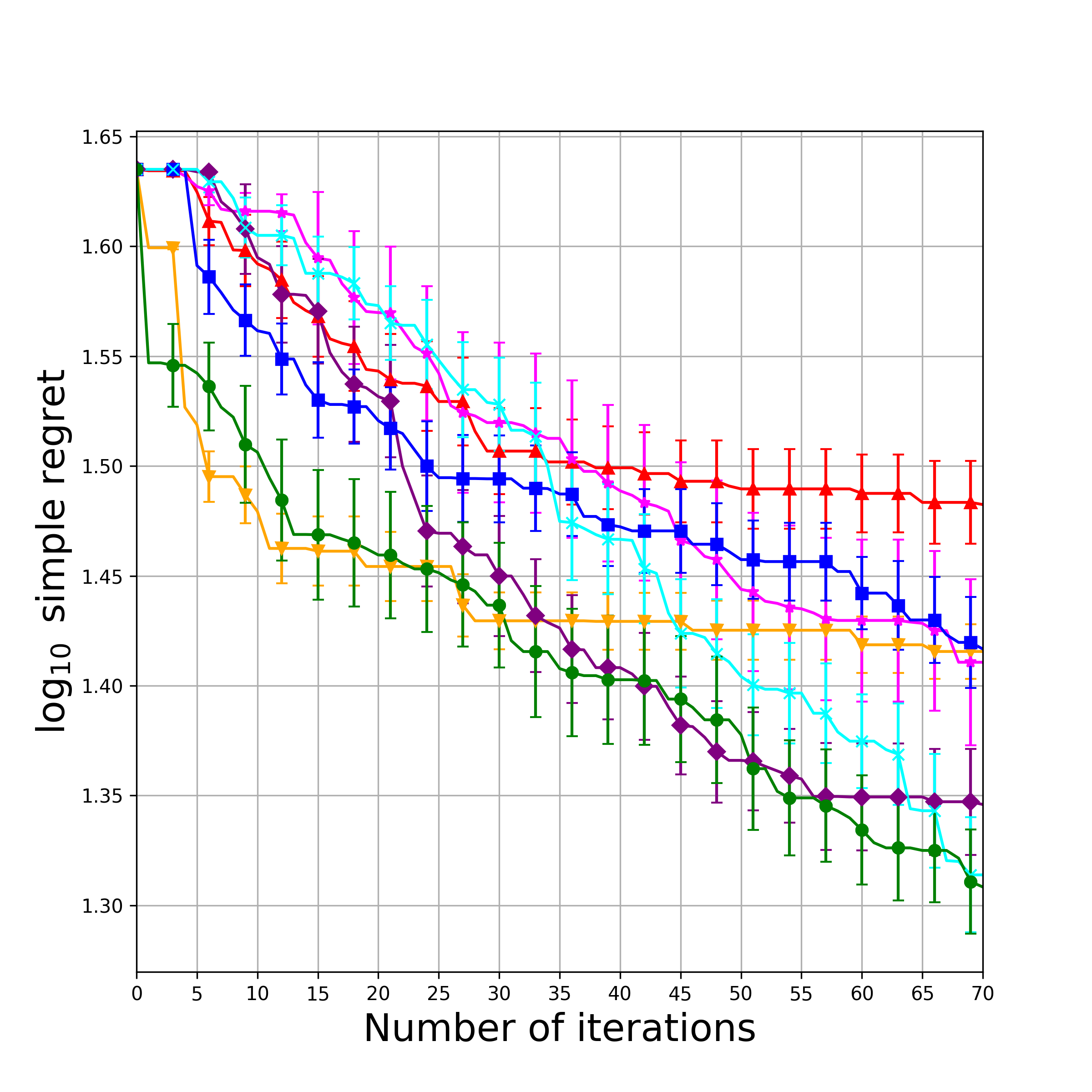}
    \small (d) Rastrigin $(d=5)$
  \end{minipage}
\end{minipage}

\begin{minipage}{\textwidth}
\centering
  \begin{minipage}{0.235\linewidth}
    \centering
    \includegraphics[width=\linewidth]{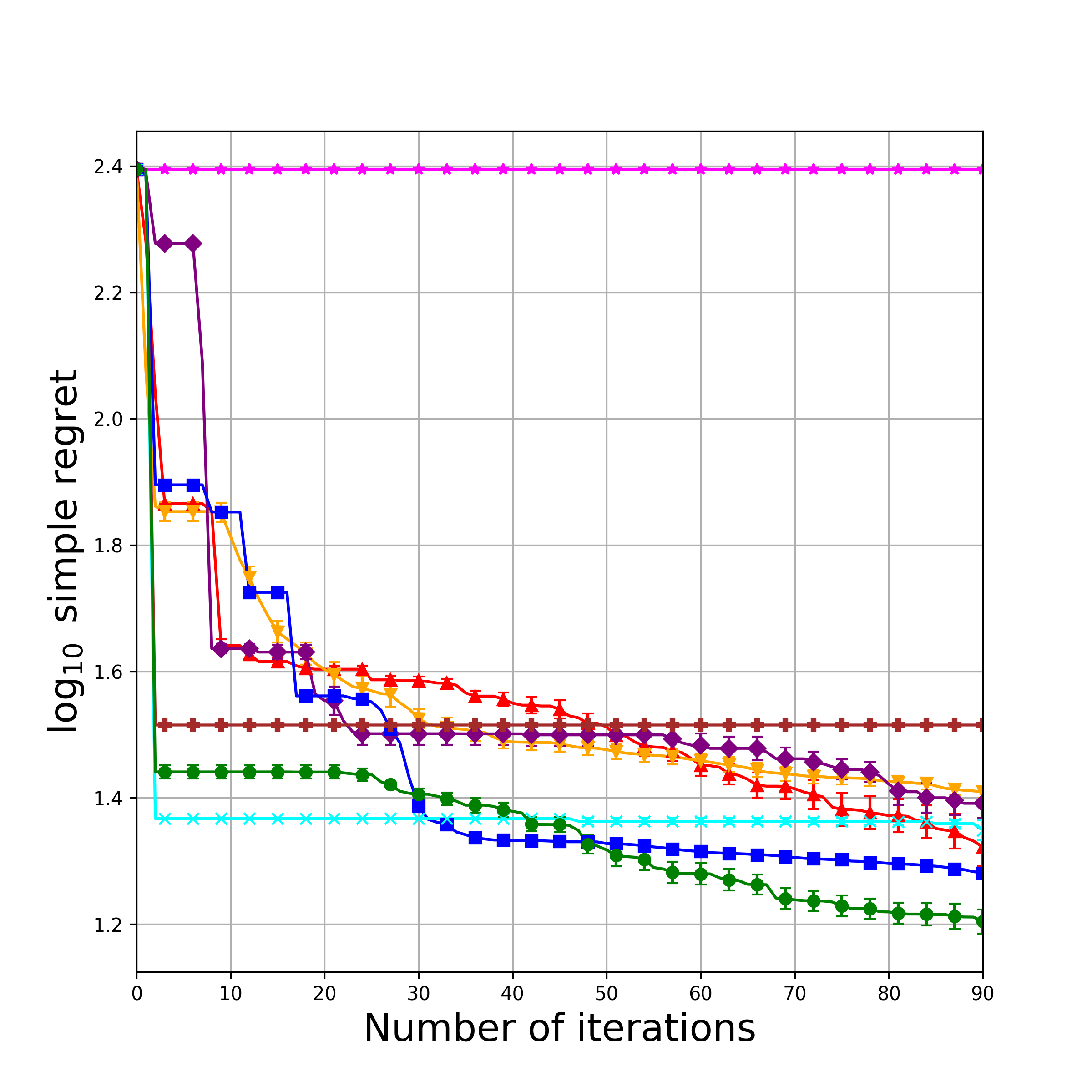}
    \small (e) Rosenbrock $(d=7)$
  \end{minipage}\hspace{0.01\linewidth}
  \begin{minipage}{0.235\linewidth}
    \centering
    \includegraphics[width=\linewidth]{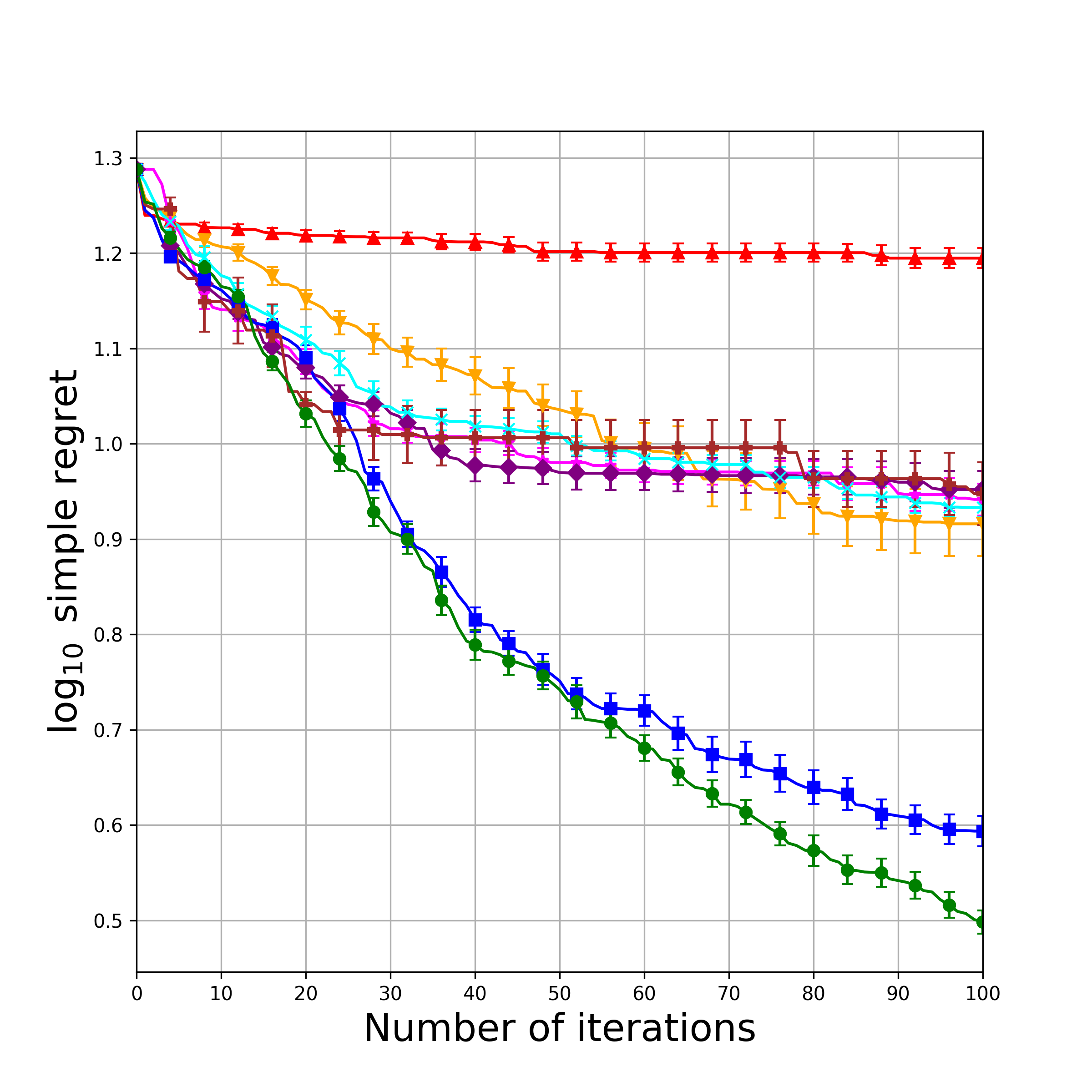}
    \small (f) Ackley $(d=8)$
  \end{minipage}\hspace{0.01\linewidth}
  \begin{minipage}{0.235\linewidth}
    \centering
    \includegraphics[width=\linewidth]{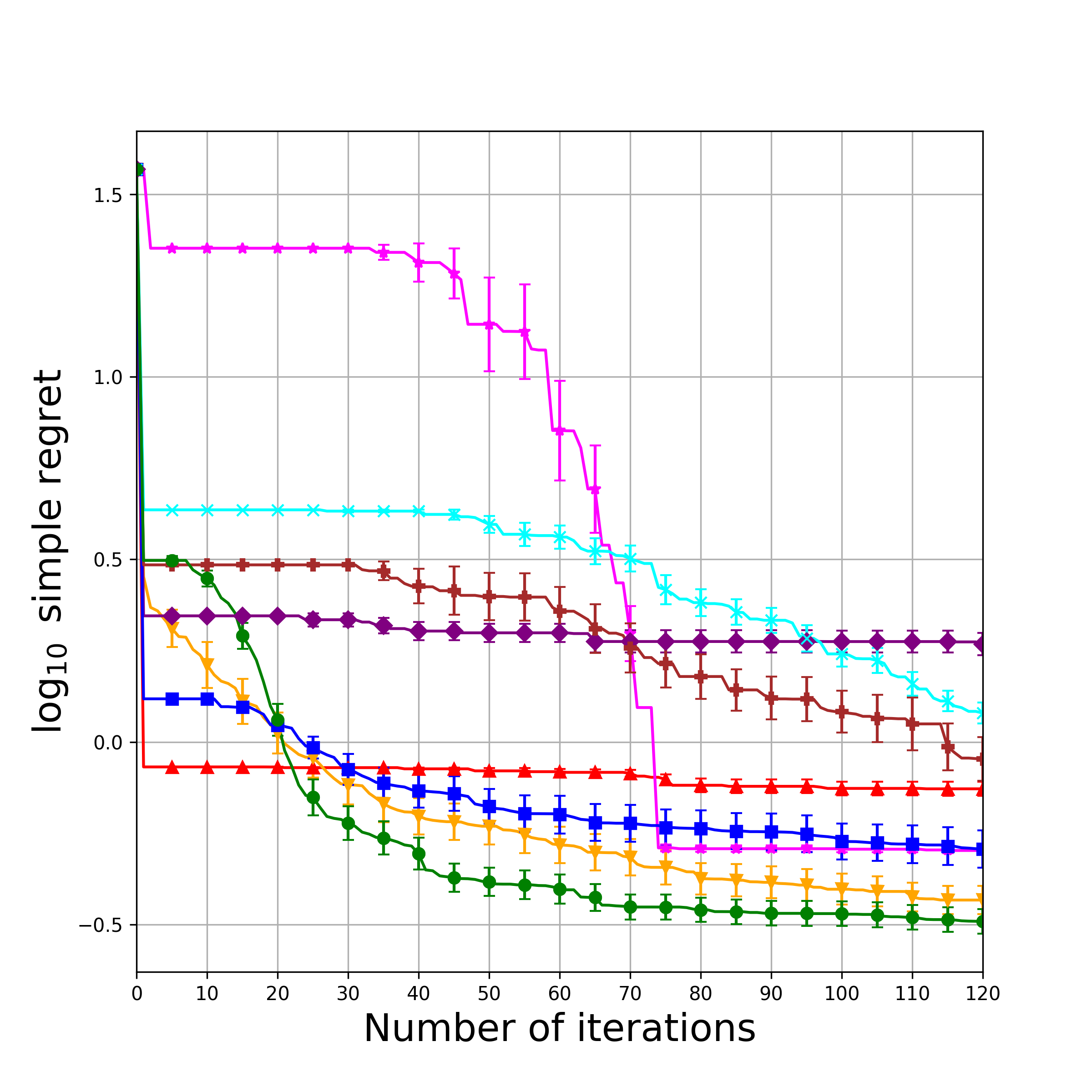}
    \small (g) Levy $(d=10)$
  \end{minipage}\hspace{0.01\linewidth}
  \begin{minipage}{0.235\linewidth}
    \centering
    \includegraphics[width=\linewidth]{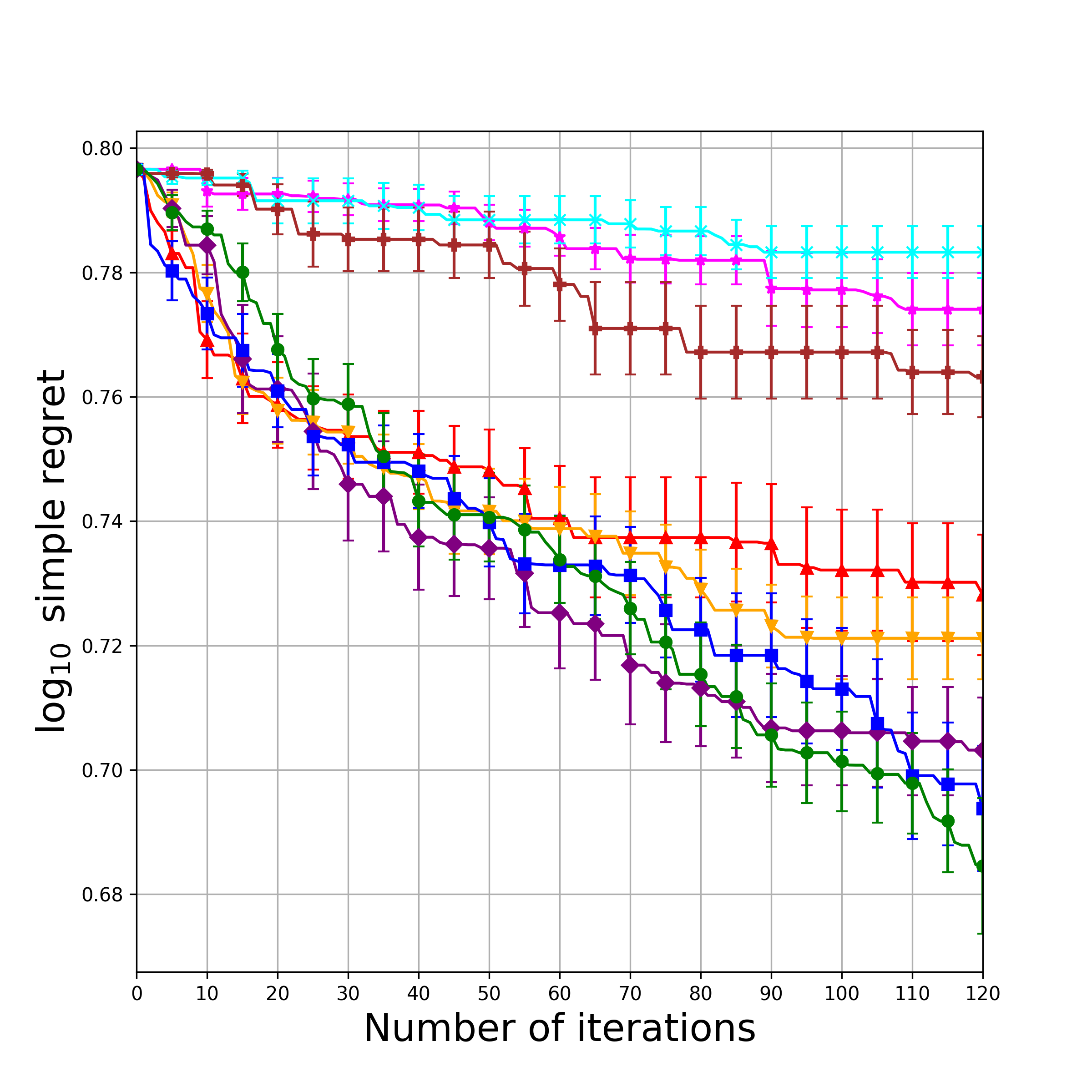}
    \small (h) Michalewicz $(d=10)$
  \end{minipage}
\end{minipage}

\begin{minipage}{\textwidth}
\centering
  \begin{minipage}{0.235\linewidth}
    \centering
    \includegraphics[width=\linewidth]{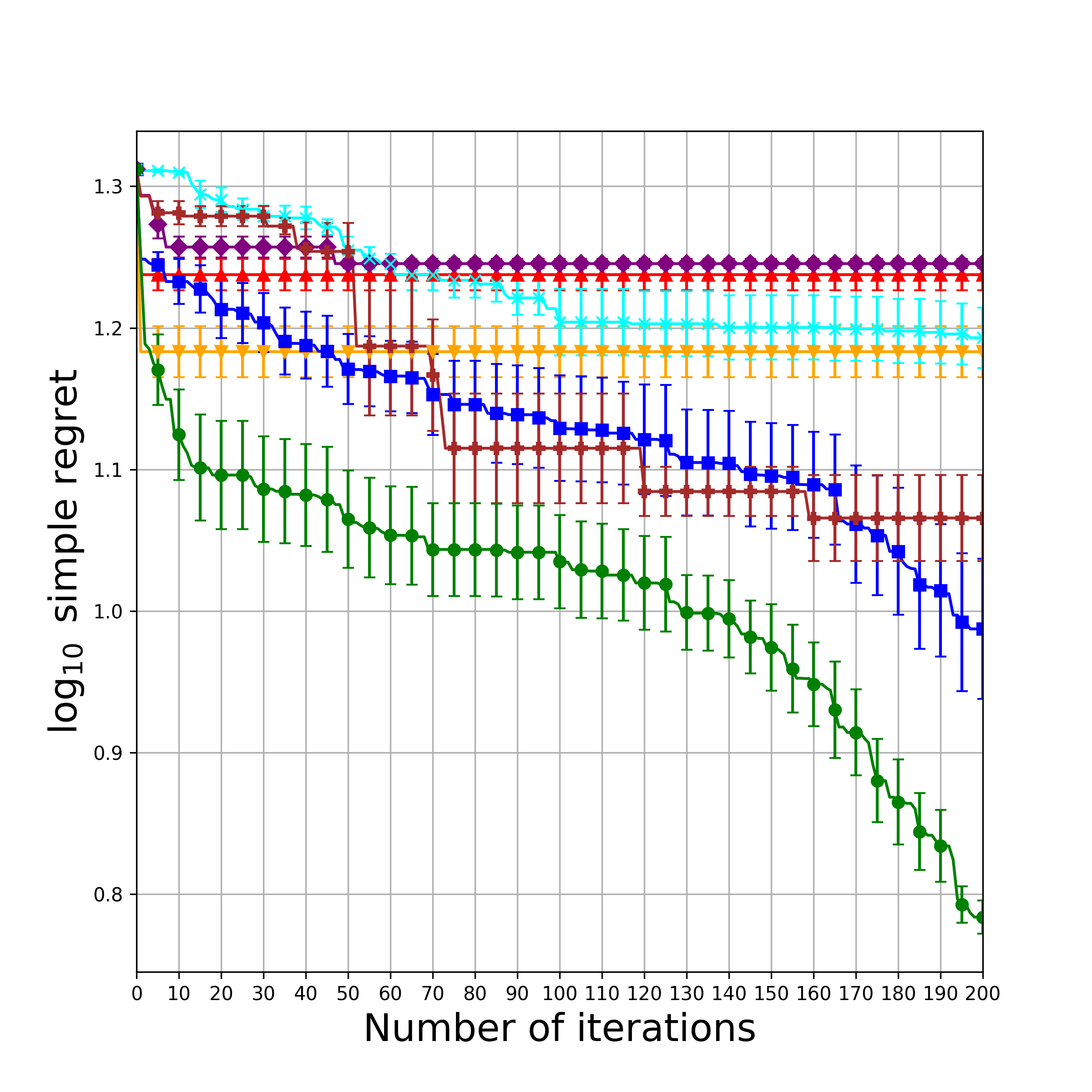}
    \small (i) Ackley $(d=20)$
  \end{minipage}\hspace{0.01\linewidth}
  \begin{minipage}{0.235\linewidth}
    \centering
    \includegraphics[width=\linewidth]{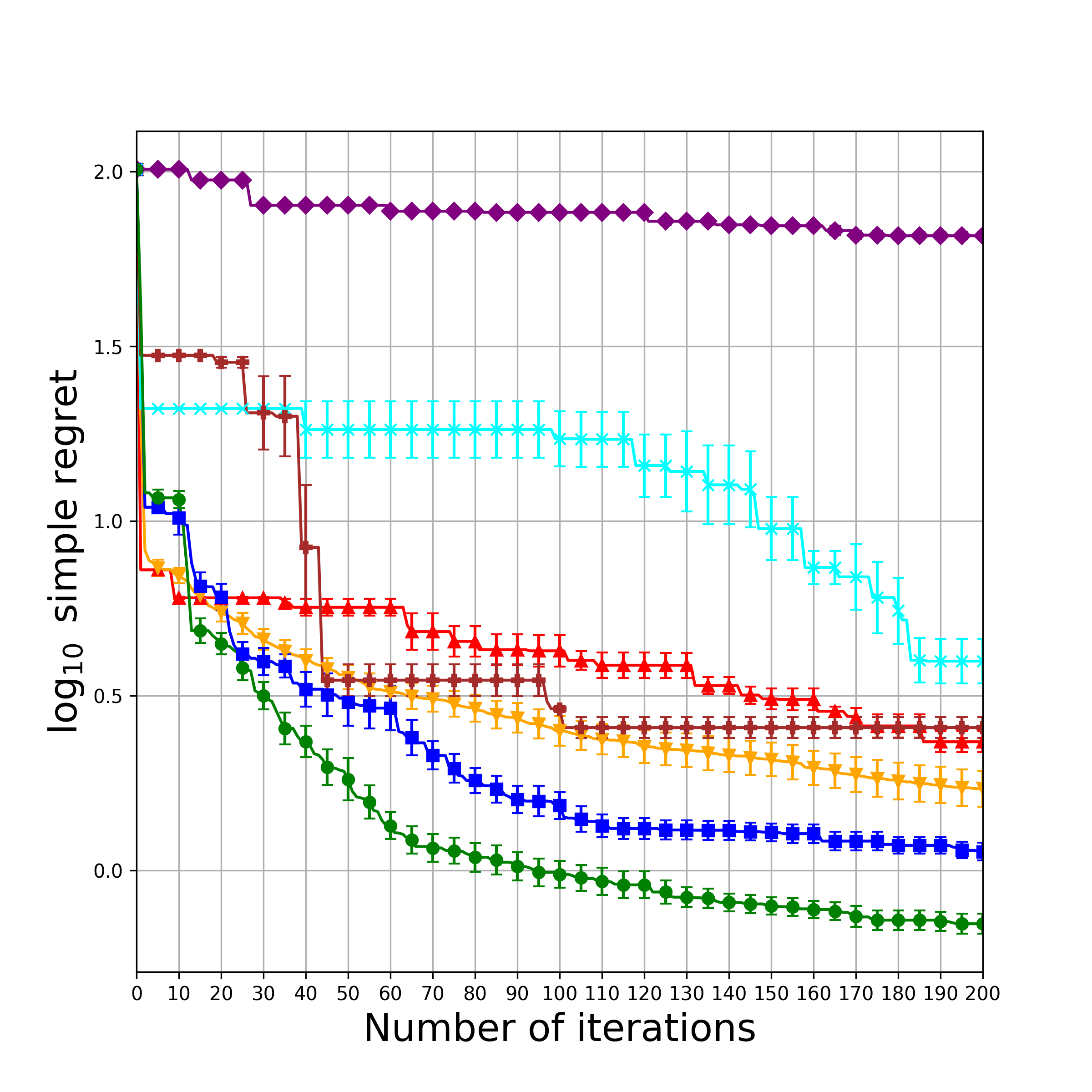}
    \small (j) Levy $(d=20)$
  \end{minipage}\hspace{0.01\linewidth}
  \begin{minipage}{0.235\linewidth}
    \centering
    \includegraphics[width=\linewidth]{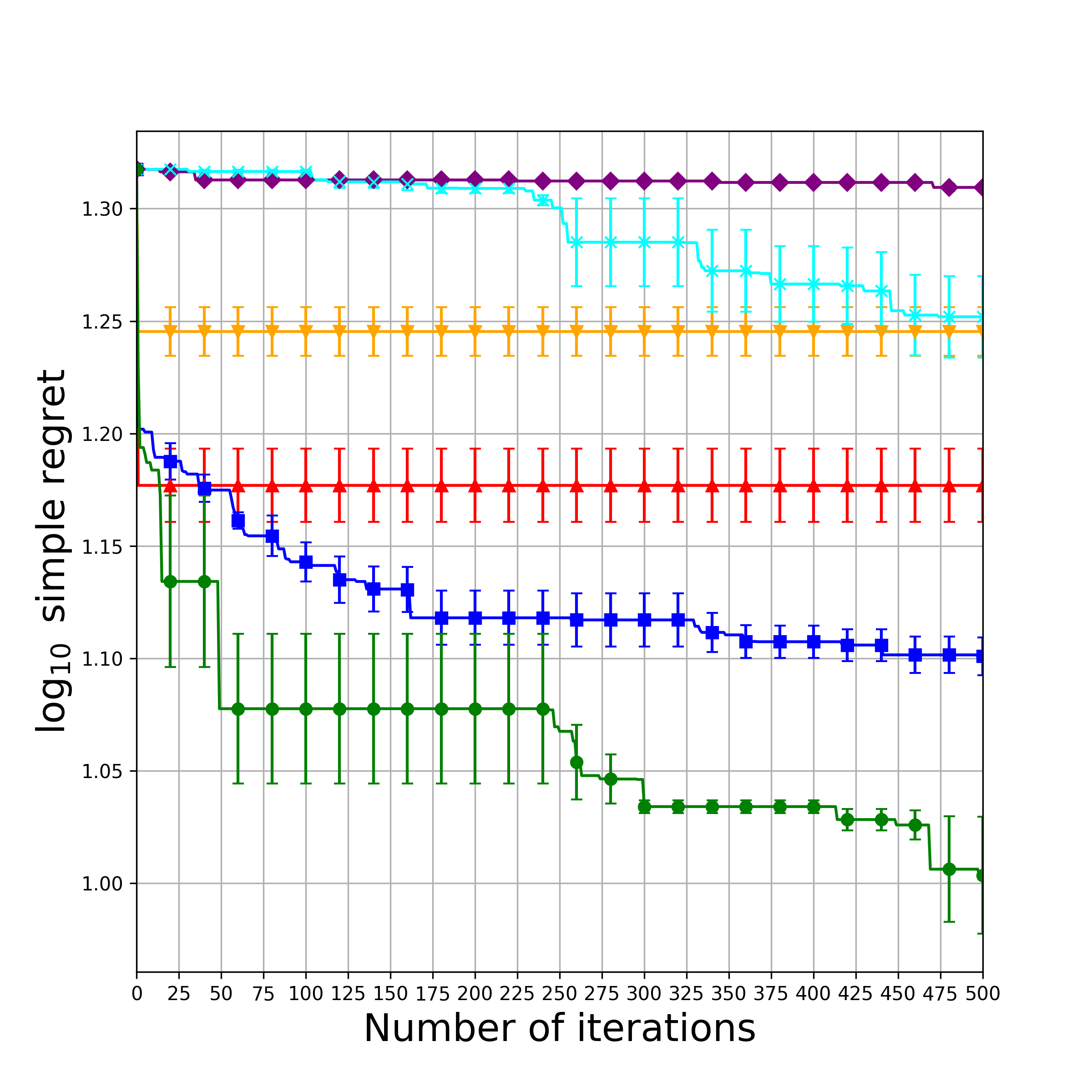}
    \small (k) Ackley $(d=50)$
  \end{minipage}\hspace{0.01\linewidth}
  \begin{minipage}{0.235\linewidth}
    \centering
    \includegraphics[width=\linewidth]{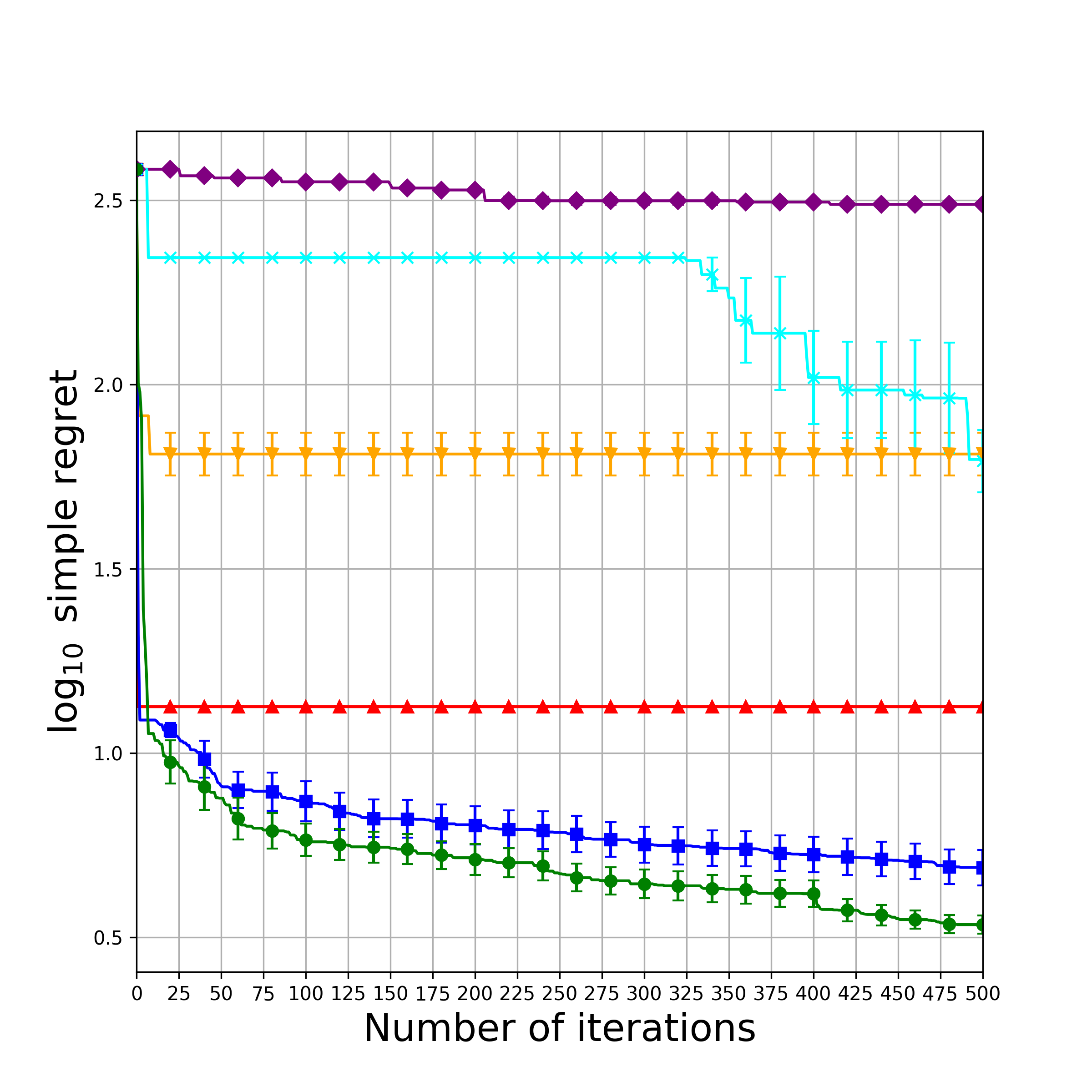}
    \small (l) Levy $(d=50)$
  \end{minipage}
\end{minipage}

\caption{Comparison between DMS and other baselines on synthetic benchmark functions. For tasks with $d \le 10$, experiments are conducted with 30 macro-replications; for higher-dimensional tasks with $d>10$, experiments are conducted with 10 macro-replications, with mean and one standard error reported.}
\label{fig:synthetic_simple_regret}
\end{figure*}

\section{Theoretical Analysis}\label{sec:theoretical_analysis}
In this section, we analyze the sub-optimality of the CDM-learned distribution. We begin by briefly recapping notations aforementioned, and introducing some new notations that will be used. 

\paragraph{Notations} Let $P(\mathbf{x} \mid y)$ denote the ground-truth conditional distribution induced by the unknown black-box function $f$, and let $P_{\mathbf{x} y}$ denote the corresponding joint distribution, from which the observed data $\left(\mathbf{x}_i, y_i\right) \in \mathcal{D}_n$ are drawn. We denote by $P_{\mathbf{x}}$ a reference distribution over $\mathcal{X}$. While the input $\mathbf{x}$ can be anywhere in the input space in our setting, assuming that $\mathbf{x}\sim P_{\mathbf{x}}$ provides a convenient probabilistic framework for the theoretical analysis. To facilitate the analysis, we introduce a surrogate-induced conditional distribution $P(\mathbf{x} \mid \widehat{y})$, where $\widehat{y}=\widehat{f}(\mathbf{x})+\xi$ with $\xi \sim \mathcal{N}\left(0, \nu^2\right)$. Note that $\xi$ is introduced solely for theoretical purposes. We further denote by $P_{\mathrm{x} \widehat{y}}$ the corresponding joint distribution, from which the pseudo-dataset $\left(\mathbf{x}_i, \widehat{y}_i\right) \in \widehat{\mathcal{D}}_m$ is generated. Finally, we denote by $\widehat{P}(\mathbf{x}\mid \widehat{y})$ the distribution learned by the CDM, which is a direct approximation to $P(\mathbf{x}\mid \widehat{y})$, and indirectly related to the ground truth distribution $P(\mathbf{x}\mid y)$.

\begin{definition}\label{def:sub-optimality}
Given the condition value $a$, we define the sub-optimality of the distribution $\widehat{P}(\mathbf{x}\mid \widehat{y}=a)$ learned by the CDM as 
\begin{equation}
\textrm{SubOpt}(\widehat{P}_a;a)=a-\mathbb{E}_{\mathbf{x}\sim \widehat{P}_a}[f(\mathbf{x})],
\end{equation}
\end{definition}
where $\widehat{P}_a$ is short for $\widehat{P}(\mathbf{x}\mid \widehat{y}=a)$. 

\begin{theorem}\label{thm:sub-optimality}
    Training CDM under Assumptions~\ref{asm:RKHS-f}, ~\ref{asm:realizability},~\ref{asm:Novikov-condition} and~\ref{asm:lipschitz} gives rise to
\begin{equation}
\begin{aligned}
    \mathrm{SubOpt}(\widehat{P}_a;a)&\le \mathcal{E}_1+\mathcal{E}_2\\
\end{aligned}
\end{equation}
where with high probability, 
\begin{equation}
    \mathcal{E}_1=\widetilde{\mathcal{O}}\left(\frac{\sigma(\log n)^{d+1}}{\sqrt{n}}+\frac{\rho(\log n)^{(d+1) / 2}}{\sqrt{n}}\right),
\end{equation}
and with probability at least $(1-\eta)(1-\delta)$, 
\begin{equation}
    \mathcal{E}_2={\mathcal{O}}\left(\frac{L_f \operatorname{diam}(\mathcal{X})}{\eta t_0^2} \sqrt{\frac{\mathcal{N}(\mathcal{S}, \frac{1}{m}) d \log (\frac{1}{\delta})}{m}}\right),
\end{equation}
where $L_f$ denotes the Lipschitz constant of the objective function $f$, $\operatorname{diam}(\mathcal{X})$ denotes the diameter of the domain $\mathcal{X}$, defined as $\operatorname{diam}(\mathcal{X})=\sup_{\mathbf{x},\mathbf{x}'\in\mathcal{X}}|\mathbf{x}-\mathbf{x}'|_2$, and $\mathcal{S}$ denotes the function class induced by the neural network architecture of $\mathbf{s}_\theta$ in Assumption~\ref{asm:realizability}, and $\mathcal{N}(\mathcal{S}, 1 / m)$ denotes the $\epsilon$-covering number of $\mathcal{S}$ with $\epsilon=1 / m$.
\end{theorem}
All proof details are provided in Appendix~\ref{app:proof-of-lemmas}. 
\begin{remark}
The bound in Theorem~\ref{thm:sub-optimality} admits a natural two-stage interpretation. The term $\mathcal{E}_1$ captures the gap between $P(\mathbf{x}\mid y)$ and the GP-induced distribution $P(\mathbf{x}\mid \widehat{y})$. The term $\mathcal{E}_2$ reflects the approximation error incurred between $P(\mathbf{x}\mid \widehat{y})$ and the CDM learned distribution $\widehat{P}(\mathbf{x}\mid \widehat{y})$. 
\end{remark}

\begin{remark}\label{rmk:E_1}
The bound $\mathcal{E}_1$ separates the effects of observation noise and balance-aware pseudo-labeling. 
The former is governed by the noise level $\sigma$, while the latter is controlled by the exploration parameter $\rho$. 
Both terms decay at the same rate $\widetilde{\mathcal{O}}(n^{-1/2})$, indicating that balance-aware pseudo-labeling does not introduce asymptotic bias.
\end{remark}

\begin{remark}\label{rmk:E_2}
The bound $\mathcal{E}_2$ reflects intuitive problem-dependent factors, including the smoothness of $f$ and the scale of the domain $\mathcal{X}$. 
Its inverse dependence on $t_0^2$ highlights the instability caused by excessively small diffusion times, consistent with the discussion in Section~\ref{subsec:CDM}. 
Moreover, the bound decreases with the training size $m$, confirming improved estimation with more training data.
\end{remark}

\begin{figure*}[t]
\centering

\begin{minipage}[t]{0.23\textwidth}
  \centering
  \includegraphics[width=\linewidth]{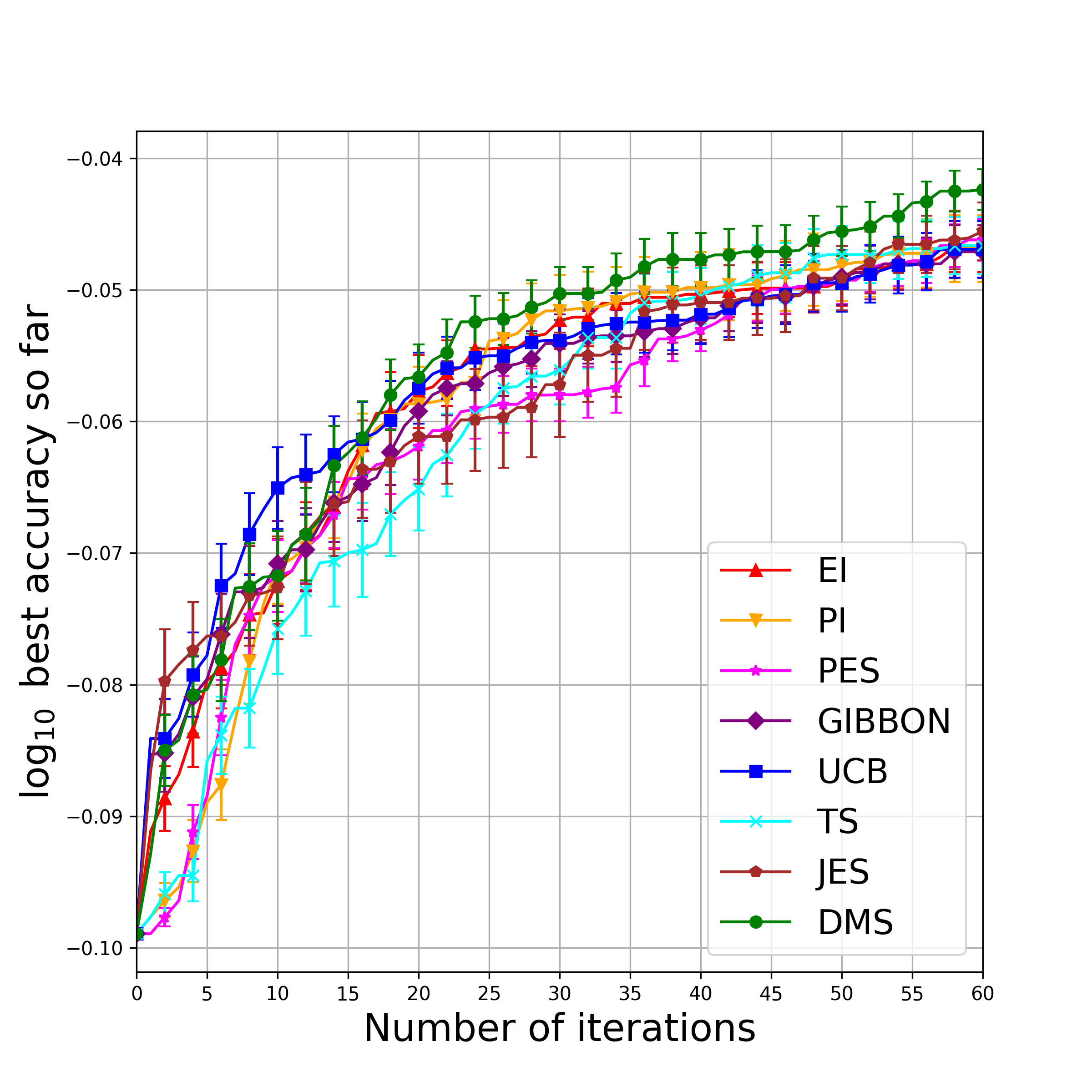}
  \small (a) Wine Recognition
\end{minipage}\hfill%
\begin{minipage}[t]{0.23\textwidth}
  \centering
  \includegraphics[width=\linewidth]{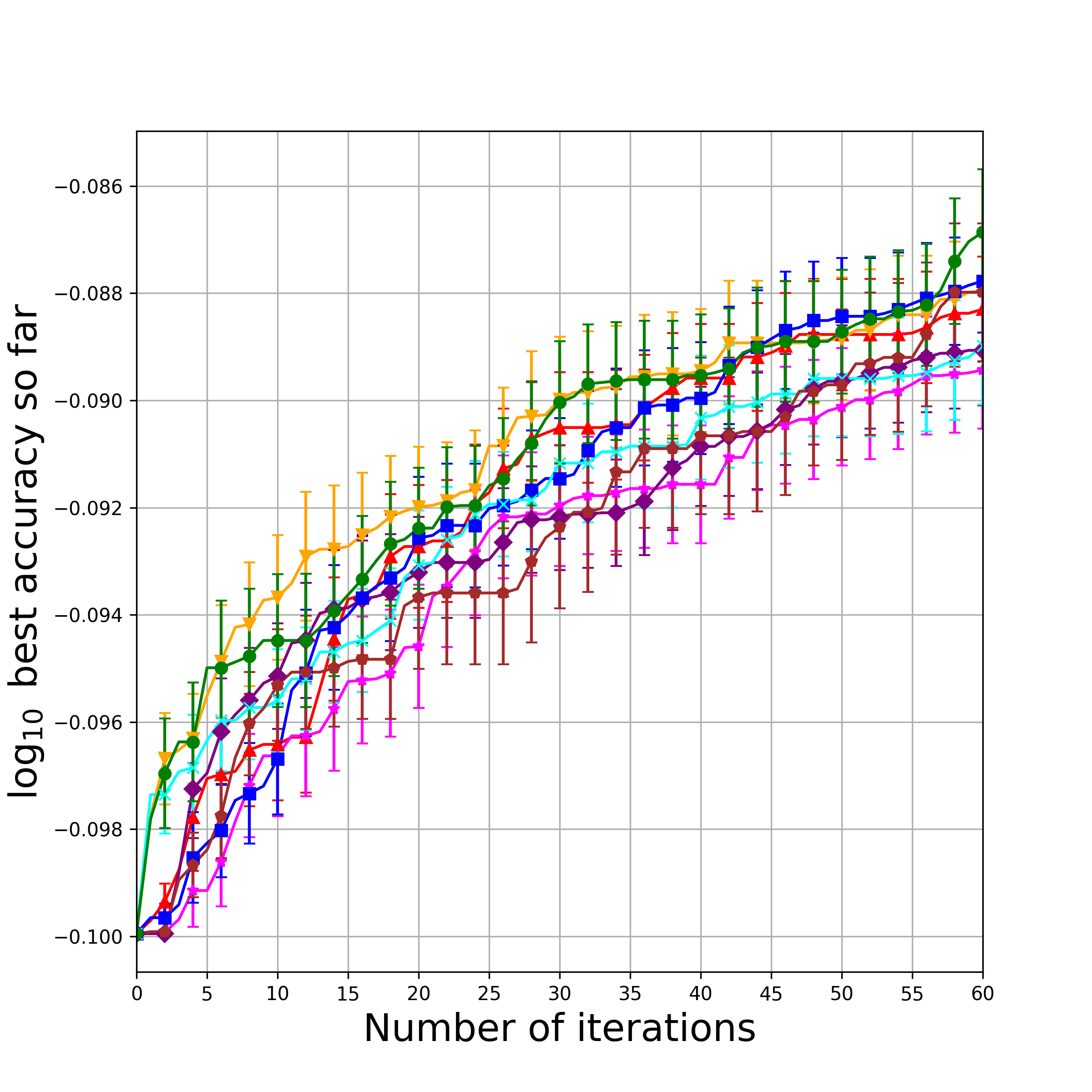}
  \small (b) Vehicle Silhouette
\end{minipage}\hfill%
\begin{minipage}[t]{0.23\textwidth}
  \centering
  \includegraphics[width=\linewidth]{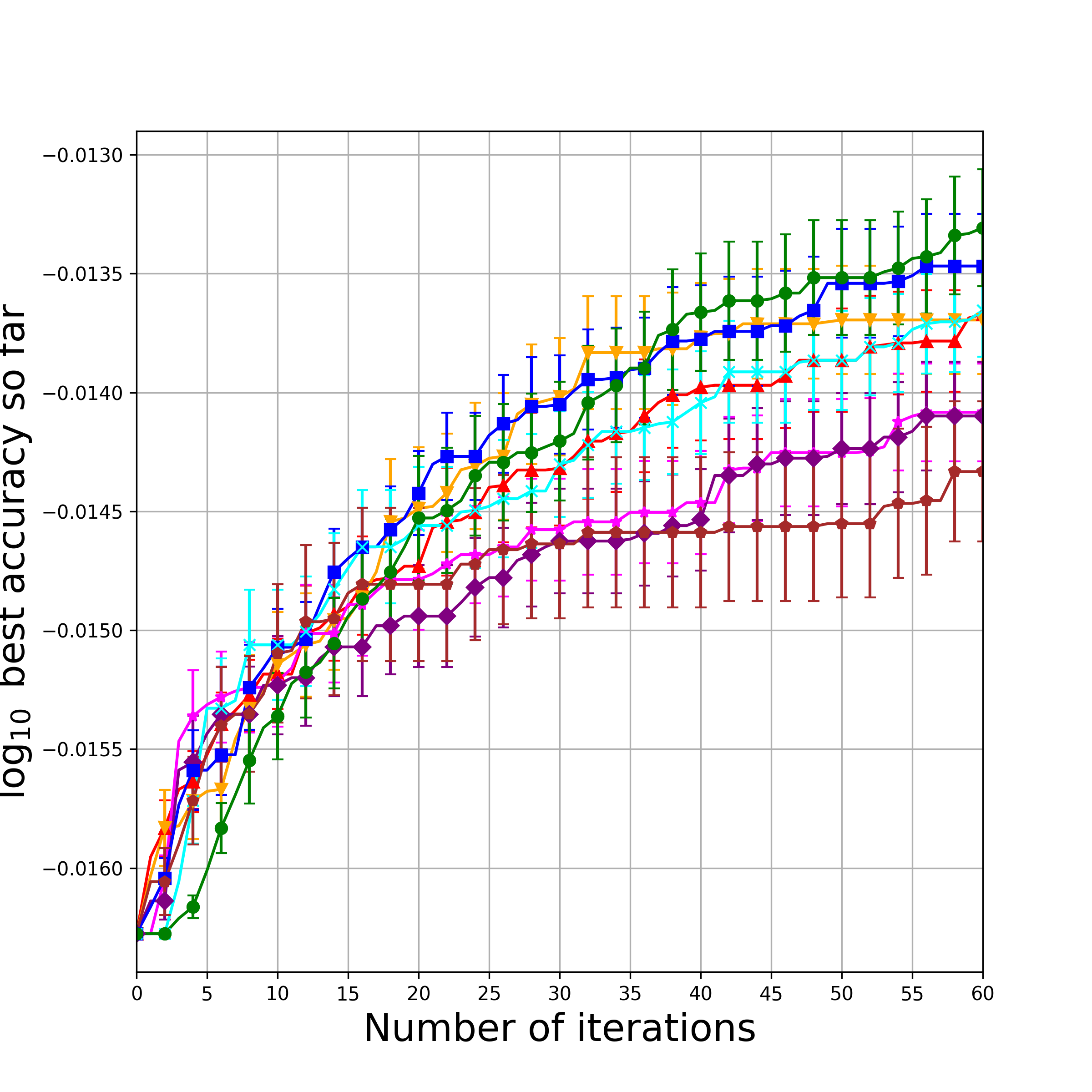}
  \small (c) Image Segmentation
\end{minipage}\hfill%
\begin{minipage}[t]{0.23\textwidth}
  \centering
  \includegraphics[width=\linewidth]{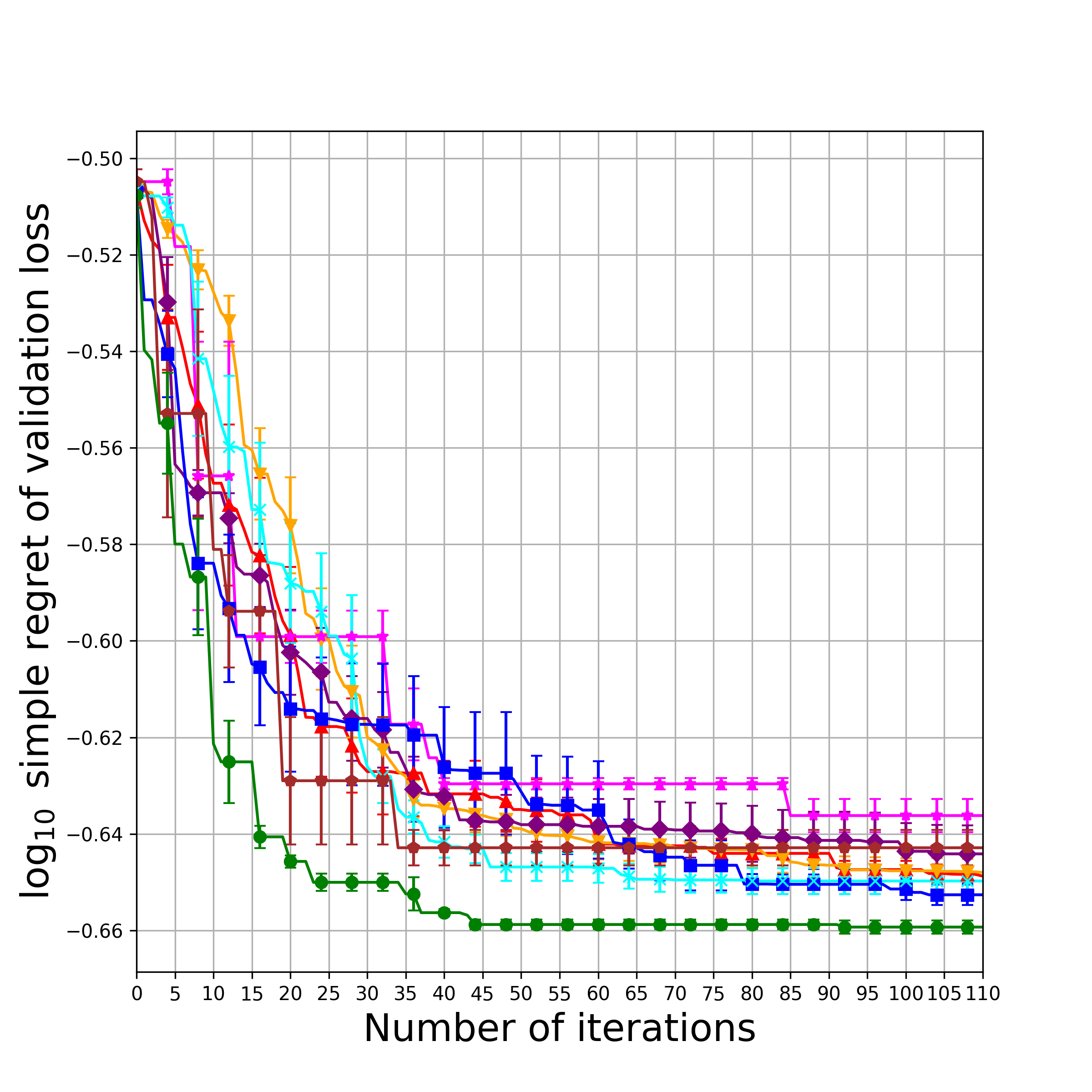}
  \small (d) Protein Structure
\end{minipage}

\caption{Comparison between DMS and other baselines on different real-world tasks. All experiments are conducted with 30 macro-replications, with mean and one standard error reported.}
\label{fig:real-world}
\end{figure*}

\section{Experiments}\label{sec:experiments}
\subsection{Experimental Setup}
We evaluate our algorithms on both synthetic functions and real-world hyperparameter optimization.

For synthetic functions, we consider a suite of widely used benchmark functions across different input dimensions. The standard benchmark set includes Styblinski--Tang ($d=2$), Griewank ($d=3$), Shekel ($d=4$), Rastrigin ($d=5$), Rosenbrock ($d=7$), Michalewicz ($d=10$), Ackley ($d=8$), and Levy ($d=10$). To evaluate scalability, we further test input dimensions with $d=20$ and $d=50$ for both Levy and Ackley functions. These functions cover a broad range of optimization challenges, including differing smoothness, strong non-convexity, pronounced multi-modality, numerous local optima, and increasing input dimensionality. Definitions of all synthetic functions adopted in the experiments are provided in Appendix~\ref{app:synthetic_benchmark}. We evaluate performance on synthetic functions using simple regret: given a black-box function $f: \mathcal{X} \rightarrow \mathbb{R}$ and a sequence of evaluation points $\left\{\mathbf{x}_1, \ldots, \mathbf{x}_n\right\}$ queried by the optimizer, the simple regret after $n$ evaluations is defined as: $
r_n^{\text {simple }}=f^{\star}-\max _{1 \leq i \leq n} f\left(\mathbf{x}_i\right)$, where $f^{\star}=\max _{\mathbf{x} \in \mathcal{X}} f(\mathbf{x})$ denotes the ground truth global maximum of the objective function. 

For real-world tasks, we consider hyperparameter optimization benchmarks from both classification and regression settings. Following the \texttt{Bayesmark} protocol, we tune an MLP with four hyperparameters on three OpenML classification datasets, and report the best observed validation accuracy. We further tune a neural network with nine hyperparameters on a regression task from the \texttt{HPOLib} benchmark, where the observation is the validation loss. For this task, we set optimal loss $\ell^\star = 0$ and report simple regret in validation loss. Details are provided in Appendix~\ref{app:real-world}.

We compare DMS with BO methods as introduced in Section~\ref{sec:background}, including PI, EI, UCB, TS, PES, General-purpose Information-Based Bayesian optimizatioN (GIBBON, a.k.a lower-bound MES) \citep{moss2021gibbon}, and JES. Configurations of these baselines are provided in Appendix~\ref{app:config_baselines}. Note that we omit PES for $d=20,50$ and JES for $d=50$ due to their prohibitive computational costs. 

For each synthetic test function $f$, we initialize the observed dataset $\mathcal{D}_n$ with $n = 10d$ Sobol-generated input points and noisy observations $y = f(\mathbf{x}) + \epsilon$, where $\epsilon \sim \mathcal{N}(0, 0.1^2)$. We set the evaluation budget to $B = 10d + 20$ when $d\le 10$, and $B=10d$ when $d=20,50$. For real-world hyperparameter optimization tasks, each experiment is initialized with an observed dataset $\mathcal{D}_n$ of size $n = 10$, where the initial configurations are sampled from each benchmark using a Sobol sequence. The total evaluation budget is also set to $B = 10d+20$.

We use \texttt{BoTorch} framework \citep{balandat2020botorch} to construct the GP model for both our algorithms and BO baselines. We adopt constant mean function and squared exponential covariance function as the GP prior, with details in Appendix~\ref{app:GP}. 

Algorithm~\ref{alg:DMS} introduces several hyperparameters, including the pseudo-dataset size $m$, the number of short-run L-BFGS steps $K$, the balance-aware coefficient $\rho$, and the number of $\mathbf{x}^{\star}$ candidates $S$. Details of these hyperparameters for all tasks are deferred to Appendix~\ref{app:CDM}.

\subsection{Main Results}\label{subsec:main-results}
Figure~\ref{fig:synthetic_simple_regret} and Figure~\ref{fig:real-world} summarize the optimization performance of DMS on both synthetic benchmarks and real-world tasks. On synthetic benchmarks, DMS consistently achieves lower simple regret than the baselines and exhibits a sustained decreasing trend on most tasks, while many baselines plateau after limited progress. On real-world tasks, DMS also maintains competitive improvement and outperforms the compared BO baselines, further supporting the effectiveness of learning a high-quality conditional distribution for candidate generation.

\subsection{Ablation Studies}\label{subsec:ablation}
We conduct ablation studies on Styblinski-Tang and Levy to examine the effectiveness of our proposed training strategies in DMS, including pseudo-labeling, steps of short-run L-BFGS, the pseudo-dataset size $\widehat{\mathcal{D}}_m$, and the number of generated candidates $\mathbf{x}^{\star}$. Details are deferred to Appendix~\ref{app:ablation}.

The ablation results show that pseudo-labeling is essential for learning an informative conditional distribution, while training the CDM solely on the observed dataset $\mathcal{D}_n$ leads to a clear performance degradation. We also observe that setting $\rho=0$ consistently performs the worst, while overly large values of $\rho$ make the optimization excessively exploratory. In addition, using short-run L-BFGS to refine Sobol-sampled inputs improves performance over using Sobol sequence alone. Finally, we observe that the number of generated $\mathbf{x}^{\star}$ candidates has little effect on performance, while moderately increasing the pseudo-dataset size $m$ can further improve the  performance.

\subsection{Computational Complexity and Wall-Clock Time}

We provide details of the computational overhead of DMS in Appendix~\ref{app:time-complexity}.
Specifically, we analyze the bound of the computational complexity for PES and DMS to generate samples of $\mathbf{x}^\star$. We also empirically report two types of wall-clock time comparisons: the time required by DMS and PES to generate the same number of $\mathbf{x}^{\star}$ samples, and the average
per-iteration time of all compared BO methods across different input dimensions. These
results show that DMS is more efficient than PES in generating $\mathbf{x}^{\star}$
candidates. Additionally, compared with other simple acquisition functions, DMS introduces
additional overhead, but its per-iteration runtime still remains in the regime of seconds and is acceptable in practice.

\section{Conclusion}\label{sec:conclusion}
In this work, we propose DMS, a CDM-based BO acquisition strategy that efficiently learns the distribution of $\mathbf{x}^\star$ and selects its mode as the next evaluation point. We introduce balance-aware pseudo-labeling and short-run L-BFGS to construct informative pseudo-training data for CDM training. We provide a distribution-level sub-optimality guarantee and demonstrate strong empirical performance on extensive synthetic and real-world BO tasks.

\bibliographystyle{plainnat}
\bibliography{references}

\appendix

\newpage
\appendix

\section{Additional Backgrounds}

\subsection{Transition Kernel in Forward SDEs}\label{app:transition-kernel}
For the general SDE in Eq.~\ref{eq:forward_process}, a key property is that when the drift coefficient $\mathbf{f}\left(\mathbf{x}_t^y, t\right)$ is affine in $\mathbf{x}_t^y$, the resulting transition kernel $p_t\left(\mathbf{x}_t^y \mid \mathbf{x}_0^y, y\right)$ admits a Gaussian form \citep{evans2012introduction, sarkka2019applied}. Moreover, since the evolution of $\mathbf{x}_t^y$ depends on the initial condition $\mathbf{x}_0^y$ alone, the transition kernel can be simplified to $p_t\left(\mathbf{x}_t^y \mid \mathbf{x}_0^y\right)$.

In diffusion models, the Variance-Preserving (VP) SDE \citep{song2020score, krishnamoorthy2023diffusion} is a representative instance of the affine SDE family, defined as 
\begin{equation}\label{eq:VP-SDE}
    \mathrm{d}\mathbf{x}_t^y=-\frac{1}{2} \beta(t) \mathbf{x}_t^y \mathrm{~d} t+\sqrt{\beta(t)} \mathbf{I}\mathrm{d} \mathbf{w}_t, \quad t \in[0,1], 
\end{equation}
where the drift and diffusion coefficients are governed by a non-negative time-dependent noise schedule function $\beta(t)=\beta_{\text{min}}+t(\beta_{\text{max}}-\beta_{\text{min}})$. 

The corresponding transition kernel admits a closed-form Gaussian solution, $p_t(\mathbf{x}_t^y\mid \mathbf{x}_0^y)=\mathcal{N}(\mathbf{x}_t^y;\boldsymbol{\mu}_t,\boldsymbol{\boldsymbol{\Sigma}}_t),$ with
\begin{equation}\label{eq:transition-kernel-VP-SDE}
\boldsymbol{\mu}_t=\mathbf{x}_0^y \exp\left(-\tfrac{1}{2}\int_0^t \beta(s)\mathrm{d}s\right),\qquad
\boldsymbol{\boldsymbol{\Sigma}}_t=\left(1-\exp\left(-\int_0^t \beta(s)\mathrm{d}s\right)\right)\mathbf{I}.
\end{equation}
For notational convenience in the theoretical analysis in Appendix~\ref{app:sub-optimality} and Appendix~\ref{app:proof-of-lemmas}, we equivalently denote the distribution of the transition kernel as $\mathcal{N}\left(\mathbf{x}_t^y ; \mathbf{x}_0^y \alpha(t), h(t)\right)$, where $\alpha(t)=\exp \left(-\frac{1}{2} \int_0^t \beta(s) \mathrm{d} s\right)$ and $h(t)=1-\exp \left(-\int_0^t \beta(s) \mathrm{d} s\right)$.

\subsection{Classifier-Free Guidance}\label{app:classifier-free-guidance}
Many empirical results show that directly training the conditional score predictor as introduced in Section~\ref{subsec:CDM} will generate low-quality samples \citep{dhariwal2021diffusion, ho2022classifier, krishnamoorthy2023diffusion}. Classifier-free training strategy proposed by \citep{ho2022classifier} is a mitigation to such issues. Specifically, during training, the condition $y$ is randomly dropped with probability $p_{\text {drop }}$, resulting in a mixed loss function that jointly learns conditional and unconditional score predictors. The resulting loss function can be written as

\begin{equation}
    \underset{t}{\mathbb{E}}\left[\underset{\mathbf{x}_0, y}{\mathbb{E}}\left[\underset{\mathbf{x}_t \mid \mathbf{x}_0,y}{\mathbb{E}}\left[\left\|\mathbf{s}_{\boldsymbol\theta}\left(\mathbf{x}_t^y, t, \tilde{y}\right)-\nabla_{\mathbf{x}} \log p_t\left(\mathbf{x}_t^y \mid \mathbf{x}_0, \tilde{y}\right)\right\|_2^2\right]\right]\right],
\end{equation}

where the effective condition $\tilde{y}$ is defined as

\begin{equation}
    \tilde{y}= \begin{cases}y, & \text { with probability } 1-p_{\text {drop }}, \\ \emptyset, & \text { with probability } p_{\text {drop }},\end{cases}
\end{equation}

and $\emptyset$ denotes the absence of conditioning. This formulation allows a single neural network to simultaneously learn the conditional score $\nabla_{\mathbf{x}} \log p_t\left(\mathbf{x}_t \mid y\right)$ and the unconditional score $\nabla_{\mathbf{x}} \log p_t\left(\mathbf{x}_t\right)$, improving robustness and stability in conditional score learning. 

After the training, the classifier-free guidance score is formulated as

\begin{equation}
    \mathrm{s}_{\boldsymbol\theta}^{\mathrm{cfg}}\left(\mathrm{x}_t^y, t, y\right)=(1+w) \mathrm{s}_{\boldsymbol\theta}\left(\mathrm{x}_t^y, t, y\right)-w \mathrm{~s}_{\boldsymbol\theta}\left(\mathrm{x}_t^y, t, \emptyset\right),
\end{equation}

where $w \geq 0$ is a guidance scale controlling the strength of conditioning. This guided score $\mathrm{s}_{\boldsymbol\theta}^{\mathrm{cfg}}\left(\mathrm{x}_t, t, y\right)$ is then used in place of the unknown conditional score function $\nabla_\mathbf{x}\log p_t(\mathbf{x}_t\mid y)$ in the backward SDE to generate samples. 

\subsection{Mean-Shift Clustering}\label{app:mean-shift}

Specifically, in our algorithm, mean-shift is initialized from each candidate point of $\mathbf{x}^\star$, i.e., $\mathbf{z}_s^{(0)}=\mathbf{x}_s^{\star}$ for $s= 1, \ldots, S$. Starting from an initial point $\mathbf{z}^{(0)}$, mean shift iteratively updates

\begin{equation}
    \mathbf{z}^{(t+1)}=\frac{\sum_{s=1}^S K_h\left(\mathbf{z}^{(t)}-\mathbf{x}_s^{\star}\right) \mathbf{x}_s^{\star}}{\sum_{s=1}^S K_h\left(\mathbf{z}^{(t)}-\mathbf{x}_s^{\star}\right)},
\end{equation}

where $K_h(\cdot)$ is a kernel function with bandwidth $h$. When initialized from different candidates $\mathbf{x}_s^{\star}$, this procedure converges to a set of local modes of the empirical density induced by $\left\{\mathbf{x}_s^{\star}\right\}_{s=1}^S$, yielding multiple cluster centers corresponding to different modes. Among all resulting cluster centers, the one with the largest number of converged samples is regarded as the dominant mode, which is exactly the next evaluation point selected by Algorithm~\ref{alg:DMS}.

\newpage
\section{Experiment Details}\label{app:experiment_details}

\subsection{Synthetic Functions}\label{app:synthetic_benchmark}
The expression of the synthetic functions we test are defined in Table~\ref{tb:synthetic_functions}. Recall that we consider maximization problems throughout this work. Hence, all of the synthetic benchmark functions are given in their negative forms.

\begin{table}[h]
\centering
\caption{Synthetic benchmark functions used in the experiments.}
\label{tb:synthetic_functions}
\small
\begin{tabular}{l c c}
\toprule
Function & Expression (neg.) & Bounds \\
\midrule
Styblinski--Tang &
$-\frac{1}{2}\sum_{i=1}^d \left(x_i^4 - 16x_i^2 + 5x_i\right)$ &
$[-5, 5]^d$ \\

Griewank &
$-\left(1 + \frac{1}{4000}\sum_{i=1}^d x_i^2 - \prod_{i=1}^d \cos\left(\frac{x_i}{\sqrt{i}}\right)\right)$ &
$[-600, 600]^d$ \\

Shekel &
$\sum_{i=1}^{10}\left(\sum_{j=1}^4\left(x_j-C_{j i}\right)^2+\beta_i\right)^{-1}$ &
$[0, 10]^d$\\

Rastrigin &
$-\sum_{i=1}^d\left[x_i^2-10 \cos \left(2 \pi x_i\right)+10\right]$ &
$[-5.12,5.12]^d$\\

Rosenbrock &
$-\sum_{i=1}^{d-1}\left[100\left(x_{i+1}-x_i^2\right)^2+\left(x_i-1\right)^2\right]$ &
$[-2.048, 2.048]^d$\\

Ackley &
$20 \exp \left(-0.2 \sqrt{\frac{1}{d} \sum_{i=1}^d x_i^2}\right)-\exp \left(\frac{1}{d} \sum_{i=1}^d \cos \left(2 \pi x_i\right)\right)+20+e$ &
$[-32.768,32.768]^d$\\

Levy &
$-\left[
\sin^2(\pi w_1)
+ \sum_{i=1}^{d-1} (w_i-1)^2\bigl(1+10\sin^2(\pi w_i+1)\bigr)
+ (w_d-1)^2
\right]$ &
$[-10, 10]^d$ \\

Michalewicz &
$\sum_{i=1}^d \sin \left(x_i\right)\left[\sin \left(\frac{i x_i^2}{\pi}\right)\right]^{2 m}$ &
$[0,\pi]^d$\\

\bottomrule
\end{tabular}
\end{table}

Notably, some benchmark functions involve additional parameters beyond the input domain and dimensionality. We specify the corresponding parameters below. 

In Shekel function, we adopt 
$$
m=10,\quad 
C = \begin{bmatrix}
4 & 4 & 4 & 4 \\
1 & 1 & 1 & 1 \\
8 & 8 & 8 & 8 \\
6 & 6 & 6 & 6 \\
3 & 7 & 3 & 7 \\
2 & 9 & 2 & 9 \\
5 & 3 & 5 & 3 \\
8 & 1 & 8 & 1 \\
6 & 2 & 6 & 2 \\
7 & 3.6 & 7 & 3.6
\end{bmatrix}, \quad
\beta = \begin{bmatrix}
0.1 & 0.2 & 0.2 & 0.4 & 0.4 & 0.6 & 0.3 & 0.7 & 0.5 & 0.5
\end{bmatrix}^\top.
$$

In Levy function, we adopt $w_i=1+\frac{x_i-1}{4}$ for $i=1,\dots, d$. 

In Michalewicz, we adopt $m=10$.

\subsection{Real-World Tasks}\label{app:real-world}
The datasets used in our experiments are obtained from OpenML, with the following dataset identifiers: Wine (ID: 187), Vehicle (ID: 54), and Image Segmentation (ID: 36). All datasets can be accessed via \url{https://www.openml.org}.

For real-world hyperparameter optimization tasks, we adopt the MLP model with Adam optimizer provided by \texttt{Bayesmark} benchmark suite. We tune four hyperparameters of the MLP model: the $\ell_2$ regularization coefficient alpha, the initial learning rate, the hidden layer size, and the mini-batch size. All remaining hyperparameters are set to their default values as specified in \texttt{Bayesmark}.

The corresponding hyperparameter search spaces are defined following the \texttt{Bayesmark} configuration. Specifically, alpha is searched on a logarithmic scale over the range $[10^{-5}, 10]$, and initial learning rate is searched on a logarithmic scale over $[10^{-5}, 0.1]$. The hidden layer size is searched on a linear scale within $[50, 200]$, while the mini-batch size is searched on a linear scale within $[10, 250]$. Since both parameters are discrete in practice, the next evaluation point determined by the algorithms will be rounded to the nearest valid integers before training the MLP.

In addition to the Bayesmark-based real-world tasks, we further consider the HPOLib FCNet tabular benchmark for higher-dimensional hyperparameter optimization. Specifically, we use the \texttt{protein\_structure} regression task from the HPOLib FCNet benchmark, where the goal is to tune the hyperparameters of a two-hidden-layer fully connected neural network. The benchmark provides precomputed evaluations of neural network configurations, and therefore each function evaluation is obtained by querying the tabular benchmark rather than retraining the neural network from scratch.

The search space contains nine hyperparameters: the activation functions of the first and second hidden layers, the mini-batch size, the dropout rates of the first and second hidden layers, the initial learning rate, the learning-rate schedule, and the numbers of units in the first and second hidden layers. The activation functions are selected from \{\texttt{relu}, \texttt{tanh}\}, the learning-rate schedule is selected from \{\texttt{cosine}, \texttt{const}\}, the mini-batch size is selected from \{8, 16, 32, 64\}, the dropout rates are selected from \{0.0, 0.3, 0.6\}, the initial learning rate is selected from \{0.0005, 0.001, 0.005, 0.01, 0.05, 0.1\}, and the numbers of units in each hidden layer are selected from \{16, 32, 64, 128, 256, 512\}.

For compatibility with our continuous optimization framework, all hyperparameters are represented in a normalized input space $[0,1]^9$. Each coordinate is mapped to the corresponding discrete hyperparameter set before querying the benchmark.

\subsection{Gaussian Process}\label{app:GP}
For the GP surrogate, we adopt the default modeling configuration used in \texttt{BoTorch}. We employ a GP prior with a constant mean function and a Squared Exponential (SE) covariance kernel with automatic relevance determination (ARD), which assigns a separate length-scale $\ell_i$ to each input dimension. All hyperparameters are assigned the same prior distributions as in \citet{hvarfner2024vanilla}, and are learned via MAP estimation. Additionally, the observed dataset is preprocessed by normalizing the inputs to the unit cube and standardizing the outputs to zero mean and unit variance. 

\subsection{Conditional Diffusion Model}\label{app:CDM}


\subsubsection{SDE Configurations}

We adopt VP SDE introduced in Appendix~\ref{app:transition-kernel} as the forward SDE, with $\beta_{\text {min }}=0.1$ and $\beta_{\text {max }}=20$ respectively. When generating $\mathbf{x}^{\star}$ candidates, the backward SDE is simulated by the second-order Heun solver. 

\subsubsection{Classifier-Free Guidance Settings}
We adopt classifier-free guidance strategy to train the CDM and generate $\mathbf{x}^{\star}$ candidates, where we set $p_{\text {drop }}=0.15$ and $w=2.0$, following common practice in the diffusion model literature, and keep them fixed across all tasks. 

\subsubsection{Architecture of Score Predictor}
We employ an MLP as the backbone of the conditional score predictor $\mathbf{s}_{\boldsymbol{\theta}}\left(\mathbf{x}_t, t, y\right)$. Specifically, the two scalar inputs $t$ and $y$ are first mapped through separate positional embedding layers \citep{vaswani2017attention}, producing embeddings of dimensions $d_t$ and $d_y$, respectively, where we set $d_t=d_y=8$. These embeddings are then concatenated with the input $\mathbf{x}_t \in \mathbb{R}^d$, forming a combined feature vector of dimension $d_t+d_y+d$. The resulting vector is processed by an MLP consisting of three linear layers with hidden size $H=256$, interleaved with Mish activations \citep{misra2019mish}. 

\subsubsection{Training Details}
When constructing the pseudo-dataset $\widehat{\mathcal{D}}_m$, we set the balance-aware pseudo-labeling coefficient $\rho$ to $1.0$ for all tasks. We then train the CDM by $\widehat{\mathcal{D}}_m$. Similar to fitting GP, the pseudo-dataset is preprocessed by normalizing the inputs to the unit cube and standardizing the outputs to zero mean and unit variance. At each BO iteration, we train the score predictor using the AdamW optimizer \citep{loshchilov2017decoupled} with weight decay set to $1\times10^{-4}$ for $100$ epochs and a mini-batch size of $256$. At the first BO iteration, all linear layers are initialized using Kaiming normal initialization with zero-initialized biases; from the second BO iteration onward, model parameters are loaded from the previous iteration. 

At each BO iteration, the learning rate is initialized to $1\times 10^{-3}$, then we apply a warm-up phase for the first $20$ epochs, during which the learning rate is held constant. After the warm-up phase, we adopt a cosine annealing schedule \citep{loshchilov2016sgdr} to gradually decay the learning rate to a minimum value of $1\times 10^{-4}$ by the end of epochs. 

Other additional task-specific hyperparameters are listed in Table~\ref{tb:training-hyperparameters}. 

\begin{table}[t]
  \caption{Task-specific hyperparameters for all tested benchmarks.}
  \label{tb:training-hyperparameters}
  \centering
  \begin{small}
    \begin{tabular}{lccc}
      \toprule
      Tasks &
      Size of Pseudo-Dataset $m$ &
      Steps of Short-Run L-BFGS $K$ &
      Number of $\mathbf{x}^{\star}$ Candidates $S$ \\
      \midrule

      \multicolumn{4}{l}{\textbf{Synthetic Functions}} \\
      \midrule
      Styblinski-Tang  & 500 & 5  & 200 \\
      Griewank         & 500 & 5  & 200 \\
      Shekel           & 500 & 5  & 200 \\
      Rastrigin        & 800 & 5  & 300 \\
      Rosenbrock       & 800 & 5  & 300 \\
      Michalewicz      & 800 & 25 & 300 \\
      Ackley-8           & 800 & 25 & 300 \\
      Levy-10             & 800 & 25 & 300 \\
      Ackley-20             & 1200 & 25 & 400 \\
      Levy-20             & 1200 & 25 & 400 \\
      Levy-50             & 1500 & 25 & 400 \\
      Ackley-50             & 1500 & 25 & 400 \\

      \midrule
      \multicolumn{4}{l}{\textbf{Real-world tasks}} \\
      \midrule
      Wine Recognition & 500 & 5 &  200 \\
      Vehicle Silhouette & 500 & 5 &  200 \\
      Image Segmentation & 500 & 15 & 200 \\
      Protein Structure & 800 & 25 & 300 \\

      \bottomrule
    \end{tabular}
  \end{small}
\end{table}

\subsubsection{Others}
Mean-shift clustering was implemented using \texttt{scikit-learn} package, with the bandwidth automatically selected via the package’s quantile-based estimator. A default flat kernel was used, such that all points within the bandwidth contribute equally to the mean shift. 

\subsection{Configurations of BO Baselines}\label{app:config_baselines}
In this section, we provide the hyperparameters we set for BO baselines. For a fair comparison, all BO baselines use the same GP fitting procedure and hyperparameter configuration as described in Section~\ref{app:GP}. For baselines that do not require additional hyperparameters, such as EI and PI, we use the default settings provided by \texttt{BoTorch}.

\paragraph{UCB} We set the exploration coefficient to 1.0 across all benchmarks., which is a common practice. 

\paragraph{PES} We draw GP posterior sample paths using \texttt{BoTorch}’s Matheron-path sampler, based on Matheron’s update rule \citep{wilson2020efficiently}. We set the number of optimal candidates required to $100$ for all benchmarks according to practical usage. 

\paragraph{TS} Similar to PES, we implement TS using posterior function paths drawn via \texttt{BoTorch}’s Matheron-path sampler. 

\paragraph{GIBBON} We set the number of input candidates required to $10000$ for all benchmarks, according to the original paper \citep{moss2021gibbon}. 

\paragraph{JES} We set the number of optimal candidates to $100$ for all benchmarks with lower bound estimation method, according to the original paper \citep{hvarfner2022joint}.

Note that we do not directly compare with prior diffusion-based black-box optimization methods, since their problem settings differ substantially from the sequential BO setting considered in this work. For example, DDOM~\citep{krishnamoorthy2023diffusion} and the method of \citet{li2024diffusion} mainly target offline black-box optimization, where no sequential function evaluations are permitted. Diff-BBO~\citep{wu2024diffusion} focuses on real-world tasks whose valid input designs are assumed to concentrate on a low-dimensional data manifold, whereas our setting assumes that designs can be valid over the entire compact input space. We also note that DiBO proposed by \citet{yun2025posterior} targets high-dimensional black-box optimization problems, typically with $d\ge 100$, and relies on large-batch evaluations per iteration, rather than the classical sequential evaluation setup considered in this work. Therefore, these methods are not directly comparable to DMS under our experimental setting.

\subsection{Hardware Information}
All experiments were conducted on a workstation equipped with an NVIDIA GeForce RTX 4090 GPU with 24GB memory and an Intel Core i9-13900K CPU. Unless otherwise stated, all reported results were obtained on the same hardware platform. The diffusion model training and sampling procedures were accelerated using the GPU, while the Gaussian process fitting and acquisition-related computations were executed on the CPU or GPU depending on the corresponding implementation.

\newpage
\subsection{Ablation Studies}\label{app:ablation}
As stated in Section~\ref{subsec:ablation}, we conduct ablation studies on Styblinski-Tang $(d=2)$ and Levy $(d=10)$. 

\subsubsection{Ablation on Pseudo-Labeling}\label{app:ablation-pseudo-labeling}
Recall that we analyze the effect of pseudo-labeling by comparing the optimization performance of training the CDM on the pseudo-dataset $\widehat{\mathcal{D}}_m$ versus training it solely on the observed dataset $\mathcal{D}_n$. In the latter case, no pseudo-labeled data pairs are incorporated, and during sampling we condition the CDM on the maximum observed value $y_{\max}=\max_{(\mathbf{x}_i,y_i)\in\mathcal{D}_n} y_i$ to generate candidates of $\mathbf{x}^{\star}$.

As shown in Figure~\ref{fig:ablation-pseudo-labeling}, training the CDM without pseudo-labeling leads to a substantial degradation in optimization performance across both test functions. Specifically, when trained only on $\mathcal{D}_n$, the resulting method exhibits significantly higher simple regret throughout the optimization process, along with a markedly slower rate of improvement. In contrast, incorporating pseudo-labeled samples enables the CDM to achieve faster convergence and substantially lower final simple regret.

This performance gap is particularly pronounced in the early and intermediate stages of optimization, where the observed dataset is small and provides limited coverage of the input space. These results suggest that relying solely on the observed dataset is insufficient for learning an informative and accurate conditional distribution, which in turn restricts the quality of the generated $\mathbf{x}^{\star}$ candidates. By augmenting the training data with pseudo-labeling, the CDM is exposed to a broader range of conditioning values and input locations, resulting in more effective guidance during the sampling stage.

\begin{figure*}[h]
\centering

\begin{minipage}{0.75\textwidth}
\centering
  \begin{minipage}{0.40\textwidth}
    \centering
    \includegraphics[width=\linewidth]{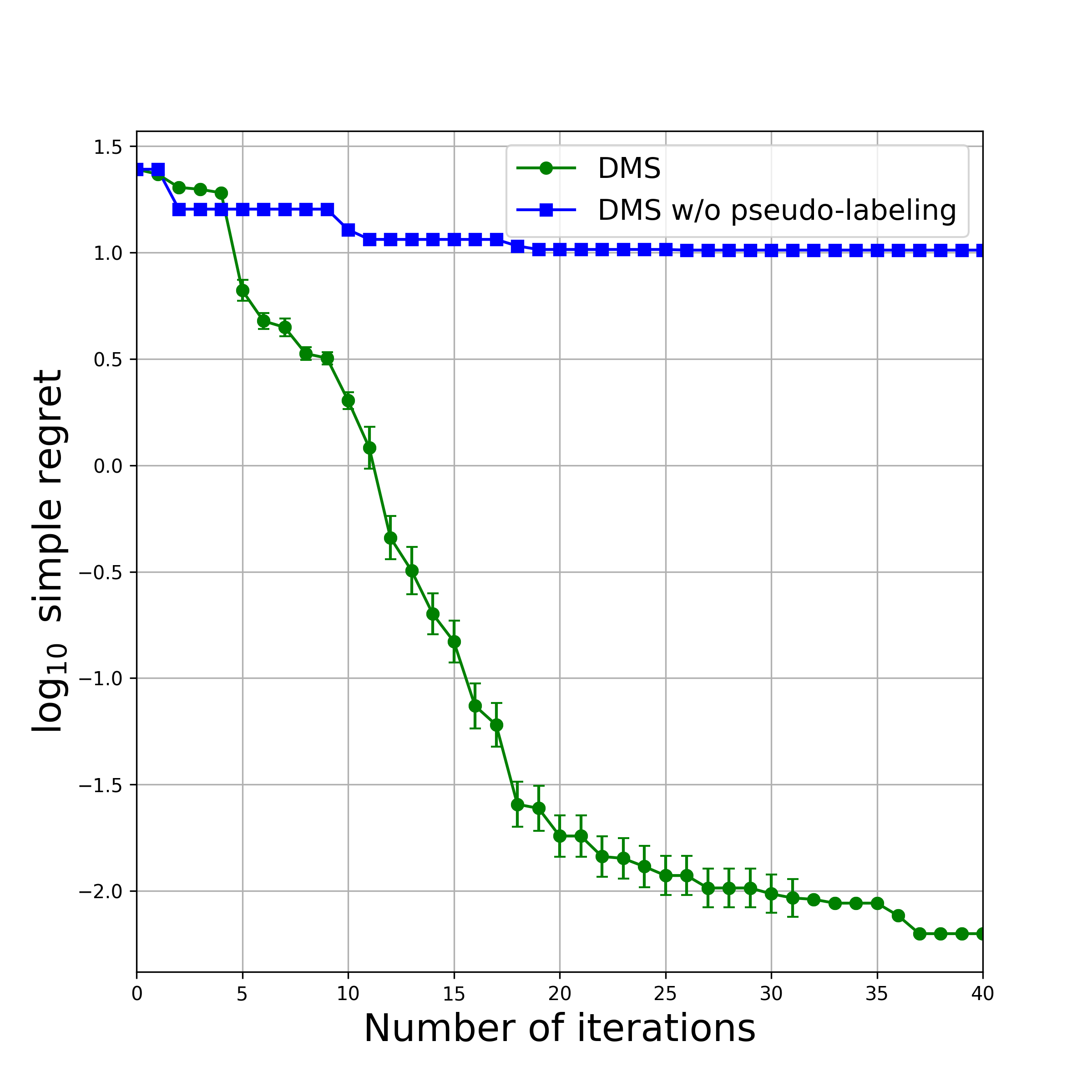}
    \small (a) Styblinski-Tang $(d=2)$
  \end{minipage}\hspace{0.03\textwidth}
  \begin{minipage}{0.40\textwidth}
    \centering
    \includegraphics[width=\linewidth]{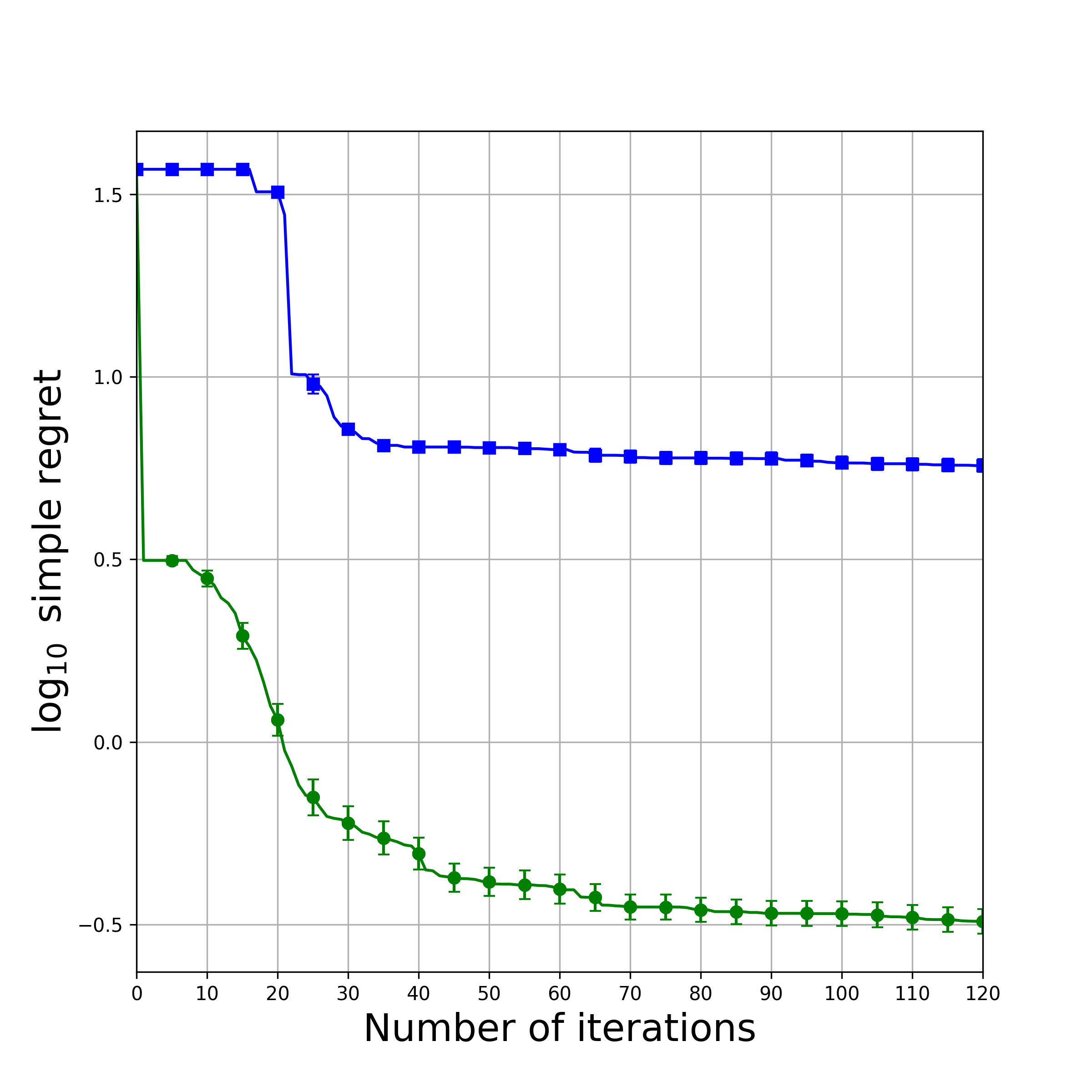}
    \small (b) Levy $(d=10)$
  \end{minipage}\hspace{0.03\textwidth}
\end{minipage}

\caption{Ablation study on training the CDM using pseudo-dataset $\widehat{\mathcal{D}}_m$ versus using only the observed dataset $\mathcal{D}_n$. All experiments are conducted with 30 macro-replications, with mean and one standard error reported.}
\label{fig:ablation-pseudo-labeling}
\end{figure*}

\subsubsection{Ablation On Coefficient $\rho$}\label{app:ablation-beta}
Recall that we study the effect of the balance-aware pseudo-labeling coefficient by varying the scaling factor $\rho \in \{0.0, 1.0, 2.0, 3.0\}$ when constructing the pseudo-dataset $\widehat{\mathcal{D}}_m$. This coefficient controls the impact of the uncertainty term used during pseudo-labeling and thus influences the resulting pseudo-labels. 

As shown in Figure~\ref{fig:ablation-beta}, we observe that setting $\rho=0$ consistently results in the worst optimization performance across both test functions. In this case, the simple regret decreases slowly and plateaus at a relatively high level, indicating limited improvement throughout the optimization process. In contrast, moderate values of $\rho$ lead to substantially faster convergence and lower final regret.

We further observe that excessively large values of $\rho$ also degrade performance. Although larger $\rho$ encourages exploration by emphasizing high-uncertainty regions during pseudo-labeling, this behavior can result in pseudo-labels that are overly explorative, which may cause the evaluation to focus on regions with high posterior uncertainty but low posterior mean, thereby slowing down convergence and limiting final performance.

Overall, these results suggest that extreme choices of $\rho$, either too small or too large, are unfavorable in practice. Empirically, intermediate values of $\rho$ provide more effective guidance for training the conditional diffusion model, leading to improved optimization performance.

\begin{figure*}[h]
\centering

\begin{minipage}{0.75\textwidth}
\centering
  \begin{minipage}{0.40\textwidth}
    \centering
    \includegraphics[width=\linewidth]{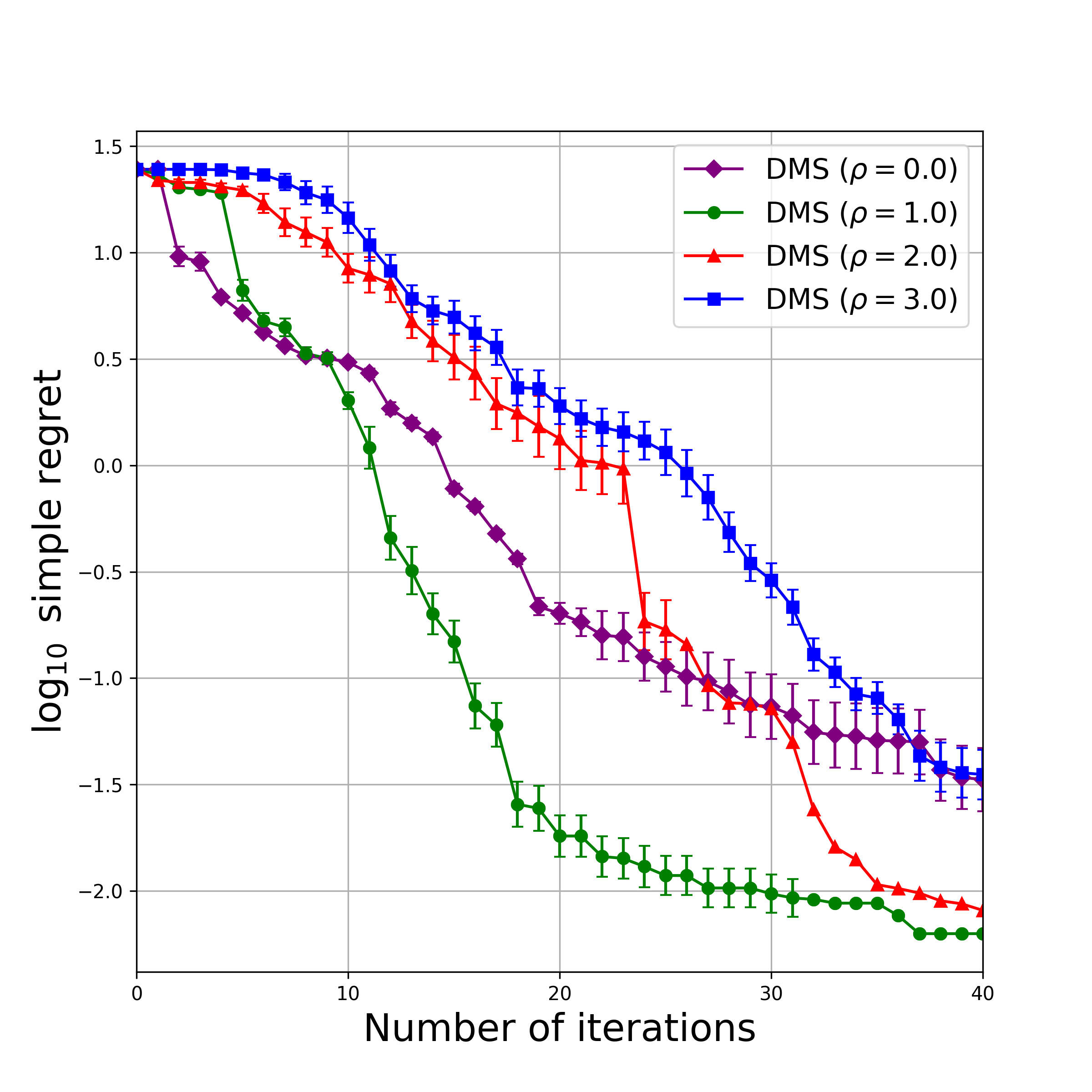}
    \small (a) Styblinski-Tang $(d=2)$
  \end{minipage}\hspace{0.03\textwidth}
  \begin{minipage}{0.40\textwidth}
    \centering
    \includegraphics[width=\linewidth]{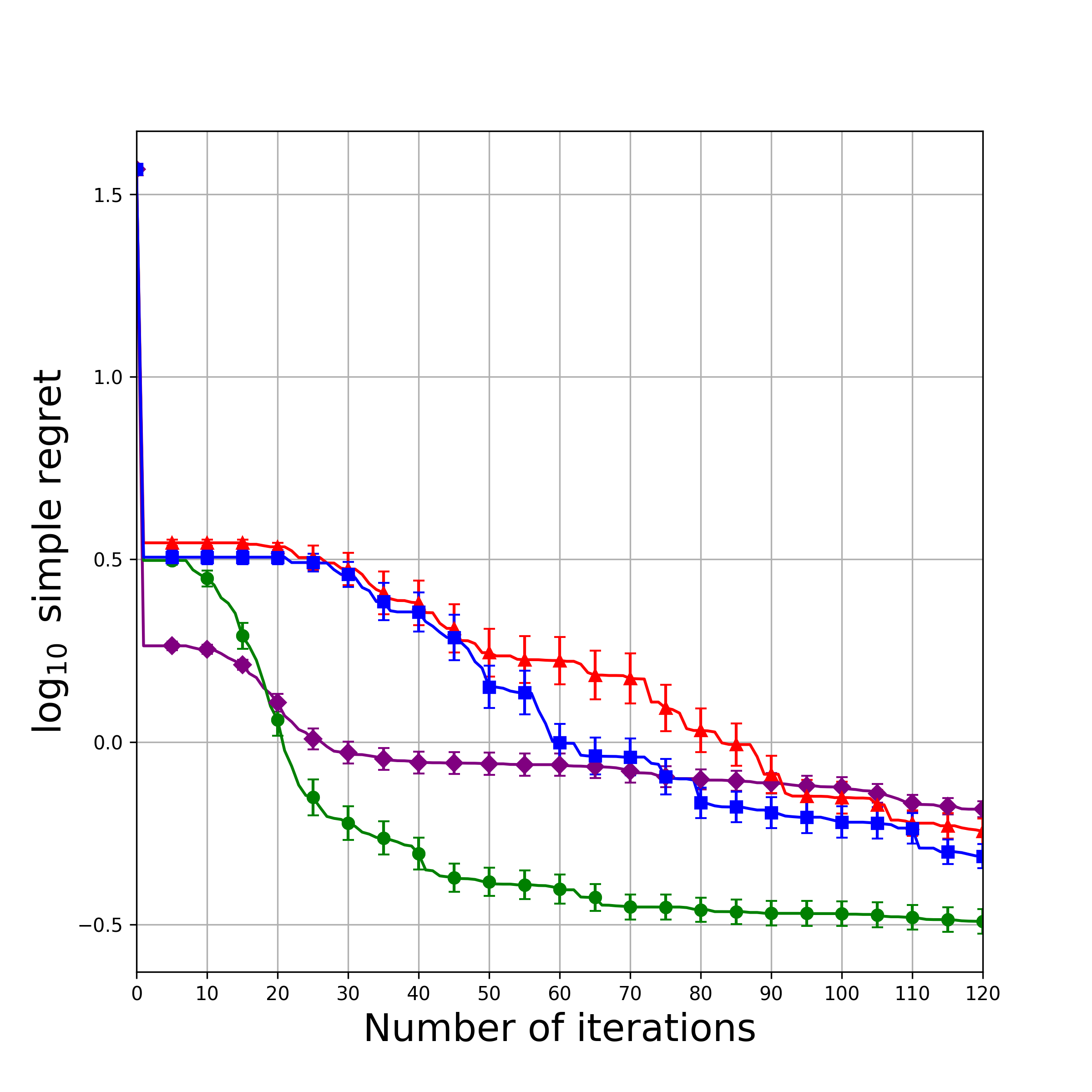}
    \small (b) Levy $(d=10)$
  \end{minipage}\hspace{0.03\textwidth}
\end{minipage}

\caption{Ablation study on the choice of the balance-aware pseudo-labeling coefficient $\rho$. All experiments are conducted with 30 macro-replications, with mean and one standard error reported.}
\label{fig:ablation-beta}
\end{figure*}

\subsubsection{Ablation on Short-Run L-BFGS}\label{app:ablation:short-run-LBFGS}
Recall that we analyze the effect of short-run L-BFGS by varying the number of optimization steps used to refine the initial Sobol-sampled inputs. Our results in Figure~\ref{fig:ablation-L-BFGS} show that employing short-run L-BFGS consistently outperforms the case of $K=0$, where inputs are constructed solely from the initial Sobol sequence without local refinement. Furthermore, the performance improvements become increasingly pronounced in Levy $(d=10)$ compared to relatively lower-dimensional Styblinski-Tang $(d=2)$.

A plausible explanation lies in the effect of dimensionality on the quality of Sobol-sampled inputs. In low-dimensional settings, the initial Sobol sequence has a higher likelihood of being close to regions associated with high pseudo-label values, thereby limiting the marginal benefit of additional local optimization. In contrast, as the dimensionality increases, the curse of dimensionality makes it unlikely for Sobol samples to directly fall into input regions associated with high pseudo-label values. In such cases, short-run L-BFGS plays a crucial role in refining inputs toward regions with higher pseudo-label values, leading to more substantial performance gains. 

Notably, we do not further increase $K$, as we observe that excessively large values of $K$ can cause the L-BFGS procedure to converge, resulting in all inputs collapsing to a single point.

\begin{figure*}[h]
\centering

\begin{minipage}{0.75\textwidth}
\centering
  \begin{minipage}{0.40\textwidth}
    \centering
    \includegraphics[width=\linewidth]{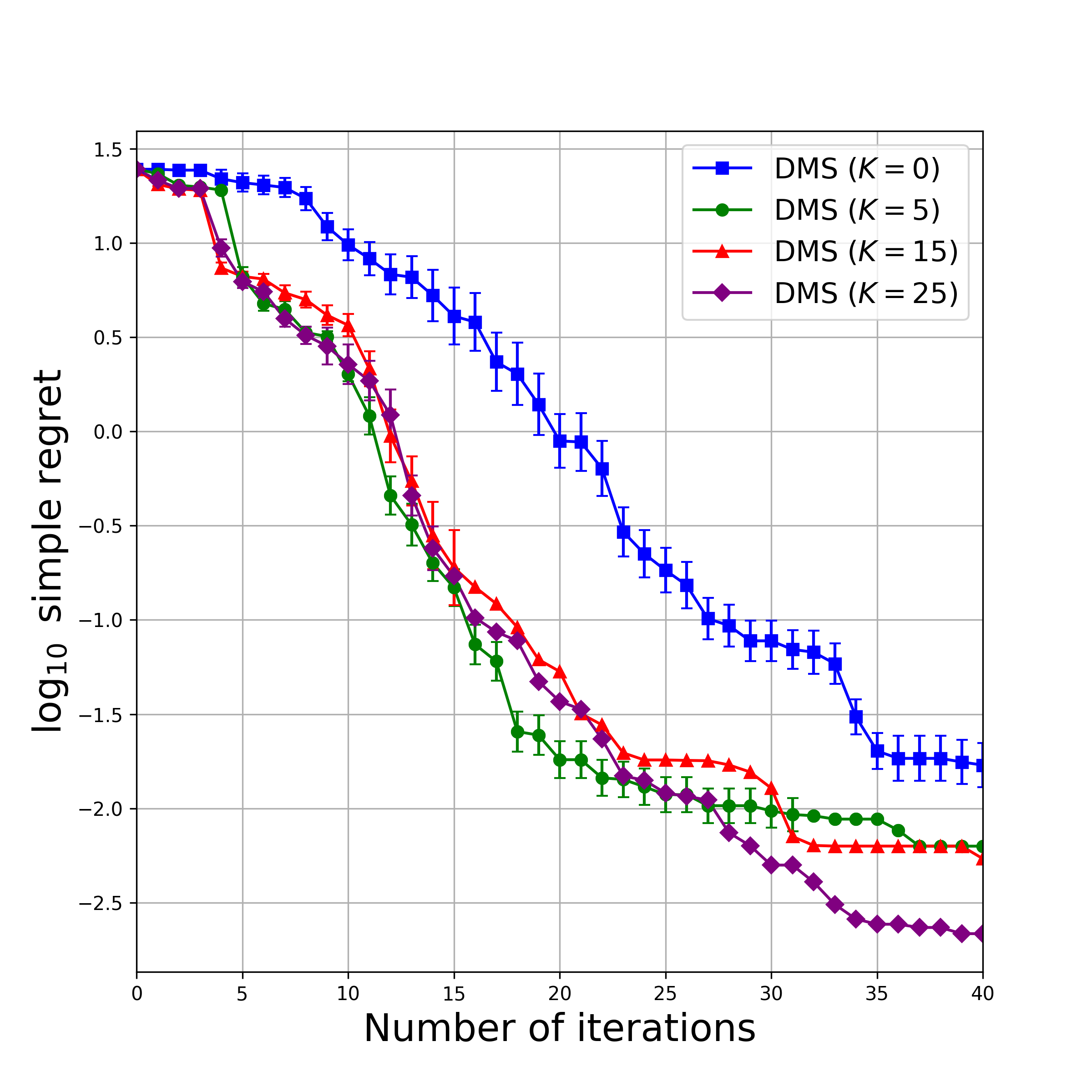}
    \small (a) Styblinski-Tang $(d=2)$
  \end{minipage}\hspace{0.03\textwidth}
  \begin{minipage}{0.40\textwidth}
    \centering
    \includegraphics[width=\linewidth]{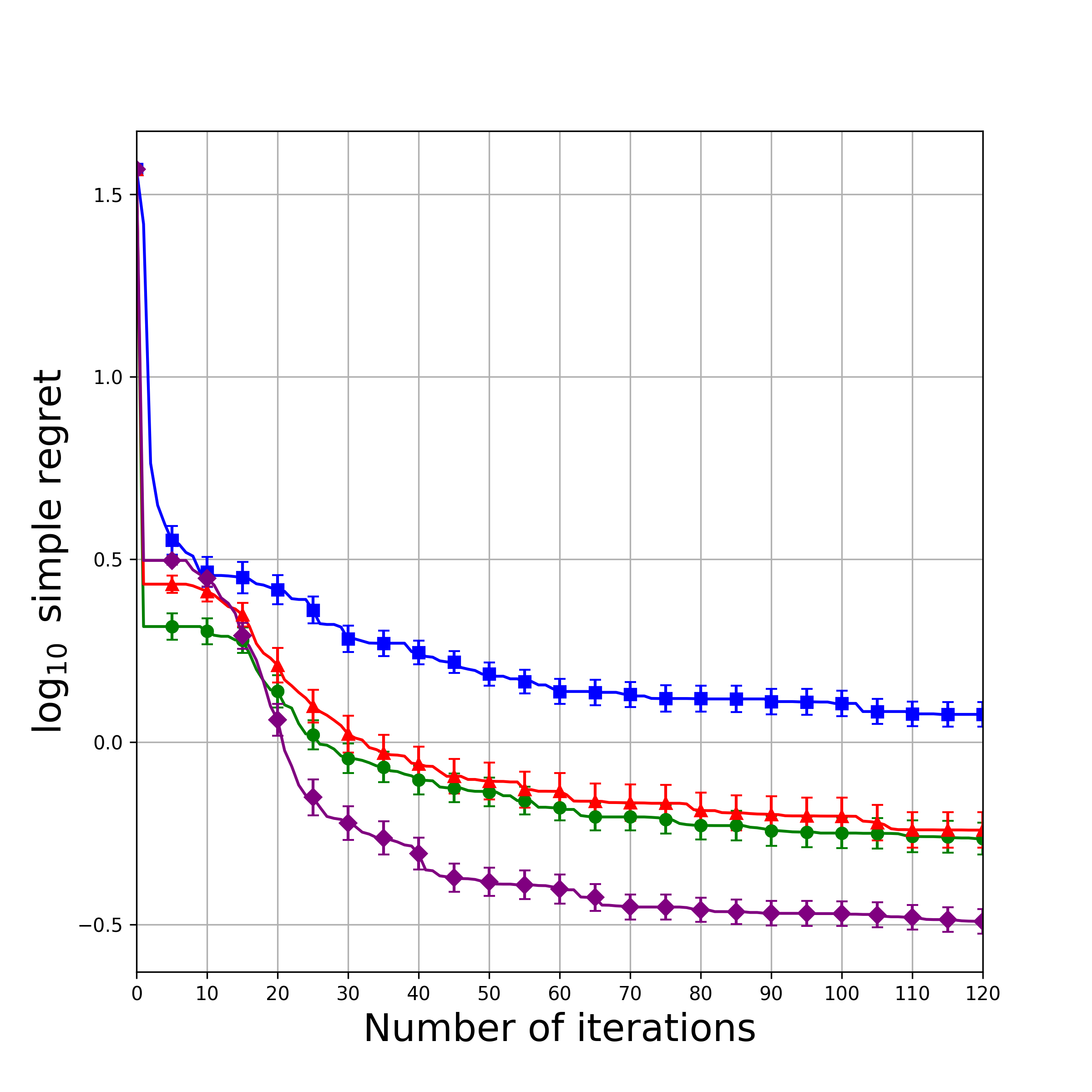}
    \small (b) Levy $(d=10)$
  \end{minipage}\hspace{0.03\textwidth}
\end{minipage}

\caption{Ablation study on the number of short-run L-BFGS optimization steps on selected synthetic benchmark functions. All experiments are conducted with 30 macro-replications, with mean and one standard error reported.}
\label{fig:ablation-L-BFGS}
\end{figure*}

\subsubsection{Ablation on the Size of Pseudo-Dataset}\label{app:ablation-size-pseudo-dataset}
We study the sensitivity of our algorithm to the size of pseudo-dataset $m$ on selected synthetic benchmarks. 

As shown in Figure~\ref{fig:ablation-pseudo-dataset-size}, in Styblinski-Tang $(d=2)$, varying $m$ has a relatively limited impact on the final optimization performance. However, we observe that using $m=100$ leads to a noticeably slower decrease in simple regret compared to larger values such as $m=500$ or $m=1000$. A plausible explanation is that, in low-dimensional settings, a moderate number of $m$ already provides sufficient coverage of the input space, so further increasing $m$ yields diminishing improvement. However, when $m$ is too small, the pseudo-dataset offers a poor approximation of the conditional distribution, slowing CDM training and delaying effective guidance toward high-quality regions.

In contrast, for Levy $(d=10)$, increasing $m$ yields more substantial performance improvements, as higher-dimensional spaces require larger pseudo-datasets to mitigate sparsity and better capture high-value regions, leading to more effective generation of high-quality candidates.

\begin{figure*}[h]
\centering

\begin{minipage}{0.75\textwidth}
\centering
  \begin{minipage}{0.40\textwidth}
    \centering
    \includegraphics[width=\linewidth]{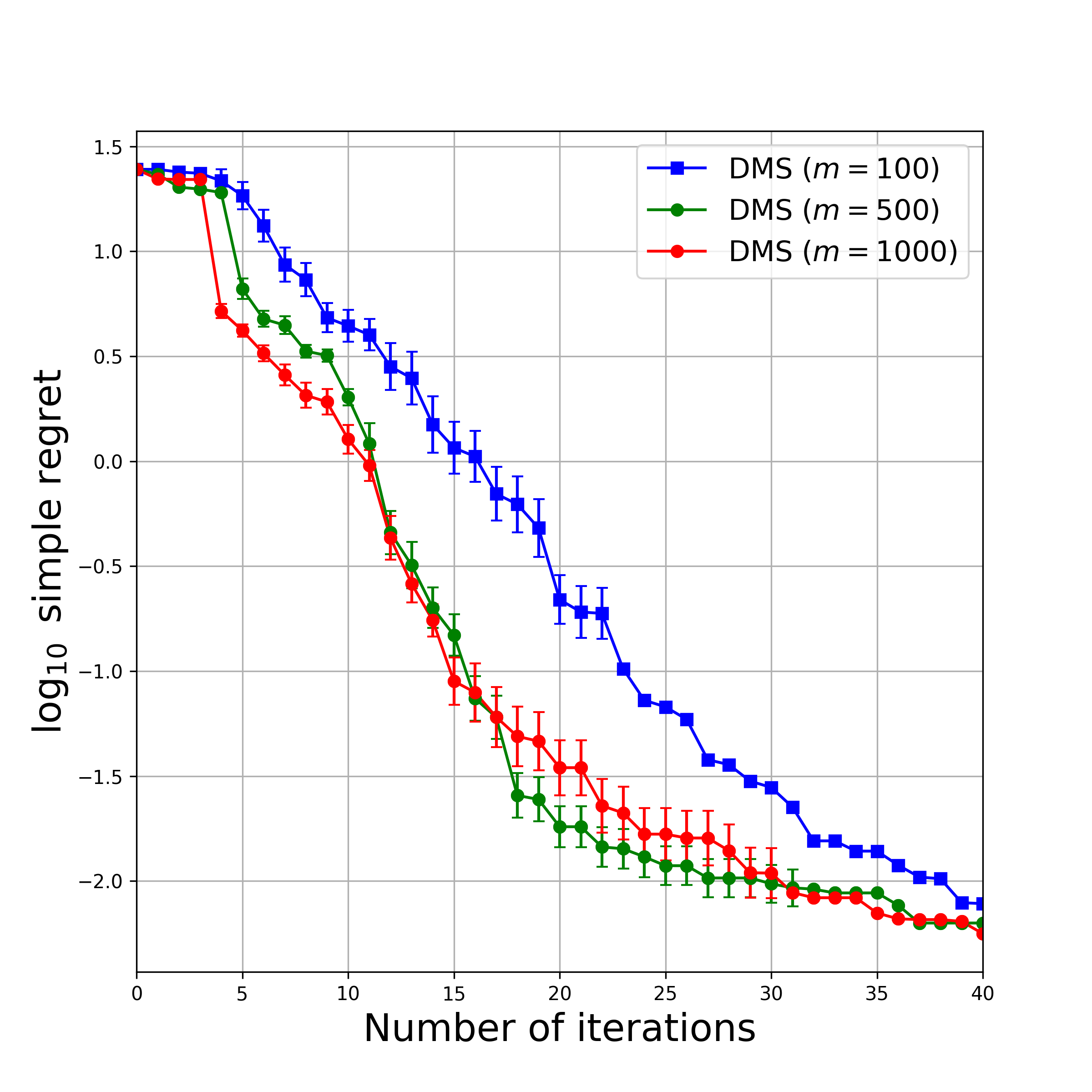}
    \small (a) Styblinski-Tang $(d=2)$
  \end{minipage}\hspace{0.03\textwidth}
  \begin{minipage}{0.40\textwidth}
    \centering
    \includegraphics[width=\linewidth]{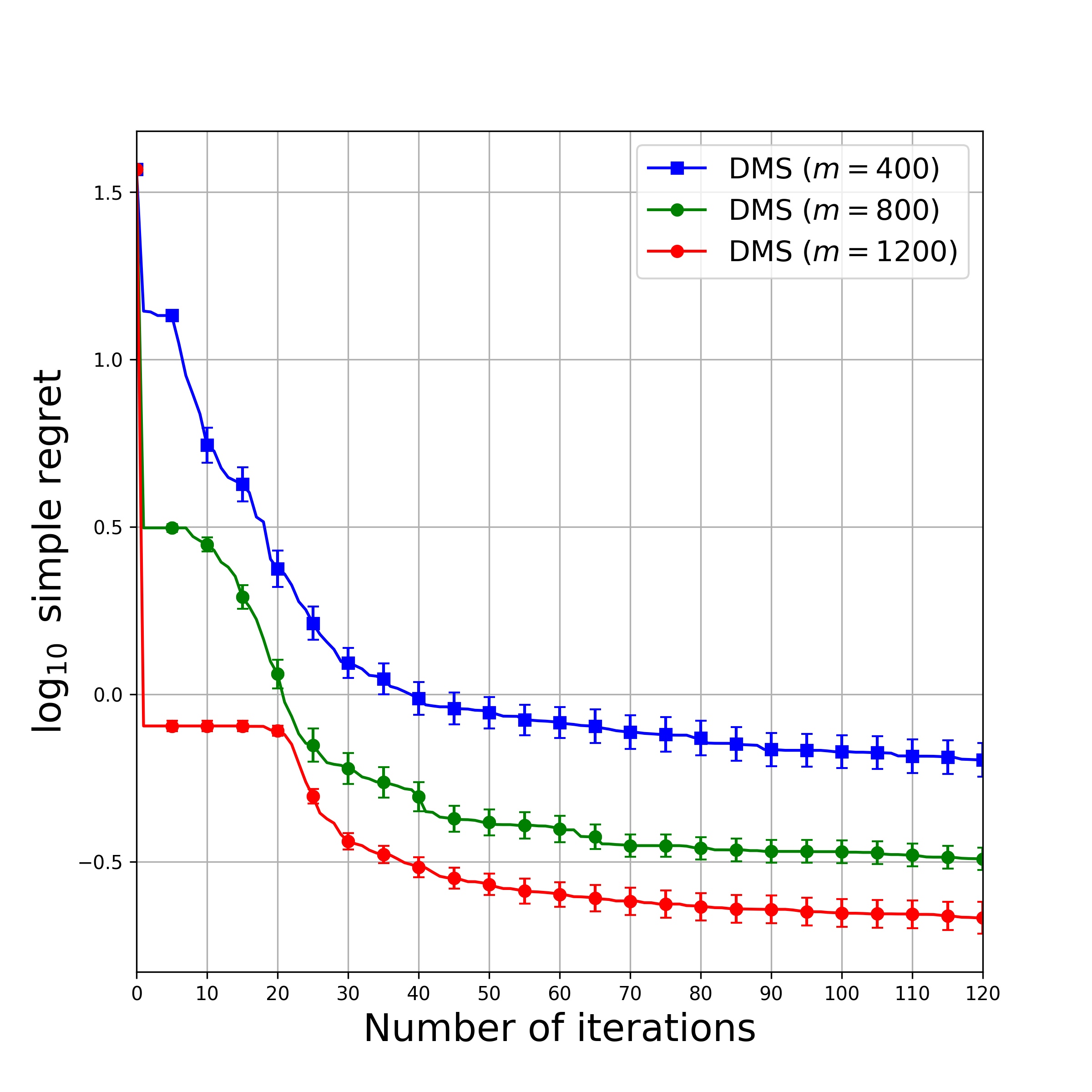}
    \small (b) Levy $(d=10)$
  \end{minipage}\hspace{0.03\textwidth}
\end{minipage}

\caption{Ablation study on the size of pseudo-dataset $m$ on selected synthetic benchmarks. All experiments are conducted with 30 macro-replications, with mean and one standard error reported.}
\label{fig:ablation-pseudo-dataset-size}
\end{figure*}

\subsubsection{Ablation on the Number of $\mathbf{x}^{\star}$ Candidates}\label{app:ablation-number-x*}
We study the sensitivity of DMS to the number of candidate points $S$ on selected synthetic benchmarks. As shown in Figure~\ref{fig:ablation-num-x-star}, with different selections of $S$, we observe no substantial change in optimization performance, suggesting that DMS is relatively insensitive to this hyperparameter once $S$ is sufficiently large. Considering the computational cost of candidate generation and the additional overhead of applying mean-shift clustering to identify cluster centers, we therefore adopt a moderate number of $\mathbf{x}^{\star}$ candidates in all experiments.

\begin{figure*}[h]
\centering

\begin{minipage}{0.75\textwidth}
\centering
  \begin{minipage}{0.40\textwidth}
    \centering
    \includegraphics[width=\linewidth]{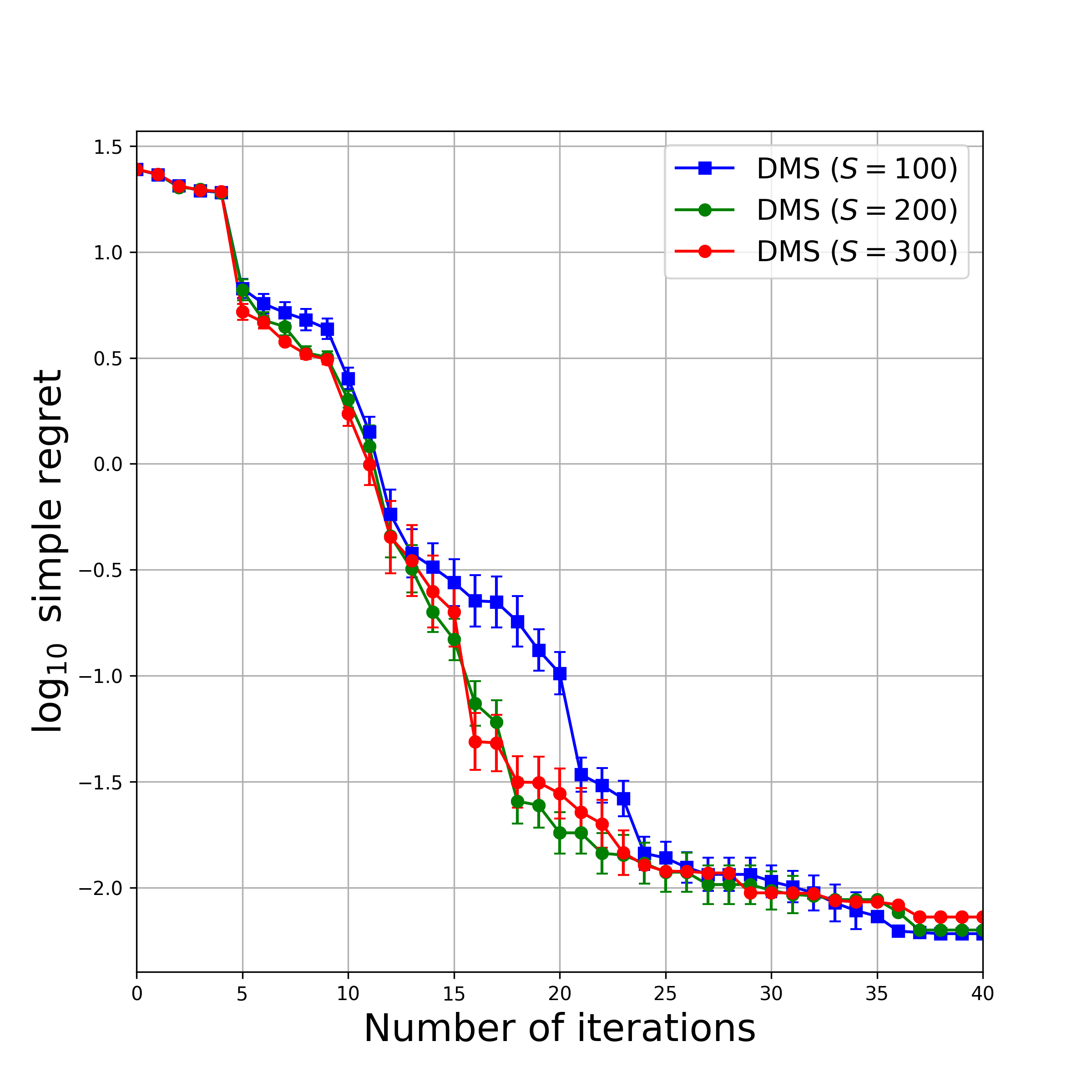}
    \small (a) Styblinski-Tang $(d=2)$
  \end{minipage}\hspace{0.03\textwidth}
  \begin{minipage}{0.40\textwidth}
    \centering
    \includegraphics[width=\linewidth]{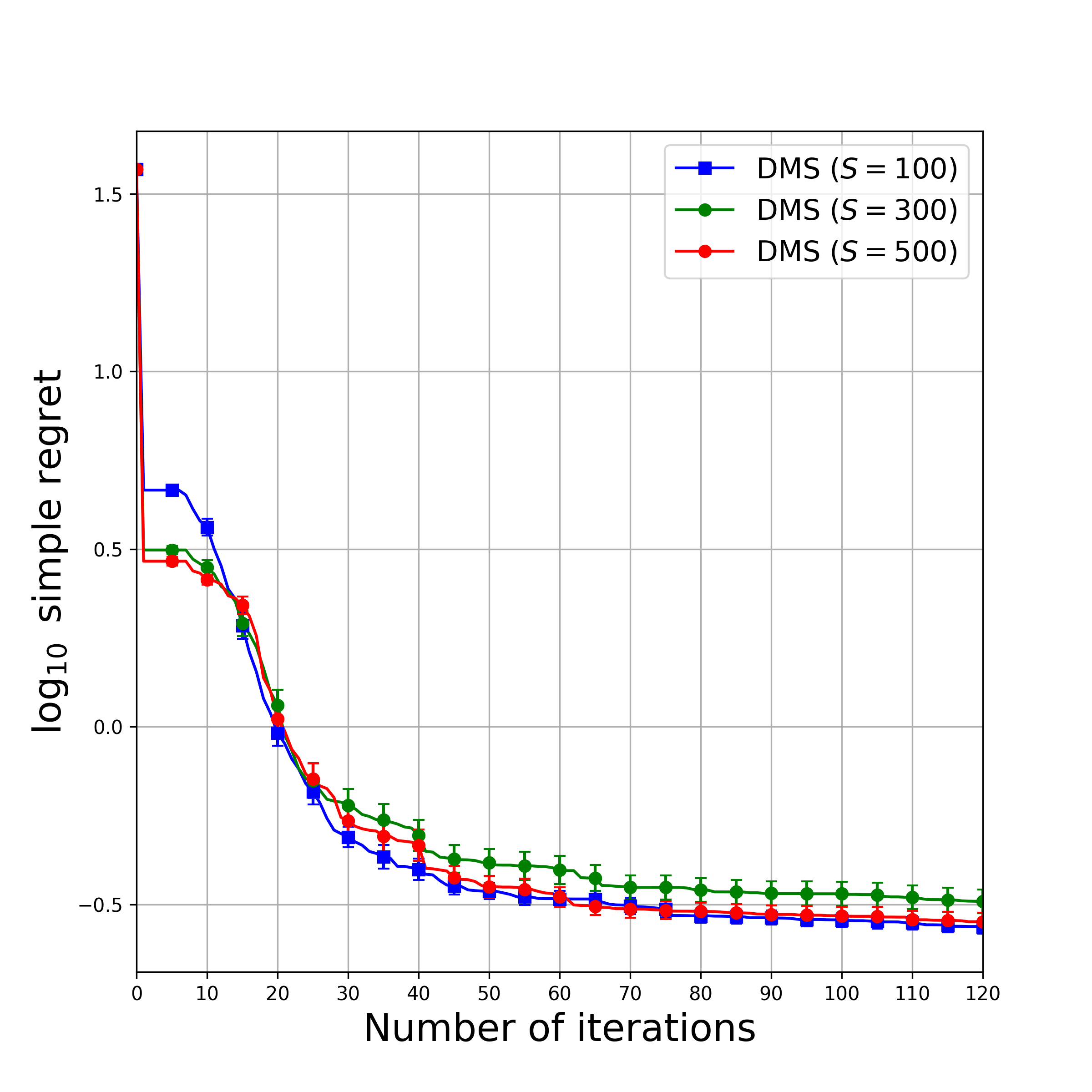}
    \small (b) Levy $(d=10)$
  \end{minipage}\hspace{0.03\textwidth}
\end{minipage}

\caption{
Ablation study on the number of $\mathbf{x}^{\star}$ candidates on selected synthetic benchmark functions. All experiments are conducted with 30 macro-replications, with mean and one standard deviation reported.}
\label{fig:ablation-num-x-star}
\end{figure*}

\newpage
\section{Visualization of GP-induced Distribution and CDM-Learned Distribution}\label{app:distribution-difference}
Now we present a visual comparison between the distribution of $\mathbf{x}^{\star}$ induced by GP posterior and CDM-learned distribution, i.e., $\widehat{P}(\mathbf{x}\mid \widehat{y}=\widehat{y}^*)$. We employ the Styblinski-Tang function $(d=2)$ for the sake of visualization, and extract snapshots at iterations $10$ and $30$ from the first replication of DMS, with the same experimental configurations as described in Appendix~\ref{app:CDM}. These two iterations are chosen to represent the early and late stages of the optimization process, respectively.

\begin{figure*}[h]
\centering

\begin{minipage}{\textwidth}
\centering
  \begin{minipage}{0.235\linewidth}
    \centering
    \includegraphics[width=\linewidth]{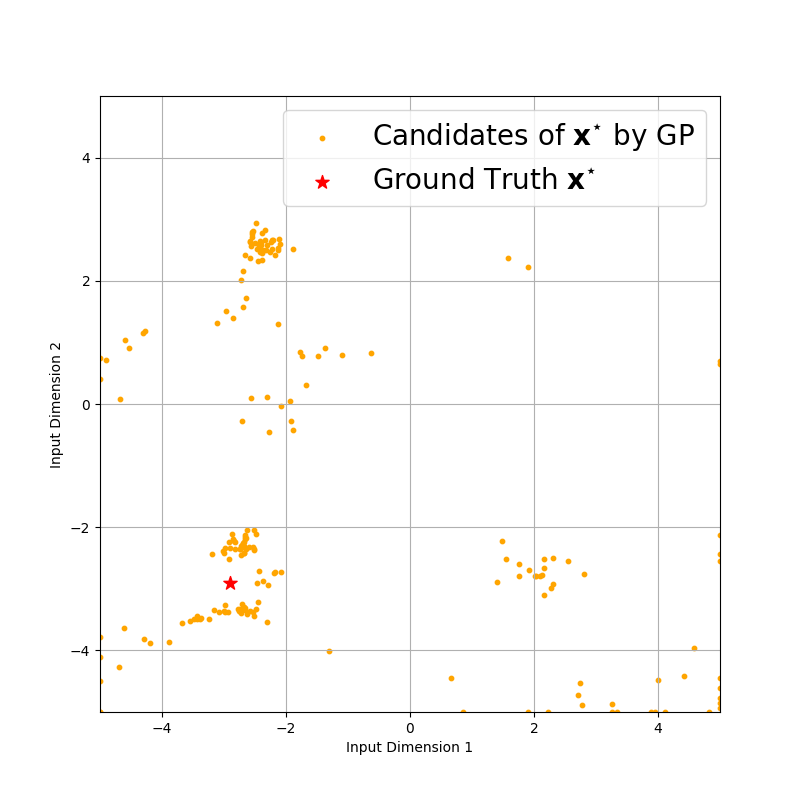}
    \small (a) Candidates by GP
  \end{minipage}\hspace{0.01\linewidth}
  \begin{minipage}{0.235\linewidth}
    \centering
    \includegraphics[width=\linewidth]{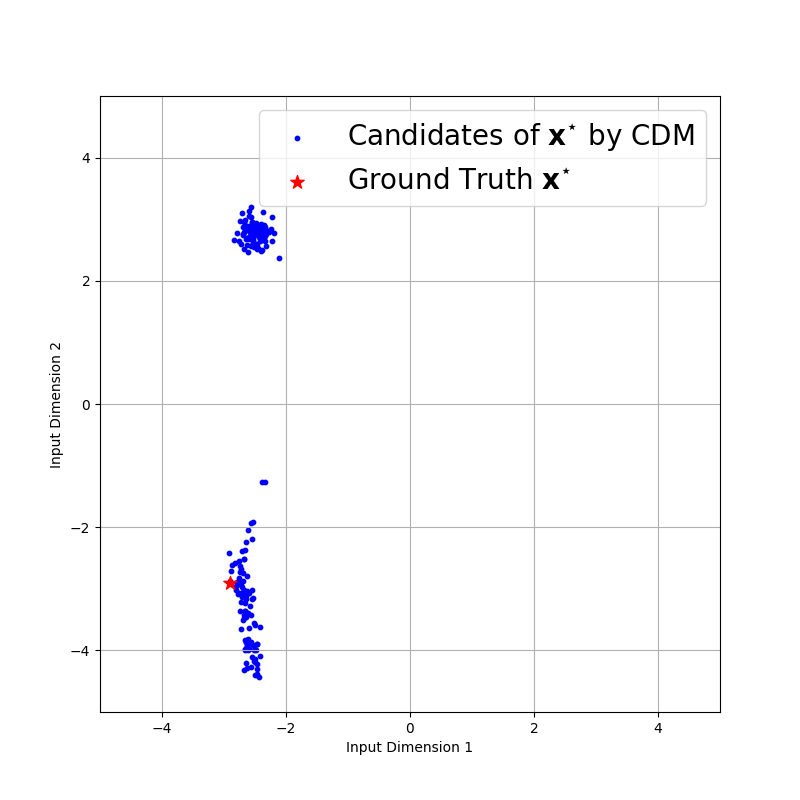}
    \small (b) Candidates by CDM
  \end{minipage}\hspace{0.01\linewidth}
  \begin{minipage}{0.235\linewidth}
    \centering
    \includegraphics[width=\linewidth]{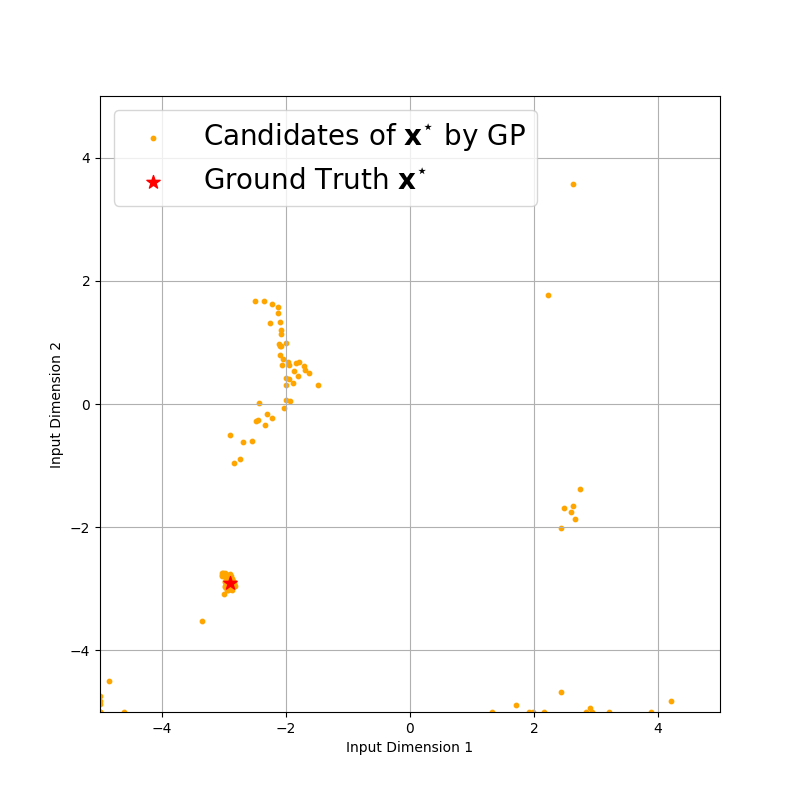}
    \small (c) Candidates by GP
  \end{minipage}\hspace{0.01\linewidth}
  \begin{minipage}{0.235\linewidth}
    \centering
    \includegraphics[width=\linewidth]{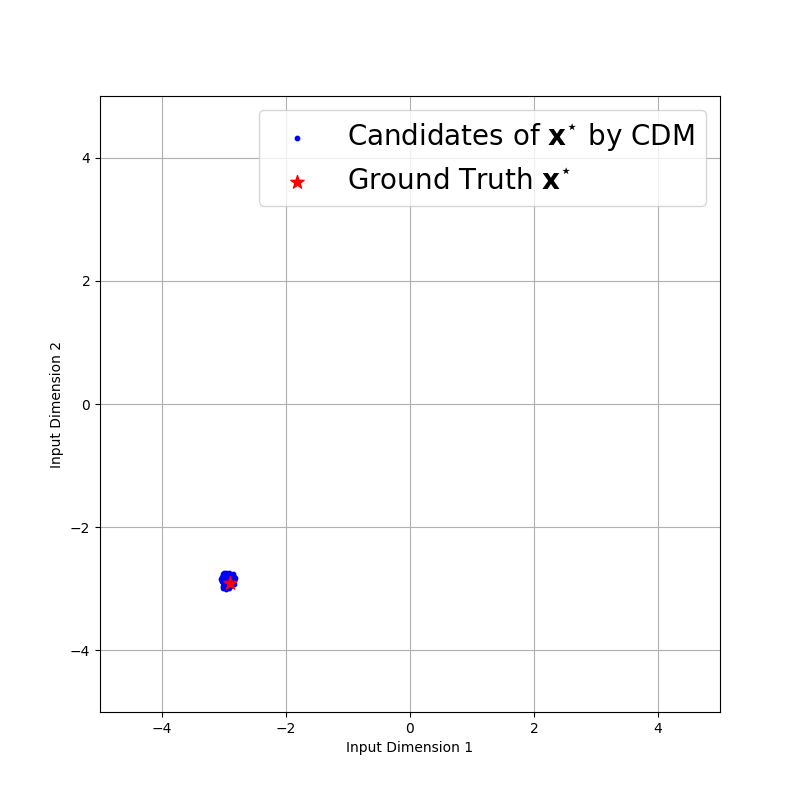}
    \small (d) Candidates by CDM
  \end{minipage}
\end{minipage}

\caption{Comparison between the distributions of $\mathbf{x}^{\star}$ induced by the GP posterior and the CDM. Panels (a)–(b) correspond to iteration $10$, while Panels (c)–(d) correspond to iteration $30$.}
\label{fig:distribution-difference}
\end{figure*}

Figure~\ref{fig:distribution-difference} reveals a clear discrepancy: the GP candidates are obviously dispersed across potential candidate regions, whereas the CDM candidates consistently concentrate around promising input regions, as discussed in Section~\ref{sec:DMS}. 

\section{Computational Complexity Analysis and Wall-Clock Time Comparison}
\label{app:time-complexity}

In this section, we provide additional details on the computational cost of DMS. We first compare the cost of generating $S$ samples of $\mathbf{x}^{\star}$ using PES and DMS. We then report the wall-clock time for generating the same number of $\mathbf{x}^{\star}$ samples. Finally, we report the average wall-clock time per BO iteration across different baselines.

\subsection{Computational Complexity of Generating $\mathbf{x}^{\star}$ Samples}
\label{app:computational-complexity}

We compare the computational complexity of generating $S$ samples $\mathbf{x}^{\star}$ for PES and DMS, respectively. This comparison focuses on the dominant overhead beyond fitting the GP surrogate, since all GP-based acquisition functions share the GP posterior update as a common component.

For PES, generating samples of $\mathbf{x}^{\star}$ typically requires two main steps. First, one samples approximate GP sample paths using random Fourier features. Let $V$ be the number of random features. Constructing the feature-space posterior requires operations involving a $V\times V$ covariance matrix, which contributes a cost of order
\begin{equation}
    \mathcal{O}(V^3+nV^2),
\end{equation}
where $n$ is the number of observed BO data points. The term $V^3$ comes from matrix factorization or inversion in the random-feature space, while the term $nV^2$ comes from incorporating the $n$ observations into the feature-space posterior.

Second, after drawing approximate GP paths, PES needs to optimize each sampled path to obtain a single sample of $\mathbf{x}^{\star}$. Suppose we generate $S$ samples of $\mathbf{x}^\star$, use $R$ random restarts for optimizing each sampled path, and run $K$ optimization steps for each restart. Evaluating a random-feature GP sample path and its derivative has cost proportional to the random-feature dimension and the input dimension, which we write as $\mathcal{O}(Vd)$. Therefore, optimizing all sampled paths contributes
\begin{equation}
    \mathcal{O}(SRKVd).
\end{equation}
Combining the random-feature posterior construction and the path optimization cost, the total complexity of generating $S$ samples of $\mathbf{x}^{\star}$ for PES is
\begin{equation}
    \mathcal{O}\left(V^3+nV^2+SRKVd\right).
\end{equation}
This cost can become large when the number of optimizer samples $S$, the random-feature dimension $V$, the input dimension $d$, or the number of restarts $R$ increases.

For DMS, the generation of $\mathbf{x}^{\star}$ samples is based on the training and sampling of the CDM. The dominant cost consists of two parts: training the conditional score network and sampling from the backward SDE. Let $m$ be the pseudo-dataset size, $B$ be the training batch size, and $E$ be the number of training epochs. If $C_{\mathrm{net}}$ denotes the cost of one score-network evaluation, then the training cost is
\begin{equation}
    \mathcal{O}\left(E\left\lceil \frac{m}{B}\right\rceil C_{\mathrm{net}}\right).
\end{equation}
After training, generating $S$ samples requires simulating the backward SDE in Eq~\ref{eq:backward_process} for $L$ discretization steps. Each step requires one score-network evaluation per sample. Hence, the reverse-time sampling cost is
\begin{equation}
    \mathcal{O}\left(SLC_{\mathrm{net}}\right).
\end{equation}
The total dominant complexity of DMS for generating $S$ samples is therefore
\begin{equation}
    \mathcal{O}\left(
    E\left\lceil \frac{m}{B}\right\rceil C_{\mathrm{net}}
    +
    SLC_{\mathrm{net}}
    \right).
\end{equation}

In our implementation in Section~\ref{app:CDM}, the score network is a lightweight three-layer MLP with hidden size $H$. Therefore, the cost of one score-network evaluation is approximately
\begin{equation}
    C_{\mathrm{net}} \approx \mathcal{O}(H^2),
\end{equation}
up to lower-order terms depending on the input dimension and embedding dimensions. Thus, the DMS sampling cost scales mainly with the number of reverse-time steps $L$, the number of generated samples $S$, and the MLP width $H$. In contrast, PES requires repeated optimization of random-feature GP sample paths, and its cost scales with the random-feature dimension $V$, the number of restarts $R$, and the input dimension $d$. This difference explains why PES becomes significantly more expensive in higher-dimensional problems, while DMS remains relatively efficient once the score network is trained, as we will show in the following section. 

We also note that the pseudo-dataset construction cost in DMS is not the dominant term in our implementation. The short-run L-BFGS refinement is applied for a small fixed number of steps and is lightweight compared with score-network training and backward SDE simulation. Therefore, we omit it from the leading-order complexity expression above.

\subsection{Wall-Clock Time for Generating $\mathbf{x}^{\star}$ Samples}
\label{app:xstar-wall-clock}

We next empirically compare the wall-clock time for generating $S$ samples of $\mathbf{x}^{\star}$ using DMS and PES. We report the results on Styblinski-Tang with $d=2$ and Levy with $d=10$. These two tasks represent a low-dimensional and a moderately higher-dimensional setting, respectively.

\begin{table}[h]
\centering
\caption{Average wall-clock time (/sec) with one standard deviation for sampling $\mathbf{x}^{\star}$.}
\begin{tabular}{llccc}
\hline
Task & Method & $S=100$ & $S=200$ & $S=300$ \\
\hline
Styblinski-2 & DMS & $1.32\pm0.08$ & $1.46\pm0.07$ & $1.60\pm0.13$ \\
Styblinski-2 & PES & $1.25\pm0.04$ & $2.60\pm0.16$ & $5.73\pm0.09$ \\
Levy-10 & DMS & $2.20\pm0.07$ & $2.48\pm0.03$ & $2.67\pm0.06$ \\
Levy-10 & PES & $9.58\pm2.36$ & $14.49\pm3.31$ & $28.74\pm2.55$ \\
\hline
\end{tabular}
\label{tab:xstar_generation_time}
\end{table}

As shown in Table~\ref{tab:xstar_generation_time}, DMS is comparable to PES on the low-dimensional Styblinski-Tang task when $S=100$, and becomes faster as the number of generated samples increases. On the Levy-10 task, DMS is consistently faster than PES across all tested values of $S$. The difference becomes more pronounced when either the input dimension or the number of generated samples increases. This is consistent with the complexity analysis in Appendix~\ref{app:computational-complexity}.

\subsection{Wall-Clock Time per BO Iteration}
\label{app:bo-iteration-time}

Finally, we report the average wall-clock time per BO iteration for all compared methods in the experiments. The purpose of this comparison is to clarify the practical overhead of DMS relative to different classes of BO baselines.

Table~\ref{tab:wall_clock_time} reports the average per-iteration time on Styblinski-2, Levy-10, Levy-20 and Levy-50, showing the overhead of each method across growing input dimensions. Note that PES is marked as N/A on Levy-20 and Levy-50 because its computational cost becomes prohibitively high in our experimental setup. JES is also marked as N/A on Levy-50 for the same reason.

We observe that DMS is indeed slower than simple acquisition functions such as EI and UCB, since DMS additionally trains a CDM and simulates backward SDE. However, DMS still remains in the regime of seconds per BO iteration. By contrast, PES becomes much more expensive on Levy-10 due to the repeated random-feature path optimization required for sampling $\mathbf{x}^{\star}$.

\begin{table}[h]
\centering
\small
\setlength{\tabcolsep}{3pt}
\caption{Average wall-clock time (/sec) with one standard deviation per BO iteration across different tasks.}
\label{tab:wall_clock_time}
\begin{tabular}{lccccccc}
\hline
Task & EI & GIBBON & UCB & DMS & PI & PES & TS \\
\hline
Styblinski-2 & $0.34{\pm}0.11$ & $0.29{\pm}0.13$ & $0.20{\pm}0.08$ & $2.13{\pm}0.20$ & $0.43{\pm}0.09$ & $8.90{\pm}2.50$ & $0.23{\pm}0.12$ \\
Levy-10 & $0.70{\pm}0.37$ & $2.04{\pm}0.15$ & $0.50{\pm}0.16$ & $4.07{\pm}0.40$ & $3.20{\pm}1.05$ & $390.20{\pm}23.40$ & $0.47{\pm}0.18$ \\
Levy-20 & $1.70{\pm}0.67$ & $4.26{\pm}2.39$ & $1.20{\pm}0.43$ & $8.16{\pm}1.71$ & $3.82{\pm}1.58$ & N/A & $1.07{\pm}0.67$ \\
Levy-50 & $2.74{\pm}0.56$ & $6.85{\pm}1.86$ & $3.64{\pm}1.85$ & $12.23{\pm}2.38$ & $3.98{\pm}0.88$ & N/A & $5.99{\pm}2.18$ \\
\hline
\end{tabular}
\end{table}

Overall, these results show that DMS introduces additional overhead compared with simple acquisition functions, but the overhead remains moderate in absolute wall-clock time. In many practical BO applications, a single function evaluation can take minutes or even hours, in which case a few seconds of acquisition overhead is often acceptable. 

\section{Limitations and Broader Impacts}\label{app:limitations-impact}

\paragraph{Limitations}
DMS has three main limitations. First, it involves several hyperparameters, including $K$, $\rho$, and $m$. Although better task-specific configurations may exist, finding them can require substantial additional overhead, and we leave adaptive selection strategies to future work. Second, our theory provides a distribution-level sub-optimality guarantee, but does not yet establish algorithm-level convergence or regret guarantees for the sequential BO procedure. Third, the computational cost of DMS may increase in very high-dimensional settings, especially when \(d \geq 100\), motivating future work on more efficient pseudo-dataset construction, sampling, and dimension-aware architectures.

\paragraph{Broader Impacts}
This work develops a general-purpose method for black-box optimization. It may benefit scientific and engineering applications where evaluations are expensive, such as hyperparameter tuning, experimental design, and automated system optimization. At the same time, like other general optimization methods, its societal impact depends on the downstream application. The method could be used to optimize objectives in domains with either beneficial or harmful consequences. We do not release new datasets involving sensitive personal information, and our experiments are conducted on standard synthetic benchmarks and public hyperparameter optimization benchmarks.

\newpage
\section{Sub-Optimality of $\mathbf{x}^{\star}$ Candidates}\label{app:sub-optimality}
The sub-optimality defined in~\ref{def:sub-optimality} can be decomposed according to the following lemma: 
\begin{lemma}\label{lem:subopt-decomposition}
Recall $\widehat{P}_a:=\widehat{P}(\cdot\mid \widehat{y}=a)$ (CDM-learned distribution) and $P_a:=P(\cdot\mid \widehat{y}=a)$ (distribution induced by the surrogate $\widehat{f}$), similar to the decomposition in \citet{li2024diffusion}, we have
\begin{equation}
    \begin{aligned}
        \operatorname{SubOpt}\left(\widehat{P}_a ; a\right)= & a-\mathbb{E}_{\mathbf{x} \sim \widehat{P}_a}\left[f(\mathbf{x})\right] \\
        \leq & \underbrace{\mathbb{E}_{\mathbf{x} \sim P_a}\left[\left|f( \mathbf{x})-\widehat{f}(\mathbf{x})\right|\right]}_{\mathcal{E}_1}+\underbrace{\left|\mathbb{E}_{\mathbf{x} \sim P_a}\left[f\left(\mathbf{x}\right)\right]-\mathbb{E}_{\mathbf{x} \sim \widehat{P}_a}\left[f\left(\mathbf{x}\right)\right]\right|}_{\mathcal{E}_2}. 
    \end{aligned}
\end{equation}
\end{lemma}

\begin{proof}
    See~\ref{proof:lem:subopt-decomposition}. 
\end{proof}

Now we provide upper bounds for term $\mathcal{E}_1$ and $\mathcal{E}_2$ respectively. 

\subsection{Bound for $\mathcal{E}_1$}\label{app:subsec:bound-for-E_1}
\begin{assumption}\label{asm:RKHS-f}
    We assume that the unknown objective $f$ belongs to the Reproducing Kernel Hilbert Space (RKHS) $\mathcal{H}_k$ induced by the SE kernel $k$, and $\|f\|_{\mathcal{H}_k}\le C$, i.e., 
    \begin{equation}
        f(\mathbf{x})=\left\langle\boldsymbol{\psi}, \Phi(\mathbf{x})\right\rangle_{\mathcal{H}_k}, \quad\left\|\boldsymbol{\psi}\right\|_{\mathcal{H}_k} \leq C,
    \end{equation}
    where $\Phi:\mathcal{X}\to \mathcal{H}_k$ satisfies $k(\mathbf{x},\mathbf{x}^\prime)=\left\langle\Phi(\mathbf{x}), \Phi\left(\mathbf{x}^{\prime}\right)\right\rangle_{\mathcal{H}_k}$.
\end{assumption}
While we focus on the SE kernel to streamline our theoretical derivation, the underlying analytical framework naturally extends to kernels with lower regularity, such as the Matérn family. 

Recall that we have the observed dataset $\mathcal{D}_{n}=\{(\mathbf{x}_i,y_i)\}_{i=1}^n$, with $y_i=f(\mathbf{x}_i)+\epsilon_i$. The GP posterior mean $\mu_n(\mathbf{x})$ introduced in Section~\ref{subsec:GP} is essentially estimating $\boldsymbol{\psi}$ with Kernel Ridge Regression (KRR), i.e., 
\begin{equation}
    \widehat{\boldsymbol{\psi}} \in \arg \min _{\theta \in \mathcal{H}_k} \sum_{i=1}^n\left(\left\langle\boldsymbol{\psi}, \Phi\left(\mathbf{x}_i\right)\right\rangle-y_i\right)^2+\lambda\|\boldsymbol{\psi}\|_{\mathcal{H}_k}^2,
\end{equation}
by which the GP posterior mean can be equivalently represented by $\mu_n(\mathbf{x})=\langle\widehat{\boldsymbol{\psi}}, \Phi(\mathbf{x})\rangle$. 

\begin{lemma}\label{lem:bound-f-minus-mu}
Under Assumption~\ref{asm:RKHS-f}, with high probability, 
\begin{equation}
    \mathbb{E}_{\mathbf{x} \sim P_a}\left[\left|f(\mathbf{x})-\mu_n(\mathbf{x})\right|\right] = \widetilde{\mathcal{O}}\left(\frac{\sigma(\log n)^{d+1}}{\sqrt{n}}\right)
\end{equation}
\end{lemma}
\begin{proof}
    See~\ref{sec:proof:lem:bound-f-minus-mu}. 
\end{proof}

\begin{lemma}\label{lem:bound-E_1}
With high probability, 
\begin{equation}
    \begin{aligned}
        \mathcal{E}_1&=\mathbb{E}_{\mathbf{x} \sim P_a}[|f(\mathbf{x})-\widehat{f}(\mathbf{x})|]\\
        &=\widetilde{\mathcal{O}}\left(\frac{\sigma(\log n)^{d+1}}{\sqrt{n}}+ \frac{\sqrt{\beta}(\log n)^{\left(d+1\right) / 2}}{\sqrt{n}}\right)
    \end{aligned}
\end{equation}
\begin{proof}
    See~\ref{sec:proof:lem:bound-E_1}. 
\end{proof}
\end{lemma}

\subsection{Bound for $\mathcal{E}_2$}\label{app:subsec:bound-for-E_2}
Hereafter, for notational simplicity, we omit the superscript $y$ in $\mathbf{x}_t^y$ and $\mathbf{x}_t^{y, \leftarrow}$ when presenting relevant random variables in the forward and backward SDEs.

\begin{assumption}\label{asm:realizability}
We assume that the ground-truth conditional score function can be represented by a score predictor in the following function class:
\begin{equation}
    \nabla_{\mathbf{x}_t}\log p_t(\mathbf{x}_t\mid y)
    \in
    \mathcal{S}
    =
    \left\{
    \mathbf{s}_{\boldsymbol{\theta}}(\mathbf{x}_t,t,y)
    =
    \frac{\alpha(t)}{h(t)}
    \vartheta_{\boldsymbol{\theta}}(\mathbf{x}_t,t,y)
    -
    \frac{\mathbf{x}_t}{h(t)}
    :
    \vartheta_{\boldsymbol{\theta}}\in\Theta
    \right\},
\end{equation}
where $\vartheta_{\boldsymbol{\theta}}:\mathbb{R}^{d}\times[t_0,T]\times\mathbb{R}\to\mathbb{R}^d$ is represented by an MLP with ReLU activations.
\end{assumption}


Note that this theoretical architecture differs from our implementation in Appendix~\ref{app:CDM} in two aspects. First, the theoretical form above explicitly writes the score predictor as $\mathbf{s}_{\boldsymbol{\theta}}(\mathbf{x}_t,t,y)
    =
    \frac{\alpha(t)}{h(t)}
    \vartheta_{\boldsymbol{\theta}}(\mathbf{x}_t,t,y)
    -
    \frac{\mathbf{x}_t}{h(t)}$, whereas our implementation directly parameterizes $\mathbf{s}_{\boldsymbol{\theta}}$ by a standard MLP. The above form is adopted only to simplify notation in the proof of Lemma~\ref{lem:score-matching-loss}. In particular, the proof first decomposes the conditional score as $\nabla_{\mathbf{x}_t}\log p_t(\mathbf{x}_t,y)
    =
    -\frac{\mathbf{x}_t}{h(t)}
    +
    \frac{\alpha(t)}{h(t)}
    \mathbf{u}(\mathbf{x}_t,t,y)$, and then regards $\vartheta_{\boldsymbol{\theta}}$ as an approximation of $\mathbf{u}$. If we instead use the implementation architecture, the same proof can be written by taking the function class to be the direct MLP class for $\mathbf{s}_{\boldsymbol{\theta}}$. Then the uniform bound on $\frac{\alpha(t)}{h(t)}\vartheta_{\boldsymbol{\theta}}(\mathbf{x}_t,t,y)$ in Lemma~\ref{lem:score-matching-loss} is replaced by the corresponding uniform bound on $\mathbf{s}_{\boldsymbol{\theta}}(\mathbf{x}_t,t,y)+\frac{\mathbf{x}_t}{h(t)}$. The subsequent steps, including the truncation argument, the metric entropy bound, and the empirical-to-population loss comparison, remain unchanged. Therefore, using the direct MLP parameterization changes only the notation and constants in the bound, not the conclusion. Details can be found in the proof of Lemma~\ref{lem:score-matching-loss}. 

Second, our implementation additionally uses positional embeddings for the scalar inputs $t$ and $y$, and uses Mish activations instead of ReLU activations. These choices are made for engineering and empirical performance considerations. Positional embeddings provide a richer representation of scalar conditioning variables and are widely used in deep learning models to encode time or index-dependent inputs \citep{vaswani2017attention, ho2020denoising, nichol2021improved}. Mish is a smooth non-monotonic activation function, which can improve gradient flow and often yields stronger empirical performance than piecewise-linear activations in neural network training \citep{misra2019mish, dubey2022activation}.

\begin{assumption}\label{asm:Novikov-condition}
We assume the Novikov's condition holds, i.e., 
    \begin{equation}
        \exp \left(\frac{1}{2} \int_0^{T-t_0}\left\|\sqrt{\beta(T-t)}\left(\widehat{\mathbf{s}}_\theta\left(\mathbf{x}_t^{\leftarrow}, T-t, y\right)-\nabla \log p_{T-t}\left(\mathbf{x}_t^{\leftarrow} \mid y\right)\right)\right\|_2^2 \mathrm{~d} t\right)<\infty
    \end{equation}
\end{assumption}

\begin{assumption}\label{asm:lipschitz}
    We assume the ground truth objective function $f$ is $L_f$-Lipschitz over the input space $\mathcal{X}$, i.e., 
    \begin{equation}
        \left|f(\mathbf{x})-f\left(\mathbf{x}^{\prime}\right)\right| \leq L_f\left\|\mathbf{x}-\mathbf{x}^{\prime}\right\|_2, \quad \forall \mathbf{x}, \mathbf{x}^{\prime} \in \mathcal{X} .
    \end{equation}
\end{assumption}

Now we are ready to derive the bound for term $\mathcal{E}_2$. We first derive the upper bound $\epsilon_{\mathrm{diff}}$ of the score matching error, i.e.
\begin{equation}
    \frac{1}{T-t_0} \int_{t_0}^T \mathbb{E}_{\left(\mathbf{x}_t, y\right) \sim P_t}\left[\left\|\nabla \log p_t\left(\mathbf{x}_t \mid y\right)-\widehat{\mathbf{s}}_{\boldsymbol{\theta}}\left(\mathbf{x}_t, t, y\right)\right\|_2^2\right] \mathrm{d} t \leq \epsilon_{\mathrm{diff }}^2. 
\end{equation}

\begin{lemma}\label{lem:score-matching-loss}
    Under Assumption~\ref{asm:realizability}, for $\delta\ge 0$, with probability $1-\delta$, the square score matching error is bounded by 
    \begin{equation}
        \epsilon_{\mathrm{diff}}^2=\mathcal{O}\left(\frac{1}{t_0^2} \sqrt{\frac{\mathcal{N}(\mathcal{S}, 1 / m)d \log (1 / \delta)}{m}}\right). 
    \end{equation}
    \begin{proof}
        See~\ref{sec:proof:lem:sore-matching-loss}. 
    \end{proof}
\end{lemma}

\begin{lemma}\label{lem:TV-distance}
Suppose Assumption~\ref{asm:Novikov-condition} holds. With probability $(1-\delta)(1-\eta)$, the conditional total variation distance satisfies
\begin{equation}
       \operatorname{TV}\left(P_{t_0}, \widehat{P}_{t_0}\right) =\mathcal{O} \left(\frac{1}{\eta t_0^2}\sqrt{\frac{\mathcal{N}(\mathcal{S}, 1 / m)d \log (1 / \delta)}{m}}\right),
\end{equation}
where $P_{t_0}$ is the distribution at time $T-t_0$ of the ground truth backward SDE, and $\widehat{P}_{t_0}$ is the distribution at time $T-t_0$ of the learned backward SDE. 
\end{lemma}
\begin{proof} 
    See~\ref{sec:proof:lem:TV-distance}. 
\end{proof}

\begin{lemma}\label{lem:bound-E_2-general}
    Let $P_1$ and $P_2$ be two probability distribution supported on $\mathcal{X}$. Under Assumption~\ref{asm:lipschitz}, we have
    \begin{equation}
        \left|\mathbb{E}_{\mathbf{x} \sim P_1}[f(\mathbf{x})]-\mathbb{E}_{\mathbf{x} \sim P_2}[f(\mathbf{x})]\right|\le L_f\mathrm{W}_1(P_1,P_2)\le L_f \mathrm{diam}(\mathcal{X})\mathrm{TV}(P_1, P_2),
    \end{equation}
    where $\mathrm{W}_1(\cdot, \cdot)$ is the $1$-Wasserstein distance and $\operatorname{diam}(\mathcal{X}):=\sup _{\mathbf{x}, \mathbf{x}^{\prime} \in \mathcal{X}}\left\|\mathbf{x}-\mathbf{x}^{\prime}\right\|_2$. 
\end{lemma}
\begin{proof}
    See~\ref{sec:proof:lem:bound-E_2-general}. 
\end{proof}

\begin{lemma}\label{lem:bound-E_2}
With probability at least $(1-\delta)(1-\eta)$, we have

\begin{equation}
    \mathcal{E}_2=\mathcal{O}\left(\frac{L_f \operatorname{diam}(\mathcal{X}) }{\eta t_0^2} \sqrt{\frac{\mathcal{N}(\mathcal{S}, 1 / m) d \log (1 / \delta)}{m}}\right)
\end{equation}
\begin{proof}
    See~\ref{sec:proof:lem:bound-E_2}. 
\end{proof}
\end{lemma}

\newpage
\section{Proof of Lemmas Omitted in Appendix~\ref{app:sub-optimality}}\label{app:proof-of-lemmas}
\subsection{Proof of Lemma~\ref{lem:subopt-decomposition}}\label{proof:lem:subopt-decomposition}
\begin{proof}
By triangular inequality, we have
\begin{equation}
    \begin{aligned}
        \mathbb{E}_{\mathbf{x} \sim \widehat{P}_a}[f(\mathbf{x})] & \geq \mathbb{E}_{\mathbf{x} \sim P_a}[f(\mathbf{x})]-\left|\mathbb{E}_{\mathbf{x} \sim \widehat{P}_a}[f(\mathbf{x})]-\mathbb{E}_{\mathbf{x} \sim P_a}[f(\mathbf{x})]\right| \\
        & \geq \mathbb{E}_{\mathbf{x} \sim P_a}[\widehat{f}(\mathbf{x})]-\underbrace{\mathbb{E}_{\mathbf{x} \sim P_a}\left[\left|\widehat{f}(\mathbf{x})-f(\mathbf{x})\right|\right]}_{\mathcal{E}_1}-\underbrace{\left|\mathbb{E}_{\mathbf{x} \sim \widehat{P}_a}[f(\mathbf{x})]-\mathbb{E}_{\mathbf{x} \sim P_a}[f(\mathbf{x})]\right|}_{\mathcal{E}_2}. 
    \end{aligned}
\end{equation}
Substituting the result into Definition~\ref{def:sub-optimality} completes the proof. 
\end{proof}

\subsection{Proof of Lemma~\ref{lem:bound-f-minus-mu}}\label{sec:proof:lem:bound-f-minus-mu}
\begin{proof}
We define $\mathbf{V}_{\lambda}:=\sum_{i=1}^n \Phi\left(x_i\right) \otimes \Phi\left(x_i\right)+\lambda \mathbf{I}$, $\|\mathbf{u}\|_{\mathbf{V}_\lambda}:=\sqrt{\langle \mathbf{u},\mathbf{V}_\lambda \mathbf{u}\rangle}$, $\|\mathbf{v}\|_{\mathbf{V}_\lambda^{-1}}:=\sqrt{\langle \mathbf{v},\mathbf{V}_{\lambda}^{-1}\mathbf{v}\rangle}$. 

For any $\mathbf{x}\in\mathcal{X}$, we have
\begin{equation}
\begin{aligned}
            |f(\mathbf{x})-\mu_n(\mathbf{x})|&=\left|\left\langle\boldsymbol{\psi}-\widehat{\boldsymbol{\psi}}, \Phi(\mathbf{x})\right\rangle\right| \\
            &\leq\left\|\boldsymbol{\psi}-\widehat{\boldsymbol{\psi}}\right\|_{\mathbf{V}_\lambda} \cdot\|\Phi(\mathbf{x})\|_{\mathbf{V}_\lambda^{-1}},
\end{aligned}
\end{equation}
Taking expectation with respect to $P_a$, we have
\begin{equation}
    \mathbb{E}_{\mathbf{x}\sim P_a}[|f(\mathbf{x})-\mu_n(\mathbf{x})|]\le \underbrace{\left\|\boldsymbol{\psi}-\widehat{\boldsymbol{\psi}}\right\|_{\mathbf{V}_\lambda}}_{(A)}\cdot \underbrace{\mathbb{E}_{\mathbf{x}\sim P_a}\left[\|\Phi(\mathbf{x})\|_{\mathbf{V}_\lambda^{-1}}\right]}_{(B)}.
\end{equation}

For term $(A)$, under Assumption~\ref{asm:RKHS-f}, and according to the results of kernelized self-normalized bound in \citet{chowdhury2017kernelized}, with probability at least $1-\delta_1$
\begin{equation}
    \left\|\widehat{\boldsymbol{\psi}}-\boldsymbol{\psi}\right\|_{\mathbf{V}_\lambda} \leq \sigma \sqrt{2\left(\gamma_n(\lambda)+\log \frac{1}{\delta_1}\right)}+\sqrt{\lambda} C, 
\end{equation}
where $\gamma_n(\lambda):=\frac{1}{2}\log \det(\mathbf{I}+\lambda^{-1}\mathbf{K}_n)$ with $\mathbf{K}_n=[k(\mathbf{x}_i,\mathbf{x}_j)]_{\mathbf{x}_i,\mathbf{x}_j\in\mathcal{D}_n}$ the Gram matrix. 

Moreover, we can verify that when using SE kernel function, there exists a constant $C_0$, such that $\gamma_n(\lambda)\le C_0(\log n)^{d+1}$. Therefore, term $(A)$ can be rewritten as $\widetilde{\mathcal{O}}\left(\sigma(\log n)^{\left(d+1\right) / 2}+\sigma \sqrt{\log \left(1 / \delta_1\right)}+\sqrt{\lambda} C\right)$. 

For term $(B)$, we have
\begin{equation}
    \begin{aligned}
        \mathbb{E}_{P_a}\|\Phi(\mathbf{x})\|_{\mathbf{V}_\lambda^{-1}} &\leq \sqrt{\mathbb{E}_{P_a}\|\Phi(\mathbf{x})\|_{\mathbf{V}_\lambda^{-1}}^2}\\
        &=\sqrt{\mathbb{E}_{P_a}\left\langle\Phi(\mathbf{x}), \mathbf{V}_\lambda^{-1} \Phi(\mathbf{x})\right\rangle}\\
        &=\sqrt{\operatorname{Tr}\left(\mathbf{V}_\lambda^{-1} \boldsymbol{\Sigma}_{P_a}^{(\Phi)}\right)},
    \end{aligned}
\end{equation}
where we define $\boldsymbol{\Sigma}_{P_a}^{(\Phi)}:=\mathbb{E}_{\mathbf{x} \sim P_a}[\Phi(\mathbf{x}) \otimes \Phi(\mathbf{x})]$. 

Recall that in our analysis, the target conditional distribution $P_a:=P(\mathbf{x} \mid \widehat{y}=a)$ is induced through the construction $\widehat{y}=\widehat{f}(x)+\xi$, where $\xi \sim \mathcal{N}\left(0, \nu^2\right)$ is an independent Gaussian perturbation introduced for technical convenience. As a consequence, $P_a$ admits an explicit density with respect to the reference distribution $P_{\mathbf{x}}$ (e.g. a uniform distribution on $\mathcal{X}$), given by

\begin{equation}
    \frac{\mathrm{d} P_a}{\mathrm{d} P_\mathbf{x}}(\mathbf{x}) \propto \exp \left(-\frac{(\widehat{f}(\mathbf{x})-a)^2}{2 \nu^2}\right) .
\end{equation}

Therefore, $P_a$ is absolutely continuous with respect to $P_\mathbf{x}$, and for any non-negative measurable function $g$, we have

\begin{equation}
    \mathbb{E}_{\mathbf{x} \sim P_a}[g(x)]=\frac{\mathbb{E}_{\mathbf{x} \sim P_\mathbf{x}}\left[g(\mathbf{x}) \exp \left(-\frac{(\widehat{f}(\mathbf{x})-a)^2}{2 \nu^2}\right)\right]}{Z(a)},
\end{equation}
where $Z(a):=\mathbb{E}_{\mathbf{x} \sim P_x}\left[\exp \left(-\frac{(\widehat{f}(\mathbf{x})-a)^2}{2 \nu^2}\right)\right]$ is the normalizing constant.

In this work, we restrict attention to target values $a$ within a non-extreme range (recall that we only assign $a=\widehat{y}^{\star}$ in our experiments), such that the reference distribution provides sufficient coverage of the induced conditional distributions. In particular, we assume that there exists a constant $Z_*>0$ satisfying $\inf _{a \in \mathcal{A}} Z(a) \geq Z_*$, where $\mathcal{A}$ denotes the set of target values within a non-extreme range. Defining $\rho:=1 / Z_*$, it follows that
\begin{equation}
    \mathbb{E}_{\mathbf{x} \sim P_a}[g(\mathbf{x})] \leq \rho \mathbb{E}_{\mathbf{x} \sim P_\mathbf{x}}[g(\mathbf{x})], \quad \forall a \in \mathcal{A} .
\end{equation}

Applying this inequality to $g(\mathbf{x})=\|\Phi(\mathbf{x})\|_{\mathbf{V}_\lambda^{-1}}^2$ gives
\begin{equation}
    \mathbb{E}_{\mathbf{x} \sim P_a}\|\Phi(\mathbf{x})\|_{\mathbf{V}_\lambda^{-1}}^2 \leq \rho \mathbb{E}_{\mathbf{x} \sim P_\mathbf{x}}\|\Phi(\mathbf{x})\|_{\mathbf{V}_\lambda^{-1}}^2
\end{equation}

Furthermore, by standard information-gain arguments for kernelized regression and GP \citep{srinivas2009gaussian, chowdhury2017kernelized}, the average posterior variance under the reference distribution satisfies

\begin{equation}
    \mathbb{E}_{\mathbf{x} \sim P_\mathbf{x}}\|\Phi(\mathbf{x})\|_{\mathbf{V}_\lambda^{-1}}^2 \leq \frac{2 \gamma_n(\lambda)}{n} .
\end{equation}

Combining the above inequalities and taking square roots yields

\begin{equation}
    \mathbb{E}_{\mathbf{x} \sim P_a}\|\Phi(\mathbf{x})\|_{\mathbf{V}_\lambda^{-1}} \leq \sqrt{\frac{2 \rho \gamma_n(\lambda)}{n}} .
\end{equation}

Combining $(A)$ and $(B)$ and choosing $\lambda=1$ gives
\begin{equation}
\begin{aligned}
            \mathbb{E}_{\mathbf{x} \sim P_a}\left[\left|f(\mathbf{x})-\mu_n(\mathbf{x})\right|\right]&\le \left(\sigma \sqrt{2\left(\gamma_n(1)+\log \frac{1}{\delta_1}\right)}+C\right) \cdot \sqrt{\frac{2 \rho \gamma_n(1)}{n}}\\
            &\le \frac{2 \sigma \sqrt{\rho}}{\sqrt{n}}\left(\gamma_n(1)+\sqrt{\gamma_n(1) \log \frac{1}{\delta_1}}\right)+C \sqrt{\frac{2 \rho \gamma_n(1)}{n}}
\end{aligned}
\end{equation}
Moreover, we can verify that when using SE kernel function, there exists a constant $C_0$ such that $\gamma_n(1)\le C_0 (\log (n))^{d+1}$. Therefore, with probability at least $1-\delta_1$, we have
\begin{equation}
\begin{aligned}
    \mathbb{E}_{\mathbf{x} \sim P_a}\left[\left|f(\mathbf{x})-\mu_n(\mathbf{x})\right|\right]&\le \frac{\sqrt{\rho}}{\sqrt{n}}\left[\mathcal{O}\left(\sigma(\log n)^{d+1}\right)+\mathcal{O}\left(\sigma(\log n)^{\frac{d+1}{2}} \sqrt{\log \frac{1}{\delta_1}}\right)+\mathcal{O}\left(C(\log n)^{\frac{d+1}{2}}\right)\right]\\
    &=\widetilde{\mathcal{O}}\left( \frac{\sigma(\log n)^{d+1}}{\sqrt{n}}\right),
\end{aligned}
\end{equation}
which completes the proof. 
\end{proof}

\subsection{Proof of Lemma~\ref{lem:bound-E_1}}\label{sec:proof:lem:bound-E_1}
\begin{proof}
Recall we use the surrogate $\widehat{f}(\mathbf{x})=\mu_n(\mathbf{x})+\sigma_n(\mathbf{x})$. We have
\begin{equation}
    \begin{aligned}
        \mathcal{E}_1&=\mathbb{E}_{P_a}[|\widehat{f}(\mathbf{x})-f(\mathbf{x})|] \\
        & \leq \underbrace{\mathbb{E}_{P_a}\left[\left|f(\mathbf{x})-\mu_n(x)\right|\right]}_{(A)}+\underbrace{\sqrt{\beta} \ \mathbb{E}_{P_a}\left[\sigma_n(\mathbf{x})\right]}_{(B)}.
    \end{aligned}
\end{equation}

Term $(A)$ can be bounded by Lemma~\ref{lem:bound-f-minus-mu}. 

For term $(B)$, similar to the arguments in the proof of Lemma~\ref{lem:bound-f-minus-mu}, we have
\begin{equation}
    \begin{aligned}
        \sqrt{\beta} \mathbb{E}_{P_a}\left[\sigma_n(\mathbf{x})\right] &\leq \sqrt{\beta} \rho \mathbb{E}_{P_\mathbf{x}}\left[\sigma_n(\mathbf{x})\right]\\
        &\le \sqrt{\beta}\rho \sqrt{\mathbb{E}_{P_x}\left[\sigma_n^2(\mathbf{x})\right]}\\
        &\lesssim \sqrt{\beta}\rho \frac{\gamma_n(1)}{n}\\
        &=\widetilde{\mathcal{O}}\left(\sqrt{\beta} \frac{(\log n)^{\left(d+1\right) / 2}}{\sqrt{n}}\right)
    \end{aligned}
\end{equation}

Combining the bounds for $(A)$ and $(B)$ completes the proof.
\end{proof}

\subsection{Proof of Lemma~\ref{lem:score-matching-loss}}\label{sec:proof:lem:sore-matching-loss}

\begin{proof}
    We first clarify the role of the short-run L-BFGS step in the score-matching analysis. Let $P_{\mathrm{init}}$ denote the initialization distribution used to generate the initial pseudo-inputs, e.g., the Sobol initialization distribution over the compact domain $\mathcal X$. Given the current GP surrogate and the balance-aware pseudo-label function $\widehat f$, the $K$-step short-run L-BFGS procedure defines a measurable map
    \begin{equation}
        \Phi_K:\mathcal X\rightarrow \mathcal X,
    \end{equation}
    where $\Phi_K(\mathbf z)$ denotes the point obtained after applying $K$ steps of short-run L-BFGS starting from $\mathbf z\sim P_{\mathrm{init}}$. Since each update is followed by projection onto the compact domain $\mathcal X$, we have $\Phi_K(\mathbf z)\in \mathcal X$ for all $\mathbf z\in\mathcal X$. Therefore, the refined pseudo-input
    \begin{equation}
        \mathbf x_0=\Phi_K(\mathbf z),\qquad \mathbf z\sim P_{\mathrm{init}},
    \end{equation}
    induces a valid probability distribution on $\mathcal X$, denoted by $P_{\mathbf x}:=(\Phi_K)_{\#}P_{\mathrm{init}}$, where $(\Phi_K)_{\#}P_{\mathrm{init}}$ is the pushforward distribution of $P_{\mathrm{init}}$ under $\Phi_K$. The corresponding pseudo-label is given by $y=\widehat f(\mathbf x_0)$. 
    Hence, the pseudo-training pairs used to train the conditional diffusion model are sampled from the joint distribution induced by
    \begin{equation}
        \mathbf x_0\sim P_{\mathbf x},\qquad y=\widehat f(\mathbf x_0).
    \end{equation}
    The short-run L-BFGS step is therefore absorbed into the data-generating distribution $P_{\mathbf x}$. It does not introduce an additional score-matching error term. Once $P_{\mathbf x}$ is a well-defined probability distribution supported on $\mathcal X$, the empirical-to-population score-matching argument applies with respect to this distribution.

    We now derive a decomposition of the conditional score function similar to \citep{li2024diffusion}. Hereafter, for notational simplicity, we omit the hat in $\widehat{y}$ and use $y$ to denote the conditioning variable. Recall that we take the training input distribution for $\mathbf{x}$ as $P_{\mathbf{x}}$, with density $p_{\mathbf{x}}$. In our algorithm, as clarified above, $P_{\mathbf{x}}$ is the pushforward distribution induced by the initialization distribution and the finite-step projected short-run L-BFGS map. Recall that in Appendix~\ref{app:transition-kernel}, the transition kernel $p_t(\mathbf{x}_t \mid \mathbf{x}_0)$ is denoted by $\mathcal{N}(\mathbf{x}_t; \alpha(t)\mathbf{x}_0,h(t)\mathbf{I})$. We have
    \begin{equation}
        \begin{aligned}
            p_t(\mathbf{x}_t, y)
            &= \int p_t(\mathbf{x}_t, y\mid \mathbf{x}_0)p_\mathbf{x}(\mathbf{x}_0) \mathrm{d}\mathbf{x}_0\\
            &= \int p_t(\mathbf{x}_t\mid \mathbf{x}_0)\cdot p(y\mid \mathbf{x}_0)p_\mathbf{x}(\mathbf{x}_0) \mathrm{d}\mathbf{x}_0 .
        \end{aligned}
    \end{equation}
    Taking the logarithm and the derivative with respect to $\mathbf{x}_t$ on $p_t(\mathbf{x}_t, y)$, we have
    \begin{equation}
    \begin{aligned}
        \nabla_{\mathbf{x}_t}\log p_t(\mathbf{x}_t, y)
        &=\frac{1}{p_t(\mathbf{x}_t, y)}
        \int \nabla_{\mathbf{x}_t}p_t(\mathbf{x}_t\mid \mathbf{x}_0)
        p(y\mid \mathbf{x}_0)p_\mathbf{x}(\mathbf{x}_0)\mathrm{d}\mathbf{x}_0 \\
        &=\int \nabla_{\mathbf{x}_t} \log p_t\left(\mathbf{x}_t \mid \mathbf{x}_0\right)
        \underbrace{
        \frac{
        p_t\left(\mathbf{x}_t \mid \mathbf{x}_0\right)
        p\left(y \mid \mathbf{x}_0\right)
        p_\mathbf{x}\left(\mathbf{x}_0\right)}
        {p_t\left(\mathbf{x}_t, y\right)}
        }_{=: p_t\left(\mathbf{x}_0 \mid \mathbf{x}_t, y\right)}
        \mathrm{d}\mathbf{x}_0\\
        &=\int -\frac{\mathbf{x}_t-\alpha(t)\mathbf{x}_0}{h(t)}
        p_t(\mathbf{x}_0\mid \mathbf{x}_t,y)\mathrm{d}\mathbf{x}_0\\
        &=-\frac{\mathbf{x}_t}{h(t)}
        +\frac{\alpha(t)}{h(t)}
        \underbrace{
        \int \mathbf{x}_0p_t(\mathbf{x}_0\mid \mathbf{x}_t,y)\mathrm{d}\mathbf{x}_0
        }_{:=\mathbf{u}(\mathbf{x}_t,t,y)} .
    \end{aligned}
    \end{equation}

    Hence, the score function can be compactly written as
    \begin{equation}
        \nabla_{\mathbf{x}_t} \log p_t\left(\mathbf{x}_t, y\right)
        =
        -\frac{\mathbf{x}_t}{h(t)}
        +
        \frac{\alpha(t)}{h(t)} \mathbf{u}(\mathbf{x}_t,t,y).
    \end{equation}
    This result motivates us to assume the neural network architecture in Assumption~\ref{asm:realizability}, where $\vartheta$ aims to approximate $\mathbf{u}$.

    Recall that we estimate the conditional score function by minimizing the denoising score matching loss as introduced in Section~\ref{subsec:CDM}. For a single pseudo-training pair $(\mathbf{x}_0,y)$ sampled from the joint distribution induced by $\mathbf{x}_0\sim P_{\mathbf{x}}$ and $y=\widehat f(\mathbf{x}_0)$, define
    \begin{equation}
        \ell(\mathbf{x}_0, y; \mathbf{s}_{\boldsymbol{\theta}})
        :=
        \frac{1}{T-t_0}
        \int_{t_0}^T
        \mathbb{E}_{\mathbf{x}_t\mid \mathbf{x}_0}
        \left[
        \left\|
        \nabla_{\mathbf{x}_t} \log p_t\left(\mathbf{x}_t \mid \mathbf{x}_0\right)
        -
        \mathbf{s}_{\boldsymbol{\theta}}\left(\mathbf{x}_t, t,y\right)
        \right\|_2^2
        \right]
        \mathrm{d}t .
    \end{equation}
    The population loss and empirical loss are defined as
    \begin{equation}
        \mathcal{L}(\mathbf{s}_{\boldsymbol{\theta}})
        :=
        \mathbb{E}_{(\mathbf{x}_0,y)\sim P_{\mathbf{x},y}}
        [\ell (\mathbf{x}_0,y;\mathbf{s}_{\boldsymbol{\theta}})],
        \qquad
        \widehat{\mathcal{L}}(\mathbf{s}_{\boldsymbol{\theta}})
        :=
        \frac{1}{m} \sum_{i=1}^{m}
        \ell\left(\mathbf{x}_i, y_i ; \mathbf{s}_{\boldsymbol{\theta}}\right),
    \end{equation}
    where $P_{\mathbf{x},y}$ denotes the joint distribution of the short-run L-BFGS refined pseudo-input and its balance-aware pseudo-label. This makes explicit that the population loss is taken with respect to the same distribution from which the empirical pseudo-training pairs are sampled.

    Similarly, we define the truncated loss for a single data pair as
    \begin{equation}
        \ell^{\text {trunc }}
        \left(\mathbf{x}_0, y ; \mathbf{s}_{\boldsymbol{\theta}}\right)
        :=
        \ell\left(\mathbf{x}_0, y ; \mathbf{s}_{\boldsymbol{\theta}}\right)
        \mathbf{1}\left\{
        \left\|\mathbf{x}_0\right\|_2 \leq R,\ |y| \leq R
        \right\},
    \end{equation}
    where $R$ is a truncation radius. Since the input space $\mathcal{X} \subset \mathbb{R}^d$ is compact and $P_{\mathbf{x}}$ is supported on $\mathcal X$, there exists a constant $R_{\mathcal{X}}>0$ such that $\left\|\mathbf{x}_0\right\|_2 \leq R_{\mathcal{X}}$ almost surely under $P_{\mathbf{x}}$. Moreover, the balance-aware pseudo-label $y$, with the hat omitted for notational simplicity, is uniformly bounded over $\mathcal{X}$ due to standard properties of GP, i.e., $|y| \leq R_y$. Consequently, we set the truncation radius as
    \begin{equation}
        R=\max \left\{R_{\mathcal{X}}, R_y\right\}.
    \end{equation}

    Additionally, we denote by $\mathcal{L}^{\text {trunc }}(\mathbf{s}_{\boldsymbol{\theta}})$ and $\widehat{\mathcal{L}}^{\text {trunc }}(\mathbf{s}_{\boldsymbol{\theta}})$ the corresponding population and empirical loss functions, respectively. We also denote by $K_{\mathrm{net}}$ the uniform upper bound of
    $\frac{\alpha(t)}{h(t)}\vartheta(\mathbf{x}_t,t,y)\mathbf{1}\left\{\left\|\mathbf{x}_0\right\|_2 \leq R,\ |y| \leq R\right\}$, i.e.,
    \begin{equation}
        \sup_{\vartheta\in\Theta}
        \left\|
        \frac{\alpha(t)}{h(t)}
        \vartheta(\mathbf{x}_t,t,y)
        \mathbf{1}\left\{
        \left\|\mathbf{x}_0\right\|_2 \leq R,\ |y| \leq R
        \right\}
        \right\|_2
        \le K_{\mathrm{net}} .
    \end{equation}
    Here we use $K_{\mathrm{net}}$ to avoid notational confusion with the number $K$ of short-run L-BFGS steps.

    Following the oracle inequality as introduced in \citet{li2024diffusion}, we have
    \begin{equation}
        \begin{aligned}
            \mathcal{L}(\widehat{\mathbf{s}}_{\boldsymbol{\theta}})
            \le
            \underbrace{
            \sup_{\mathbf{s}_{\boldsymbol{\theta}}\in\mathcal{S}}
            \mathcal{L}^{\text {trunc }}(\mathbf{s}_{\boldsymbol{\theta}})
            -
            \widehat{\mathcal{L}}^{\text {trunc}}(\mathbf{s}_{\boldsymbol{\theta}})
            }_{(A)}
            +
            \underbrace{
            \sup_{\mathbf{s}_{\boldsymbol{\theta}}\in\mathcal{S}}
            \mathcal{L}(\mathbf{s}_{\boldsymbol{\theta}})
            -
            \mathcal{L}^{\text {trunc }}(\mathbf{s}_{\boldsymbol{\theta}})
            }_{(B)} .
        \end{aligned}
    \end{equation}

    We bound term $(A)$ by similar arguments in [\citep{chen2023score}, Theorem 2]:
    \begin{equation}
        \begin{aligned}
            \ell^{\text{trunc}}\left(\mathbf{x}_0, y ; \mathbf{s}_{\boldsymbol{\theta}}\right)
            &=
            \frac{1}{T-t_0}
            \int_{t_0}^T
            \mathbb{E}_{\mathbf{x}_t \mid \mathbf{x}_0}
            \left[
            \left\|
            \nabla_{\mathbf{x}_t} \log p_t\left(\mathbf{x}_t \mid \mathbf{x}_0\right)
            -
            \mathbf{s}_{\boldsymbol{\theta}}\left(\mathbf{x}_t, t, y\right)
            \right\|_2^2
            \right]
            \mathbf{1}\left\{
            \left\|\mathbf{x}_0\right\|_2 \leq R,\ |y| \leq R
            \right\}
            \mathrm{d} t\\
            &=
            \frac{1}{T-t_0}
            \int_{t_0}^T
            \mathbb{E}_{\mathbf{x}_t\mid\mathbf{x}_0}
            \left[
            \left\|
            \mathbf{s}_{\boldsymbol{\theta}}
            +
            \frac{\mathbf{x}_t}{h(t)}
            -
            \frac{\alpha(t) \mathbf{x}_0}{h(t)}
            \right\|^2_2
            \right]
            \mathbf{1}\left\{
            \left\|\mathbf{x}_0\right\|_2 \leq R,\ |y| \leq R
            \right\}
            \mathrm{d}t\\
            &\le
            \frac{2}{T-t_0}
            \int_{t_0}^T
            \mathbb{E}_{\mathbf{x}_t\mid\mathbf{x}_0}
            \left[
            \left\|
            \mathbf{s}_{\boldsymbol{\theta}}
            +
            \frac{\mathbf{x}_t}{h(t)}
            \right\|^2_2
            \right]
            \mathbf{1}\left\{
            \left\|\mathbf{x}_0\right\|_2 \leq R,\ |y| \leq R
            \right\}
            \mathrm{d}t\\
            &\quad+
            \frac{2}{T-t_0}
            \int_{t_0}^T
            \left\|
            \frac{\alpha(t)\mathbf{x}_0}{h(t)}
            \right\|^2_2
            \mathbf{1}\left\{
            \left\|\mathbf{x}_0\right\|_2 \leq R,\ |y| \leq R
            \right\}
            \mathrm{d}t\\
            &\le
            \frac{2}{T-t_0}
            \left(
            \int_{t_0}^T
            \sup_{\vartheta\in\Theta}
            \left\|
            \frac{\alpha(t)}{h(t)}
            \vartheta(\mathbf{x}_t,t,y)
            \mathbf{1}\left\{
            \left\|\mathbf{x}_0\right\|_2 \leq R,\ |y| \leq R
            \right\}
            \right\|_2^2
            \mathrm{d}t
            \right.\\
            &\quad\left.
            +
            \int_{t_0}^T
            \left\|
            \frac{\alpha(t)\mathbf{x}_0}{h(t)}
            \right\|_2^2
            \mathbf{1}\left\{
            \left\|\mathbf{x}_0\right\|_2 \leq R,\ |y| \leq R
            \right\}
            \mathrm{d}t
            \right)\\
            &\le
            \frac{2(K_{\mathrm{net}}^2+R^2)}{T-t_0}
            \int_{t_0}^T
            \left(\frac{\alpha(t)}{h(t)}\right)^2
            \mathrm{d}t \\
            &\overset{(i)}{\lesssim}
            \frac{2(K_{\mathrm{net}}^2+R^2)}{T-t_0}
            \frac{T-t_0}{(1-\exp(-\beta_{\text{min}}t_0))^2}\\
            &=
            \mathcal{O}\left(
            \frac{K_{\mathrm{net}}^2+R^2}
            {(1-\exp(-\beta_{\text{min}}t_0))^2}
            \right),
        \end{aligned}
    \end{equation}
    where inequality $(i)$ comes from the fact that $\alpha(t)=\exp(-t/2)\le 1$ and
    \begin{equation}
        h(t)
        =
        1-\exp\left(-\int_0^t\beta(s)\mathrm{d}s\right)
        \ge
        1-\exp(-\beta_{\text{min}}t).
    \end{equation}

    When $\beta_{\text{min}}t_0$ is close to $0$, which is also the practical regime, we have the approximation
    \begin{equation}
        1-\exp(-\beta_{\min} t_0)\approx \beta_{\min} t_0 .
    \end{equation}
    Moreover, following \citet{li2024diffusion}, the quantity $K_{\mathrm{net}}$ mainly depends on the input dimension $d$. In particular, when $\vartheta$ is Lipschitz continuous and $\|\mathbf{x}\|_2 \le R_{\mathcal{X}}\sqrt{d}$, we can verify that $K_{\mathrm{net}}=\mathcal{O}(\sqrt{d})$. Therefore, the truncated single-sample loss is bounded by a dimension-dependent quantity of order
    \begin{equation}
        \ell^{\text{trunc}}(\mathbf{x}_0,y;\mathbf{s}_{\boldsymbol{\theta}})
        =
        \mathcal{O}\left(\frac{\sqrt{d}}{t_0^2}\right),
    \end{equation}
    up to constants depending on the diffusion schedule and the truncation radius.

    By standard metric entropy and symmetrization techniques similar to the arguments of [\citep{li2024diffusion}, Lemma B.1], with probability at least $1-\delta$,
    \begin{equation}
        (A)
        =
        \mathcal{O}\left(
        \frac{1}{t_0^2}
        \sqrt{
        \frac{
        \mathcal{N}\left(\mathcal{S}, 1 / m\right)
        d\log (1 / \delta)}
        {m}
        }
        \right),
    \end{equation}
    where $\mathcal{N}(\mathcal{S},1/m)$ denotes the $\epsilon$-covering number of the function space $\mathcal{S}$ with $\epsilon=1/m$. This concentration step is taken over the pseudo dataset drawn from $P_{\mathbf{x},y}$, which already includes the effect of the short-run L-BFGS refinement through the pushforward distribution $P_{\mathbf{x}}=(\Phi_K)_{\#}P_{\mathrm{init}}$.

    For term $(B)$, recall that
    \begin{equation}
        (B)
        :=
        \sup_{\mathbf{s}_{\boldsymbol{\theta}} \in \mathcal S}
        \mathbb{E}_{(\mathbf{x}_0,y)\sim P_{\mathbf{x},y}}
        \left[
        \ell(\mathbf{x}_0, y ; \mathbf{s}_{\boldsymbol{\theta}})
        \mathbf{1}\{\|\mathbf{x}_0\|_2 > R\}
        \right].
    \end{equation}
    Under the compact input space setting, the projected short-run L-BFGS map satisfies $\Phi_K(\mathcal X)\subseteq\mathcal X$. Hence $P_{\mathbf{x}}$ is supported on $\mathcal X$, and $\|\mathbf{x}_0\|_2 \le R$ almost surely. Therefore, the indicator function vanishes almost surely, and we obtain
    \begin{equation}
        (B)=0.
    \end{equation}

    Finally, note that the score estimation error satisfies
    \begin{equation}
        \epsilon^2_{\text{diff}}
        =
        \mathcal{L}(\widehat{\mathbf{s}}_{\boldsymbol{\theta}})
        -
        \mathcal{L}\left(\nabla\log p_t(\mathbf{x}_t\mid y)\right)
        \le
        \mathcal{L}(\widehat{\mathbf{s}}_{\boldsymbol{\theta}}),
    \end{equation}
    which follows from the equivalence between denoising score matching and $L^2$ score estimation up to an additive constant \citep{vincent2011connection, song2020score}. Therefore, summing up $(A)$ and $(B)$ completes the proof.
\end{proof}

\subsection{Proof of Lemma~\ref{lem:TV-distance}}\label{sec:proof:lem:TV-distance}
\begin{proof}
    For clarity, we rewrite backward SDE from time $T$ to time $t_0$ on $[0, T-t_0]$:
\begin{equation}
    \mathrm{d} \mathbf{x}_t^{\leftarrow}=b\left(\mathbf{x}_t^{\leftarrow}, T-t ; y\right) \mathrm{d} t+\sqrt{\beta(T-t)} \mathrm{d} \bar{\mathbf{w}}_t, \quad \mathbf{x}_0^{\leftarrow} \sim p_T(\cdot \mid y),
\end{equation}
with the ground truth reverse drift coefficient, i.e., the drift coefficient containing the ground truth conditional score function:
\begin{equation}
    b(\mathbf{x}_t, t ; y)=-\frac{1}{2} \beta(t) \mathbf{x}_t-\beta(t) \nabla_\mathbf{x} \log p_t(\mathbf{x}_t \mid y) .
\end{equation}

Now we define the learned backward SDE, i.e., the backward SDE with conditional score function substituted by the learned conditional score predictor $\widehat{\mathbf{s}}$:
\begin{equation}
    \mathrm{d} \widetilde{\mathbf{x}}_t^{\leftarrow}=\widehat{b}\left(\widetilde{\mathbf{x}}_t^{\leftarrow}, T-t ; y\right) \mathrm{d} t+\sqrt{\beta(T-t)} \mathrm{d} \bar{\mathbf{w}}_t,
\end{equation}
where $\widehat{b}(\mathbf{x}_t, t ; y)=-\frac{1}{2} \beta(t) \mathbf{x}_t-\beta(t) \widehat{\mathbf{s}}_\theta(\mathbf{x}_t, t, y)$, and $\widetilde{\mathbf{x}}_0^{\leftarrow} \sim \mathcal{N}\left(0, \mathbf{I}\right)$. 

To be consistent with the notation in the main text, we denote by $P_{t_0}(\cdot\mid y)$ the marginal distribution of $\mathbf{x}_{T-t_0}\mid y$ under the ground true backward SDE (equivalently, the marginal distribution at time $t_0$ in the forward SDE), and by $\widehat{P}_{t_0}(\cdot\mid y)$ the marginal distribution of $\tilde{\mathbf{x}}^{\leftarrow}_{T-t_0}\mid y$ under the learned backward SDE. 

Let $\widetilde{Q}_{t_0}(\cdot \mid y)$ be the marginal law at time $T-t_0$ of the learned reverse SDE, but started from the same initialization as the ground truth reverse SDE, i.e. $\widetilde{\mathbf{x}}_0^{\leftarrow, Q} \sim p_T(\cdot \mid y)$. That is, $\widetilde{Q}_{t_0}$ corresponds to $\widetilde{\mathbf{x}}_{T-t_0}^{\leftarrow, Q}$.

By triangle inequality, we have
\begin{equation}
    \operatorname{TV}\left(P_{t_0}, \widehat{P}_{t_0}\right) \leq \underbrace{\operatorname{TV}\left(P_{t_0}, \widetilde{Q}_{t_0}\right)}_{\text{TV 1}}+\underbrace{\operatorname{TV}\left(\widetilde{Q}_{t_0}, \widehat{P}_{t_0}\right)}_{\text{TV 2}} .
\end{equation}
Now we derive the bound for $\text{TV 1}$ and $\text{TV 2}$ respectively. 

We denote by $\mathbb{P}$ the law of $\{\mathbf{x}_t^{\leftarrow}\}_{t\in[0, T-t_0]}$ with drift coefficient $b(\cdot, T-t;y)$, and $\mathbb{P}^{\boldsymbol\theta}$ the law of $\{\widetilde{\mathbf{x}}_t^{\leftarrow, Q}\}_{t\in[0, T-t_0]}$ with drift coefficient $\widehat{b}(\cdot, T-t;y)$. 

Then we define the drift difference $\Delta b(\mathbf{x}_t, t ; y):=\widehat{b}(\mathbf{x}_t, t ; y)-b(\mathbf{x}, t ; y)=-\beta(t)\left(\widehat{\mathbf{s}}_{\boldsymbol\theta}(\mathbf{x}_t, t, y)-\nabla \log p_t(\mathbf{x}_t \mid y)\right)$. 

Recall that the diffusion coefficient at reverse-time $t$ is $\sigma(T-t)=\sqrt{\beta(T-t)} \mathbf{I}$. Thus the Girsanov "control" term is

$$
\mathbf{g}_t(x):=\sigma(T-t)^{-1} \Delta b(\mathbf{x}_t, T-t ; y)=\sqrt{\beta(T-t)}\left(\widehat{\mathbf{s}}_{\boldsymbol\theta}(\mathbf{x}_t, T-t, y)-\nabla \log p_{T-t}(\mathbf{x}_t \mid y)\right). 
$$

Under Assumption~\ref{asm:Novikov-condition}, Girsanov's theorem applies and yields the KL divergence between the path laws
\begin{equation}
    \mathrm{KL}\left(\mathbb{P} \| \mathbb{P}^{\boldsymbol{\theta}}\right)=\frac{1}{2} \mathbb{E}_{\mathbb{P}} \int_0^{T-t_0}\left\|\mathbf{g}_t\left(\mathbf{x}_t^{\leftarrow}\right)\right\|_2^2 \mathrm{~d} t
\end{equation}
Substituting $\mathbf{g}_t$ and changing variable $\tau=T-t$ gives

\begin{equation}
    \mathrm{KL}\left(\mathbb{P} \| \mathbb{P}^\theta\right)=\frac{1}{2} \mathbb{E} \int_{t_0}^T \beta(\tau)\left\|\widehat{\mathbf{s}}_{\boldsymbol\theta}\left(\mathbf{x}_\tau, \tau, y\right)-\nabla \log p_\tau\left(\mathbf{x}_\tau \mid y\right)\right\|_2^2 \mathrm{~d} \tau\le \frac{1}{2}\beta_{\text{max}}(T-t_0)\epsilon_{\text{diff}}^2(y)
\end{equation}
where we use the fact that the reverse path $\{\mathbf{x}_t^{\leftarrow}, T-t\}$ has the same marginals as the forward path $\{\mathbf{x}_\tau,\tau\}$, and we define the conditional score matching error $\epsilon_{\text {diff }}^2(y):=\frac{1}{T-t_0} \int_{t_0}^T \mathbb{E}_{X_t \sim p_t(\cdot \mid y)}\left[\left\|\widehat{s}\left(X_t, t ; y\right)-s^{\star}\left(X_t, t ; y\right)\right\|_2^2\right] \mathrm{d} t$. 

Applying Pinsker's inequality, we can bound term $\text{TV 1}$:
\begin{equation}
    \operatorname{TV}\left(P_{t_0}(\cdot \mid y), \widetilde{Q}_{t_0}(\cdot \mid y)\right) \leq \sqrt{\frac{1}{2} \operatorname{KL}\left(P_{t_0} \| \widetilde{Q}_{t_0}\right)} \leq \sqrt{\frac{1}{2} \operatorname{KL}\left(\mathbb{P} \| \mathbb{P}^{\boldsymbol\theta}\right)}\le \frac{1}{2}\sqrt{\beta_{\text{max}}(T-t_0)}\epsilon_{\text{diff}}(y).
\end{equation}
By Lemma~\ref{lem:score-matching-loss} and Markov inequality, for any $\eta\in(0, 1)$, it holds with probability at least $(1-\delta)(1-\eta)$ over the randomness of both the training data and the condition $y\sim p(y)$ that
\begin{equation}
    \epsilon_{\mathrm{diff}}^2(y) \leq \frac{1}{\eta} \epsilon_{\mathrm{diff}}^2=\mathcal{O}\left(\frac{1}{\eta t_0^2} \sqrt{\frac{\mathcal{N}(\mathcal{S}, 1 / m) d\log (1 / \delta)}{m}}\right). 
\end{equation}
Therefore, we have
\begin{equation}
    \operatorname{TV}\left(P_{t_0}, \widetilde{Q}_{t_0}\right) =\mathcal{O} \left(\frac{1}{\eta t_0^2}\sqrt{\frac{\mathcal{N}(\mathcal{S}, 1 / m)d \log (1 / \delta)}{m}}\right)
\end{equation}

Notice that the processes defining $\widetilde{Q}_{t_0}$ and $\widehat{P}_{t_0}$ follow the same learned reverse SDE (same drift $\widehat{b}$, same diffusion), differing only in the initial law:

\begin{equation}
    \widetilde{\mathbf{x}}_0^{\leftarrow, Q} \sim p_T(\cdot \mid y), \quad \widetilde{\mathbf{x}}_0^{\leftarrow} \sim \mathcal{N}\left(0, I_D\right) .
\end{equation}
Let K denote the Markov kernel mapping the initial distribution at time 0 to the marginal at time $T-t_0$ under this learned SDE. Then
\begin{equation}
    \widetilde{Q}_{t_0}=\mathrm{K}_{\sharp}\left(p_T(\cdot \mid y)\right), \quad \widehat{P}_{t_0}=\mathrm{K}_{\sharp}\left(\mathcal{N}\left(0, I_D\right)\right) .
\end{equation}
By data processing (contractivity of KL under Markov kernels), we have
\begin{equation}
    \operatorname{KL}\left(\widetilde{Q}_{t_0} \| \widehat{P}_{t_0}\right) \leq \operatorname{KL}\left(p_T(\cdot \mid y) \| \mathcal{N}\left(0, I_D\right)\right) .
\end{equation}
Again by Pinsker, we can bound term $\text{TV 2}$:
\begin{equation}
    \operatorname{TV}\left(\widetilde{Q}_{t_0}, \widehat{P}_{t_0}\right) \leq \sqrt{\frac{1}{2} \operatorname{KL}\left(\widetilde{Q}_{t_0} \| \widehat{P}_{t_0}\right)} \leq \sqrt{\frac{1}{2} \operatorname{KL}\left(p_T(\cdot \mid y) \| \mathcal{N}\left(0, I_D\right)\right)}, 
\end{equation}
where the term $\operatorname{KL}\left(p_T(\cdot \mid y) \| \mathcal{N}\left(0, I_D\right)\right)$ characterizes the mismatch between the terminal marginal of the forward process and the Gaussian prior. When using VP SDE as the forward SDE, this mismatch decays exponentially fast in $T$ \citep{chen2023score}, and becomes negligible compared to $\text{TV 1}$. 

Moreover, we have the Wasserstein-2 distance between $P_{t_0}$ and $P$ (or $P_0$, i.e., the distribution at time $0$ in the forward SDE): 
\begin{equation}
    \mathrm{W}_2(P_{t_0},P)=\mathcal{O}(\sqrt{dt_0}),
\end{equation}
which follows from the results in [\citet{chen2023score}, Theorem 3]. Since in practice $t_0$ is close to $0$, this term is also negligible compared to $\text{TV 1}$. 

Finally, summing up $\text{TV 1}$, $\text{TV 2}$ and $\text{W}_2$ completes the proof. 

\end{proof}

\subsection{Proof of Lemma~\ref{lem:bound-E_2-general}}\label{sec:proof:lem:bound-E_2-general}
\begin{proof}
By the Kantorovich-Rubinstein duality, we have
    \begin{equation}
        \mathrm{W}_1(P_1,P_2)=\sup_{\text{Lip}(g)\le 1}\left\|\mathbb{E}_{P_1}[g(\mathbf{x})]-\mathbb{E}_{P_2}[g(\mathbf{x})]\right\|
    \end{equation}
    Taking $g(\mathbf{x})=f(\mathbf{x})/L_f$, we have
    \begin{equation}
        \left|\mathbb{E}_{P_1}[f(\mathbf{x})]-\mathbb{E}_{P_2}[f(\mathbf{x})]\right| \leq L_f \mathrm{W}_1\left(P_1, P_2\right)
    \end{equation}
    Moreover, since both distributions are supported on the compact set $\mathcal{X}$, we have $\mathrm{W}_1(P_1,P_2)\le \text{diam}(\mathcal{X})\text{TV}(P_1,P_2)$, which completes the proof. 
\end{proof}

\subsection{Proof of Lemma~\ref{lem:bound-E_2}}\label{sec:proof:lem:bound-E_2}
\begin{proof}
By Lemma~\ref{lem:TV-distance} and Lemma~\ref{lem:bound-E_2-general}, with probability $(1-\delta)(1-\eta)$, we have
    \begin{equation}
    \begin{aligned}
        \mathcal{E}_2
        &=\left|\mathbb{E}_{\mathbf{x} \sim P_a}[f(\mathbf{x})]-\mathbb{E}_{\mathbf{x} \sim \widehat{P}_a}[f(\mathbf{x})]\right|\\
        &\le 
        \left|\mathbb{E}_{\mathbf{x} \sim P(\cdot\mid a)}[f(\mathbf{x})]-\mathbb{E}_{\mathbf{x} \sim P_{t_0}(\cdot\mid a)}[f(\mathbf{x})]\right|
        +
        \left|\mathbb{E}_{\mathbf{x} \sim P_{t_0}(\cdot\mid a)}[f(\mathbf{x})]-\mathbb{E}_{\mathbf{x} \sim \widehat{P}_{t_0}(\cdot\mid a)}[f(\mathbf{x})]\right|\\
        &\le L_f\, \mathrm{W}_1\!\left(P(\cdot\mid a),P_{t_0}(\cdot\mid a)\right)
        + L_f\,\mathrm{diam}(\mathcal{X})\,\mathrm{TV}\!\left(P_{t_0}(\cdot\mid a),\widehat{P}_{t_0}(\cdot\mid a)\right)\\
        &\le L_f\text{W}_2(P,P_{t_0}) + L_f\text{diam}(\mathcal{X})\text{TV}(P_{t_0}, \widehat{P}_{t_0})\\
        &=\mathcal{O}\!\left(L_f\sqrt{d t_0}\right)
        +\mathcal{O}\!\left(\frac{L_f\,\mathrm{diam}(\mathcal{X})}{\eta t_0^2}
        \sqrt{\frac{\mathcal{N}(\mathcal{S}, 1/m)d\log(1/\delta)}{m}}
        \right).
    \end{aligned}
\end{equation}
And for the same reason in~\ref{sec:proof:lem:TV-distance}, the first term is negligible, which completes the proof. 
\end{proof}

\end{document}